\PassOptionsToPackage{table,dvipsnames}{xcolor}
\documentclass{wanderlab}

\usepackage{amsmath,amsfonts,bm}

\usepackage[dvipsnames]{xcolor}
\usepackage[textsize=scriptsize]{todonotes}

\newif\ifshowcomments
\showcommentstrue

\ifshowcomments
  \newcommand{\alperen}[2][]{%
    \todo[
      color=blue!15,
      linecolor=blue!70!black,
      bordercolor=blue!70!black,
      #1
    ]{\textbf{Alperen:} #2}%
  }

  \newcommand{\alpereninline}[2][]{%
    \todo[
      inline,
      color=blue!10,
      linecolor=blue!70!black,
      bordercolor=blue!70!black,
      #1
    ]{\textbf{Alperen:} #2}%
  }
\else
  \newcommand{\alperen}[2][]{}
  \newcommand{\alpereninline}[2][]{}
\fi

\def\eqref#1{equation~\ref{#1}}

\def\1{\bm{1}}

\DeclareMathAlphabet{\mathsfit}{\encodingdefault}{\sfdefault}{m}{sl}
\SetMathAlphabet{\mathsfit}{bold}{\encodingdefault}{\sfdefault}{bx}{n}

\renewcommand*{\backref}[1]{}
\renewcommand*{\backrefalt}[4]{\ifcase#1\relax\or(page #2)\else(pages #2)\fi}

\usepackage{mathtools}

\AddToHook{cmd/appendix/before}{%
    \crefalias{section}{appendix}%
    \crefalias{subsection}{appendix}
}
\renewcommand{\eqref}[1]{(\ref{#1})}
\usepackage{url}
\usepackage{wrapfig}
\usepackage{array}
\usepackage{comment}
\usepackage{tabularx,makecell}

\definecolor{OursBg}{HTML}{EDE7F6}
\definecolor{LinkBlue}{HTML}{1A0DAB}

\newcommand{\cmark}{\textcolor{green!50!black}{\ding{51}}}
\newcommand{\xmark}{\textcolor{red!60!black}{\ding{55}}}

\usepackage{colortbl}

\title{To Learn is to Wander: Learning Across Graphs and\\[-0.25em] Tasks with Random Walks}

\renewcommand{\authorlist}{%
  {
    \textbf{Louis Tichelman}\textsuperscript{1,2},\;\;
    \textbf{Xingyue Huang}\textsuperscript{3},\;\;
    \textbf{Jinwoo Kim}\textsuperscript{4},\;\;\textbf{\.{I}smail \.{I}lkan Ceylan}\textsuperscript{1,2,3}
  }
}

\renewcommand{\affiliationlist}{%
  {\small\textsuperscript{1}TU Wien\quad \textsuperscript{2}AITHYRA\quad \textsuperscript{3}University of Oxford\quad \textsuperscript{4}KAIST}
}

\abstract{Graph foundation models aim to transfer across graphs, feature spaces, relational schemas, and prediction tasks, yet existing approaches typically generalize only within particular graph modalities or tasks. We propose \textsc{Wander}, a graph foundation model designed to operate across these settings within \emph{a single pretrained checkpoint}. Following the prior-predictive perspective, we formulate graph learning as completion of a partially observed graph. We realize this task-general view through a common interface based on random walks, allowing the same model to operate across homogeneous and multi-relational graphs with varying features, labels, and relational schemas. \textsc{Wander} can increase its structural context at inference time without changing its learned parameters and, under suitable assumptions, universally approximates the corresponding Bayes-optimal predictor on bounded connected graphs. Empirically, a single pretrained checkpoint achieves state-of-the-art or highly competitive results across node classification, homogeneous link prediction, and knowledge-graph link prediction. 
Moreover, joint pretraining across graph modalities and tasks preserves performance in specialized settings while enabling positive transfer and the composition of separately learned capabilities at inference time.%
\\[4pt]
\textbf{Code:  \url{https://github.com/louistichelman/wander-gfm}}}

\begin{document}

\maketitle

\section{Introduction}
Graph-structured data arises across a wide range of domains, from molecules and citation networks to social systems and knowledge graphs~\citep{pei2020geom, liben2003link, huang2023theory}. These domains differ in the information associated with their graphs and in the prediction problems of interest, and graph-learning models have traditionally been trained end-to-end for a specific task setting. Graph foundation models aim to make learning on such data more general: rather than training a separate model for every graph and task, the goal is to learn models that transfer across graphs, feature spaces, relation types, and prediction problems. 

\begin{wraptable}{r}{0.6\columnwidth}
\vspace{-12pt}
\caption{Existing GFMs generalize across different subsets of graph modalities and prediction tasks. \textsc{Wander} uses a single model across both node- and edge-level tasks on homogeneous and multi-relational graphs.}
\centering
\scriptsize\label{tab:comparison_other_models}
\begin{tabular}{ccccc}
        \toprule
        \multirow{3}{*}{\textbf{\shortstack{Graph Foundation\\[-1pt]Model Families}}}
        & \multicolumn{2}{c}{\textbf{Input Graphs}}
        & \multicolumn{2}{c}{\textbf{Prediction Task}} \\
                & \shortstack{node\\[-2pt]features}
        & \shortstack{edge\\[-2pt]types}
        & \shortstack{node\\[-2pt]property}
        & \shortstack{link}\\
        \midrule

        \hyperlink{cite.ultra}{\textsc{Ultra}} (2024),
        \hyperlink{cite.kim2025flock}{\textsc{Flock}} (2025)
        & \xmark & \textcolor{green!50!black}{\cmark}
        & \xmark & \textcolor{green!80!black}{\cmark} \\

        \hyperlink{cite.xia2024opengraph}{\textsc{OpenGraph}},
        \hyperlink{cite.xia2024anygraph}{\textsc{AnyGraph}} (2024)
        & \textcolor{green!80!black}{\cmark} & \textcolor{red!00!black}{\xmark}
        & \textcolor{green!80!black}{\cmark}
        & \textcolor{green!80!black}{\cmark} \\

        \hyperlink{cite.finkelshtein2026equivariance}{\textsc{TS-Net}} (2026),
        \hyperlink{cite.eremeev2025graphpfn}{\textsc{GraphPFN}} (2025)
        & \textcolor{green!80!black}{\cmark} & \xmark
        & \textcolor{green!80!black}{\cmark} & \xmark \\
        
        \hyperlink{cite.dong2024universal}{\textsc{UniLP}} (2024),
        \hyperlink{cite.liao2026tfmlinker}{\textsc{TFMLinker}} (2026)
        & \xmark & \xmark
        & \xmark &  \textcolor{green!80!black}{\cmark} \\
        
        \midrule
        
        \textbf{\textsc{Wander}} (ours)
        & \textcolor{green!80!black}{\cmark}
        & \textcolor{green!80!black}{\cmark}
        & \textcolor{green!80!black}{\cmark}
        & \textcolor{green!80!black}{\cmark} \\

        \bottomrule
    \end{tabular}
\vspace{-8pt}
\end{wraptable}
Existing approaches, however, only partially realize this goal and the emerging landscape of GFMs remains fragmented, as summarized in \Cref{tab:comparison_other_models}: 
some generalize node prediction across attributed homogeneous graphs with arbitrary feature and label spaces \citep{zhao2024graphany,xia2024anygraph,finkelshtein2026equivariance,hayler2025graph,choi2026learning,eremeev2025graphpfn}, some others generalize link prediction across knowledge graphs with arbitrary relational vocabularies \citep{ultra, trix,motif,kim2025flock}, and some focus on link prediction on homogeneous graphs \citep{dong2024universal,liao2026tfmlinker}.
What is still missing is a unified formulation that spans these different graph-learning settings. This leaves two closely related questions: \emph{what predictive function should such a general graph foundation model learn?} \emph{How can a single architecture realize this function across previously unseen graphs, feature spaces, relations, and tasks?}

This work makes three contributions toward answering these questions:

\vspace{5pt}
\textbf{(1) A task-general predictive objective.} We answer the first question through a probabilistic view of graph learning. Following the prior-predictive perspective of \textsc{TabPFN} \citep{hollmann2022tabpfn}, we view an observed graph as a partially specified realization of an underlying complete graph drawn from a graph distribution. Node- and edge-level prediction problems can then be expressed as completion tasks on this partially observed graph: completing node labels corresponds to node classification, while completing edges corresponds to link prediction. More generally, each task defines a query about an unobserved component of the graph, with the Bayes-optimal prediction given by the posterior expectation of the query-dependent target conditioned on the observation. This gives one predictive function across graph-learning tasks.

\vspace{5pt}
\textbf{(2) A common computational architecture.} Realizing this task-general predictor requires a common interface that can operate across varying graph structures, feature spaces, relations, and label spaces. We propose \textsc{Wander}, which uses random walks as this interface. Random walks have a long history in graph learning \citep{PerozziAS14,node2vec,tonshoff_crawl, walklm,kim2025revisiting}, and have recently re-emerged as a computational primitive in knowledge graph foundation models \citep{kim2025flock}. \textsc{Wander} extends this interface across graph settings and prediction tasks by representing features, labels, and structure through separate
channels and processing their information along sampled walks, enabling a single
architecture to operate on homogeneous and multi-relational graphs, accommodate
varying feature and label spaces, and answer both node- and edge-level queries.
It has three properties particularly useful for a foundation model. 
\begin{itemize}[leftmargin=3em, labelsep=0.3em,topsep=0pt, partopsep=0pt,itemsep=0.2em]
    \item Unlike message-passing GFMs~\citep{finkelshtein2026equivariance,eremeev2025graphpfn}, whose structural context is fixed by network depth, \textsc{Wander} can flexibly adjust its structural context at inference time by varying its random-walk sampling strategy, without modifying its learned parameters.
    \item Random-walks can capture structural properties beyond standard message passing architectures. We formalize this by showing that, under suitable capacity assumptions, \textsc{Wander} is a universal approximator of the Bayes-optimal graph predictor on bounded connected graphs (Theorem~\ref{thm:two-layer-universality}).
    \item \textsc{Wander} is permutation equivariant in distribution, allowing stochastic symmetry breaking while respecting graph symmetries at the distributional level to facilitate generalization (Proposition~\ref{prop:permutation-equivariance}).
\end{itemize}

\vspace{5pt}
\textbf{(3) A generalist model with compositional and adaptable capabilities.} Empirically, we evaluate a \emph{single} pretrained \textsc{Wander} checkpoint across node classification, homogeneous link prediction, and knowledge-graph link prediction, obtaining state-of-the-art or highly competitive results across all three settings. 
Beyond this breadth, joint pretraining yields \textit{positive transfer} across graph-learning settings. Further to this, \textsc{Wander} can \emph{combine separately learned capabilities} to solve new task settings.
Finally, its random-walk interface supports \textit{adaptation} without parameter updates: expanding or targeting the sampled context improves long-range reasoning on synthetic and real-world graphs. Together, these findings show benefits of unification beyond task coverage alone.

\section{Related work}
\label{sec::related work}

\paragraph{Graph foundation models for node property prediction.}
Node property prediction has traditionally relied on message-passing neural networks such as GCN~\citep{kipf2016semi}, GAT~\citep{gat}, and GIN~\citep{xu2018how}, which are typically trained end-to-end on a single graph. Early graph foundation models aim for architectures that transfer across feature and label spaces. \text{GraphAny}~\citep{zhao2024graphany} combines predictions from multiple linear graph filters through learned attention that transfer across graphs. TS-GNN~\citep{finkelshtein2026equivariance} formalizes the symmetry requirements of node-level graph foundation models under node, feature, and label permutations, and proposes a corresponding symmetry-preserving architecture. Another line of work builds on \textit{tabular} foundation models~\citep{hollmann2022tabpfn,tabpfnv2,qu2025tabicl}, which perform \emph{in-context learning} from labeled rows at inference time. G2T-FM~\citep{eremeev2025turning} and TAG~\citep{hayler2025graph} encode graphs as node-level tables for in-context prediction with pretrained tabular foundation models. \text{NodePFN}~\citep{choi2026learning} extends \text{TabPFN}~\citep{hollmann2022tabpfn} to graphs through pretraining on synthetic graphs, retaining a single representation per node and a fixed maximum feature dimension. \text{GraphPFN}~\citep{eremeev2025graphpfn} builds on LimiX~\citep{zhang2025limix} with graph-aware aggregation between feature-tokenized blocks, more naturally accommodating varying feature dimensions across graphs. \textsc{Wander} instead uses random walks as its structural interface for in-context learning, accommodating varying feature spaces and edge types.

\paragraph{Graph foundation models for link prediction.}
Classical approaches for link prediction over \emph{attributed graphs}, where nodes carry features and edges are typically of a single type, range from GNN-based encoders~\citep{kipf2016semi} to link-specific architectures such as SEAL~\citep{zhang2018link} and BUDDY~\citep{chamberlain2022graph}. More recent foundation models aim to replace graph-specific training with transferable predictors that generalize across unseen graphs, including \text{Prodigy}~\citep{huang2023prodigy}, \text{OFA}~\citep{liu2024one}, \text{UniLP}~\citep{dong2024universal}, \text{OpenGraph}~\citep{xia2024opengraph}, and \text{AnyGraph}~\citep{xia2024anygraph}. TFMLinker~\citep{liao2026tfmlinker} instead tabularizes link-centric graph information for in-context prediction with a pretrained tabular foundation model. In parallel, \emph{knowledge graphs} link prediction has developed a separate family of architectures, from R-GCN~\citep{schlichtkrull2017modeling} and CompGCN~\citep{vashishth2020compositionbased} to query-conditioned models such as NBFNet~\citep{zhu2022neural} and C-MPNNs~\citep{huang2023theory}. Knowledge graph foundation models (KGFMs) further extend this setting to fully inductive transfer, where both entities and relation types can be unseen at test time. \textsc{Ultra}~\citep{ultra}, \textsc{Motif}~\citep{motif}, and \textsc{Trix}~\citep{trix} achieve this primarily through relation- and entity-level message passing, while \textsc{Flock}~\citep{kim2025flock} instead uses random walks as the interface for structural reasoning. 
These models provide broad transfer within either homogeneous attributed graphs or relational graphs, but not a common model across both regimes and prediction tasks. 

\paragraph{Random walks for graph learning.}
Random walks have long been used to capture graph structure, with DeepWalk~\citep{PerozziAS14} and node2vec~\citep{node2vec} learning node embeddings from sampled walk sequences. Later methods use walks more directly as graph representations: CRaWL~\citep{tonshoff_crawl} processes collections of walks with neural networks, while WalkLM~\citep{walklm} incorporates node attributes into walk sequences for language-model-based processing. RWNN~\citep{kim2025revisiting} further treats anonymized random walks as a general computational interface for graph learning. \textsc{Flock}~\citep{kim2025flock} extends this paradigm to knowledge graphs for zero-shot link prediction. 
\textsc{Wander} builds on this line by using the same walk-based graph computation across graph settings and tasks, while its in-context architecture accommodates node attributes, labels, and relational edge types within a single model.

\section{Preliminaries}\label{sec::preliminary}

\paragraph{Attributed edge-typed graphs.}
An \textit{attributed edge-typed graph} is a tuple
$G=(V,E,R,X,Y)$, where $V$ is the set of nodes, $R$ is the set of
relation types, $E\subseteq V\times R\times V$ is the set of typed edges,
$X\in\{\emptyset\}\cup\mathbb{R}^{V\times d_X}$ is the matrix of
$d_X$-dimensional node features, and
$Y\in\{\emptyset\}\cup([d_Y]\cup\{\bot\})^V$ contains the node labels,
with $d_Y\in\mathbb{N}$ denoting the number of classes. For any
$k\in\mathbb{N}$, we write $[k]:=\{1,\dots,k\}$. We call a node
$v\in V$ \emph{unlabeled} if $Y_v=\bot$. Features and labels may both
be absent.
We define $\mathbb{G}_{n,m,b}$ as the set of attributed edge-typed graphs with at most $n$ nodes, $m$ relation types, and $b$ feature and label dimensions and $\mathbb{G}_{n,m,b}^{\star}\subset \mathbb{G}_{n,m,b}$ as the set of \textit{fully specified graphs} without unlabeled nodes.
In the following, we refer to attributed edge-typed graphs simply as \textit{graphs}.

\paragraph{Partial observations and graph isomorphism.}
We write $\pi=(\pi_V,\pi_R,\pi_X,\pi_Y)$ for bijective
relabelings of nodes, relation types, feature dimensions, and
class labels, respectively. We extend $\pi_Y$ by
$\pi_Y(\bot)=\bot$ and denote the relabeled graph by $\pi(G)$.
A graph $G$ is a \textit{partial observation} of $G'$ (written
$G\preceq G'$) if, up to such a relabeling, it can be obtained
from $G'$ by deleting a subset of its edges and masking a
subset of its node labels with $\bot$.
Two graphs are isomorphic (written $G\cong G'$) if
$G\preceq G'$ and $G'\preceq G$.

\paragraph{Modeling posterior graph distributions.} In this work we are interested in predicting missing labels or edges in an (incomplete) observed graph $G_{\text{obs}}\in \mathbb{G}_{n,m,b}$ conditioned on a prior distribution over graphs. Formally, let $p$ be a probability distribution over the isomorphism classes
of $\mathbb G_{n,m,b}^{\star}$. In the following, we slightly abuse notation by identifying an isomorphism class with any of its representatives and write $G$ interchangeably for a graph and its isomorphism class. We denote by \(p(\,\cdot\mid G_{\mathrm{obs}})\) the conditional
distribution over graph completions consistent with \(G_{\mathrm{obs}}\).

However, rather than explicitly representing the full posterior distribution,
we aim to model posterior expectations of query-dependent target random
variables. Formally, we define a \textit{link prediction query} $q=(h_q,r_q,?)$ consisting of a query head $h_q\in V_{\text{obs}}$ and query relation $r_q\in R_{\text{obs}}$, and a \textit{label prediction query} $q=V_q\subset V_{\text{obs}}$ consisting of a set of unlabeled nodes. For $G^{\star}\sim p(\cdot\mid G_{obs})$ we then define the \textit{query target} random variable $T_q(G^{\star})$, where 
\[
\qquad\qquad T_q(G^{\star})=
\left\{
\begin{array}{llr}
(\textbf{1}_{(h_q,r_q,t)\in E^{\star}})_{t\in V}
& \text{if } q=(h_q,r_q,?)
& \textit{Link Prediction}
\\[1mm]

((\textbf{1}_{Y_v^{\star}=c})_{c\in[d_Y]})_{v\in V_q}
& \text{if } q=V_q\quad
& \textit{Node Classification}
\end{array}
\right.
\]
The Bayes-optimal predictor for query $q$ is therefore 
\begin{align*}
    M_p^{\star}:(G_{\text{obs}}, q)\mapsto \mathbb{E}_{G^{\star}\sim p(\cdot\mid G_{\text{obs}})}[T_q(G^{\star})].
\end{align*}
Intuitively, the query specifies which missing part of the graph is to be
completed, while the same predictor returns its posterior estimate regardless
of whether that part is a node label or an edge.

\paragraph{Random walks.}
Given a graph $G$, we define a walk $\eta$ of length $S$ in $G$ as a sequence
\(
\eta=(v_0,r_1^{\alpha_1},v_1,\ldots,r_S^{\alpha_S},v_S),
\)
where $v_j\in V$, $r_j\in R$, and $(v_{j-1},r_j,v_j)\in E$ if
$\alpha_j=1$ and $(v_j,r_j,v_{j-1})\in E$ if $\alpha_j=-1$.
We denote the set of such walks of arbitrary length by $\mathcal W(G)$.
A random-walk algorithm $A$ associates with every graph-query pair $(G,q)$
a random variable $A(G,q)$ taking values in $\mathcal W(G)$.

\section{Learning the Bayes-Optimal Graph Predictor}
\label{sec::bayes-predictor}
This section characterizes the task-general predictor introduced above. We establish Bayes-level compatibility across graph tasks, characterize when expanding the graph prior can improve target performance, and formalize the predictor's permutation symmetries. 
Proofs are in Appendix~\ref{app::proofs}.

\paragraph{Learning across tasks.}
Any graph foundation model can be trained to approximate
$M_p^\star$ using samples $G^\star\sim p$, incomplete observations
$G_{\mathrm{obs}}\preceq G^\star$, and queries $q$, by minimizing
\begin{equation*}
    \mathcal{L}_p(\theta)
    =
    \mathbb{E}_{p}\!\left[
        \ell_q\!\left(
            M_\theta(G_{\mathrm{obs}},q),T_q(G^\star)
        \right)
    \right].
\end{equation*}
Here, the expectation includes observation and query sampling,
consistent with the posterior defined above.
Categorical cross-entropy for node classification and binary
cross-entropy for link prediction both have $M_p^\star$ as their population minimizer.
Moreover, restricting this objective to one task leaves its
optimal predictions unchanged: tasks merely query different
components of the same posterior over graph completions.
Thus, assuming sufficient capacity,
\textit{joint training has no Bayes-level disadvantage over
single-task training}.

\paragraph{Learning across diverse graphs.}
Broadening the pretraining prior can improve coverage of target
graphs, but may also introduce graph families whose predictions are
unsuitable for some targets. Theorem~\ref{thm:prior-expansion} shows that expansion
cannot hurt if the observation rules out the added family
wherever its predictions would be worse. More generally,
it guarantees improvement when the family's
posterior-weighted predictive benefits outweigh its harm.

Consider two graph distributions $p$ and $r$ and define $p_\lambda$ by sampling from $p$ with probability $(1-\lambda)$ and from $r$ otherwise.
Further, define $\beta_{\lambda,r}(G_{\text{obs}})$ as the posterior weight of $r$, i.e., the probability that the complete graph stems from $r$ given observation $G_{\text{obs}}$. Then we can write
\begin{align*}
    M_{p_\lambda}^\star(G_{\mathrm{obs}},q)
    ={}&
    \bigl(1-\beta_{\lambda,r}(G_{\mathrm{obs}})\bigr)
    M_p^\star(G_{\mathrm{obs}},q) +
    \beta_{\lambda,r}(G_{\mathrm{obs}})
    M_r^\star(G_{\mathrm{obs}},q).
\end{align*}
Thus, the influence of an added family depends on its posterior
relevance to each target observation.
Let now $\nu$ be a target
distribution over complete graphs, observations, and queries,
and define the risk of a classifier $M$ as
\(
    R_\nu(M)
    :=
    \mathbb{E}_{\nu}\!\left[
        \ell_q\!\left(
            M(G_{\mathrm{obs}},q),T_q(G^\star)
        \right)
    \right].
\)
\begin{theorem}
\label{thm:prior-expansion}
Assume that \(\ell_q\) is convex in its prediction argument,
that \(p(\,\cdot\mid G_{\mathrm{obs}})\) is defined
\(\nu\)-almost surely, and that the risks below are finite. Define the added family's conditional risk difference
relative to the existing predictor as
\[
\begin{aligned}
    \Delta_\nu(G_{\mathrm{obs}},q)
    := \mathbb{E}_{\nu}\!\Big[
        &\ell_q\!\left(
            M_r^\star(G_{\mathrm{obs}},q),T_q(G^\star)
        \right) -
        \ell_q\!\left(
            M_p^\star(G_{\mathrm{obs}},q),T_q(G^\star)
        \right)
        \mid G_{\textnormal{obs}}, q\Big].
\end{aligned}
\]
On observations with zero marginal probability density under $r$, set
$M_r^\star=M_p^\star$. Then
\begin{equation*}
    R_\nu(M_{p_\lambda}^\star)-R_\nu(M_p^\star)
    \le
    \mathbb{E}_{\nu}\!\left[
        \beta_{\lambda,r}(G_{\mathrm{obs}})
        \cdot\Delta_\nu(G_{\mathrm{obs}},q)
    \right].
\end{equation*}
\end{theorem}
Thus, expanding the prior helps when the added graph family receives posterior
mass where its predictions are beneficial and cannot hurt where it is ruled out by the observation.
\paragraph{Permutation equivariance.}
Lastly, we formalize the symmetries of the Bayes-optimal
predictor under a relabeling $\pi$ as introduced in Section~\ref{sec::preliminary}.
We write $\pi(q)$ for the correspondingly relabeled query
and $\pi_{\mathrm{out}}$ for the induced relabeling of output
indices: node indices transform according to $\pi_V$ and,
for node classification, class indices additionally transform
according to $\pi_Y$.
\begin{proposition}\label{prop:bayes-is-equivariant}
For any graph prior $p$ over isomorphism classes, the Bayes-optimal predictor
is permutation equivariant:
\(
M_p^\star(\pi(G),\pi(q))
=
\pi_{\mathrm{out}}\!\left(M_p^\star(G,q)\right).
\)
\end{proposition}

\section{Methodology}\label{sec::method}
\begin{figure}[t]
    \centering
    \includegraphics[width=\columnwidth]{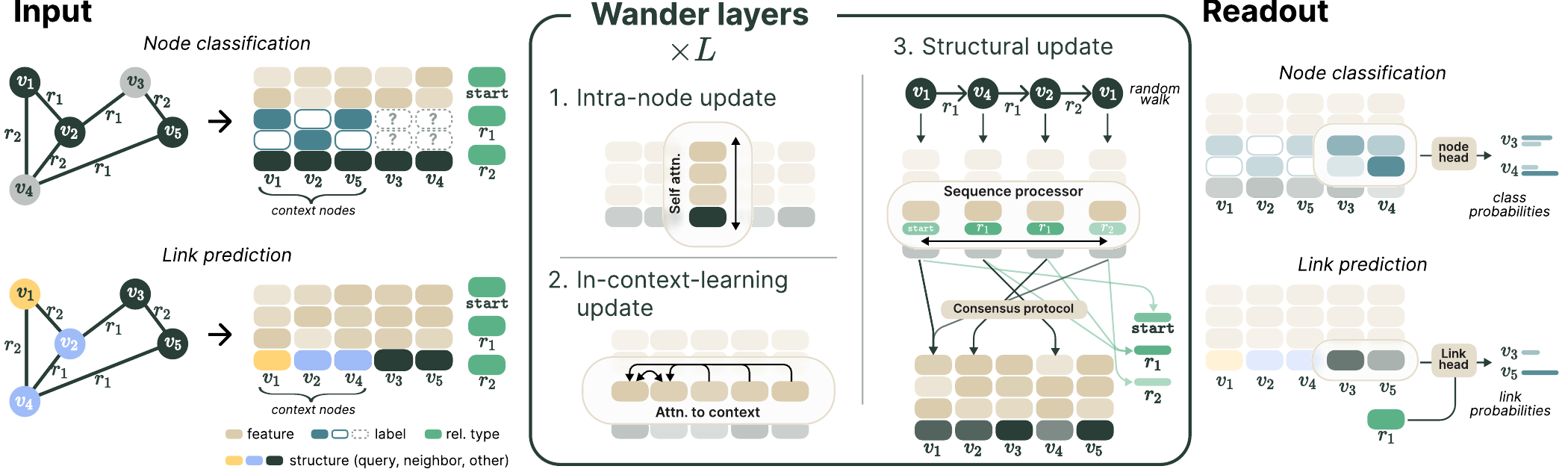}
    \caption{\textbf{Overview of \textsc{Wander}}. Each layer combines (1) intra-node attention across feature, label, and structure tokens; (2) in-context attention from each node to context nodes within the same channel; and (3) a stochastic structural update that processes sampled walks and aggregates proposals through a consensus protocol. After \(L\) layers, task-specific readouts map the query node’s label tokens to class probabilities or each candidate node’s structure token to a link probability.}
    \label{fig:wander_architecture}
\end{figure}

In this section, we present the \textsc{Wander} framework. The implementation details for the experiments presented in Section \ref{sec::experiments} are provided in Appendix~\ref{app::method}.  An overview is provided in \Cref{fig:wander_architecture}. We additionally include a comparison to existing foundation models in Appendix~\ref{app:comparison}.

\paragraph{Input and internal representation.}
Let $G=(V,E,R,X,Y)$ and query $q$ be the input to \textsc{Wander}. Depending on the prediction task, we define the set of \emph{context nodes} as
\begin{align*}
    V_{\mathrm{context}}
    =
    \begin{cases}
        \{v\in V \mid v=h_q \text{ or } (h_q,r_q,v)\in E\},
        & \text{if } q=(h_q,r_q,?),\\
        \{v\in V \mid Y_v\neq\bot\},
        & \text{if } q=V_q.
    \end{cases}
\end{align*}
Thus, context nodes correspond to observed positive examples for link prediction and observed labeled examples for node prediction.

Internally, \textsc{Wander} maintains node states
$\mathbf{v}:V\times[d_X+d_Y+1]\rightarrow\mathbb{R}^d$
and relation states
$\mathbf{r}:R\rightarrow\mathbb{R}^d$.
Each node has one token for every feature dimension and label class, together with a single \emph{structure token}; we refer to these indices collectively as \emph{channels}.
At the beginning of a forward pass, feature and label tokens are initialized from their observed values, structure tokens encode the node's role relative to the query, and relation types all receive the same learnable initial state.
We denote the resulting states by $\mathbf{v}^{(0)}$ and $\mathbf{r}^{(0)}$.

\paragraph{Single layer.}
Each layer $l\in\{0,\ldots,L-1\}$ updates these states through three operations.

\begin{itemize}[leftmargin=3em, labelsep=0.3em,topsep=0pt, partopsep=0pt,itemsep=0.2em]

\item \emph{Intra-node update.}
For each node $v\in V$, an intra-node processor $\mu_\theta$ is applied to the set of tokens $\textbf{v}^{(l)}(v,\cdot)$ corresponding to $v$, yielding intermediate updates $\textbf{v}^{(l)'}(v,\cdot)$.
This allows feature, label, and structural information associated with the same node to interact.

\item \emph{In-context-learning update.}
For each channel $i\in[d_X+d_Y+1]$, each token $\textbf{v}^{(l)'}(v,i)$ attends to the context tokens $\{\textbf{v}^{(l)'}(u,i)\mid u\in V_{\text{context}}\}$, yielding intermediate states $\textbf{v}^{(l)''}(v,i)$.
This provides the in-context learning mechanism through which observed examples can inform the query representations.

\item \emph{Stochastic structural update.}
Finally, graph structure is incorporated through random walks.
This update consists of four steps: First, a \textit{random walk algorithm} $A$ produces a set of walks $\eta_1,\dots,\eta_K$ on $G$. Next, a \textit{recording protocol} $w:\eta_{j}\mapsto \{Z_{j,i}\mid i\in [d_X+d_Y+1]\}$ broadcasts each walk to one walk recording per channel $i$. Then, a \textit{sequence processor} $\nu_\theta :Z_{j,i}\mapsto Z_{j,i}'$ is applied to each sequence independently. Since the same state may be encountered by several walks, a permutation-invariant \emph{consensus protocol} aggregates its proposals, yielding  updated states $\textbf{v}^{(l+1)}(\cdot,\cdot), \textbf{r}^{(l+1)}(\cdot)$.

\end{itemize}
We require the recording protocol $w$ to be equivariant to permutations of node
and relation identities. All channel-wise modules are shared across
channels of the same type (feature, label and structure).

\paragraph{Readout.}
After $L$ layers, the same representation supports the different prediction tasks.
For a node-prediction query $q=V_q$, the label tokens of $v\in V_q$ are passed to a node head, producing class probabilities
$M_\theta(G,q)\in[0,1]^{V_q\times d_Y}$.
For a link-prediction query $q=(h_q,r_q,?)$, the structure token of each candidate node, together with the state of the query relation $r_q$, is passed to a link head, producing $M_\theta(G,q)\in[0,1]^V$.

\section{Theoretical analysis of \textsc{Wander}}\label{sec::theory}
This section establishes the model’s main theoretical guarantees, namely universality and permutation equivariance. All proofs can be found in Appendix~\ref{app::proofs}.

\paragraph{Universality.} 
We first formalize universality with respect to a graph prior. We here restrict attention to \textit{admissible} graph priors, under which the continuous graph attributes admit a continuous mapping within each discrete isomorphism class. A formal definition is given in Appendix~\ref{app::proofs}. We call a model architecture $M$ a \textit{universal graph prior approximator} if for any admissible prior $p$ over isomorphism classes of $\mathbb{G}_{n,m,b}^\star$ and any $\epsilon, \delta>0$, there exists a choice of network parameters $\theta$ and a finite sampling budget $K\in \mathbb{N}$ such that
\begin{align*}
    \Pr\big[|\!|M_{\theta,K}(G,q)-M_p^\star(G,q)|\!|<\epsilon\big] > 1-\delta 
\end{align*}
for all connected observations \(G \in \mathcal G_{n,m,b}\)
in the support of \(p\) and all valid queries \(q\).
\begin{theorem}
\label{thm:two-layer-universality}
Fix $n,m,b\in\mathbb{N}$. With at least two layers,
a walk length of at least $n$, a sufficiently uniform random walk sampler, and sufficient hidden dimension and capacity of all modules\footnote{See Assumptions (A1)-(A5) in Appendix~\ref{app::proofs}.}, \textsc{Wander} is a universal graph prior approximator for graph priors over $\mathbb{G}_{n,m,b}^\star$.
\end{theorem}
The result concerns the general \textsc{Wander} framework under the
capacity assumptions of Appendix~\ref{app::proofs}, rather than the finite
implementation used in our experiments. For graphs with at most \(n\)
nodes, a walk length of \(n\) suffices. The number of sampled walks,
however, can be increased arbitrarily at inference time without changing
the learned parameters, allowing the sampling budget required by the
approximation guarantee. Unlike the results of
\citet{kim2025revisiting,kim2025flock}, the required walk length need
not scale with the graph's cover time.

\paragraph{Permutation equivariance.} A graph foundation model intended to approximate the Bayes-optimal predictor $M_p^\star$ should respect
the same symmetries (see Proposition \ref{prop:bayes-is-equivariant}) by design. For the stochastic computation of \textsc{Wander},
we require\textit{ equivariance in distribution}. We call a random-walk algorithm $A$
permutation invariant if, for every $G,q,\eta,$ and $\pi$,
\[
\Pr[A(G,q)=\eta]
=
\Pr[A(\pi(G),\pi(q))=\pi(\eta)],
\]
and a stochastic model $M_\theta$ permutation equivariant in distribution if
\[
M_\theta(\pi(G),\pi(q))
\overset{d}{=}
\pi_{\mathrm{out}}\!\left(M_\theta(G,q)\right).
\]

\begin{proposition}\label{prop:permutation-equivariance}
Suppose that the random-walk algorithm $A$ is permutation invariant and that
the intra-node processor $\mu_\theta$ is equivariant in distribution with
respect to permutations of the feature and label channels. Then every
\textsc{Wander} model is permutation equivariant in distribution with respect
to node, relation, feature, and label permutations.
\end{proposition}

\section{Experiments}\label{sec::experiments}
We organize our experiments around four main questions.

\begin{itemize}[leftmargin=3em, labelsep=0.3em,topsep=0pt, partopsep=0pt,itemsep=0.2em]
    \item[(\textbf{Q1})] Can a single \textsc{Wander} model remain competitive with
foundation models specialized to individual graph-learning settings?
\item[(\textbf{Q2})] Does joint pretraining across these settings cause interference
or enable positive transfer?
\item[(\textbf{Q3})] Can \textsc{Wander} combine
separately learned capabilities during inference?
\item[(\textbf{Q4})] Can \textsc{Wander} adapt its structural computation at
inference time to different graph and task requirements? 
\end{itemize}

Additional experiments on runtime, structural expressivity, finetuning, pretraining priors, and architecture ablations are provided in Appendices ~\ref{app::runtime}--\ref{app::ablations}.

\subsection{Generalist performance across graph learning tasks (Q1)}

To answer \textbf{Q1}, we train \textsc{Wander} across a deliberately diverse graph prior, spanning multiple tasks and evaluate on three standard graph-learning regimes represented in our benchmark suite: knowledge-graph link prediction, node classification on simple graphs, and homogeneous link prediction.

\paragraph{Pretraining prior.}
Following the prior-data fitted network paradigm~\citep{muller2021transformers,hollmann2022tabpfn}, we construct most of \textsc{Wander}'s pretraining distribution synthetically. For attributed graphs, we combine two approaches inspired by \textsc{NodePFN}~\citep{choi2026learning} and \textsc{GraphPFN}~\citep{eremeev2025graphpfn}: \textsc{NodePFN} generates attributes through an SCM before sampling graph structure using stochastic block or Erd\H{o}s--R\'enyi models~\citep{holland1983stochastic,erdos1959random}, while \textsc{GraphPFN} first samples graph structure and subsequently generates attributes with a graph-aware SCM. We also diversify the structural prior using Watts--Strogatz~\citep{watts1998collective}, random geometric~\citep{gilbert1961random}, and kernel-based graph models~\citep{lovasz2006limits}. These synthetic graphs provide both node-classification and homogeneous link-prediction episodes. 
Since realistic synthetic generation of multi-relational graphs is less straightforward, we additionally include FB15k-237~\citep{FB15k-237}, WN18RR~\citep{WN18RR}, and CoDEx Medium~\citep{codex} for multi-relational link-prediction pretraining. Details and discussion of the pretraining prior are provided in Appendix~\ref{app:pretraining-prior}.

\begin{table}[t]

  \small
      \centering
      \caption{Knowledge graph link prediction (MRR and Hits@10). }
      \label{tab:link_prediction_summary}
      \setlength{\tabcolsep}{3pt}
      \begin{tabular}{lcccccc|cccc}
      \toprule
      \multirow{2}{*}{}
        & \multicolumn{2}{c}{Inductive ($e$,$r$)}
        & \multicolumn{2}{c}{Inductive ($e$)}
        & \multicolumn{2}{c|}{Transductive}
        & \multicolumn{2}{c}{Mean}
        & \multicolumn{2}{c}{Mean rank} \\
      & MRR & H@10 & MRR & H@10 & MRR & H@10 & MRR & H@10 & MRR & H@10 \\
      \midrule
      {\# graphs}
        & \multicolumn{2}{c}{23}
        & \multicolumn{2}{c}{18}
        & \multicolumn{2}{c|}{13}
        & 
        &  \\
      \midrule
      \textsc{Ultra} & 0.345 & 0.513 & 0.431 & 0.566 & 0.312 & 0.458 & 0.366 & 0.518 & 3.250 & 3.417 \\
      \textsc{Trix} & 0.368 & 0.540 & 0.455 & 0.592 & 0.339 & 0.500 & 0.390 & 0.548 & 2.657 & 2.750 \\
      \textsc{Flock} & \underline{0.369} & \underline{0.554} & \underline{0.456} & \underline{0.604} & \underline{0.340} & \underline{0.509} & \underline{0.391} & \underline{0.560} & \underline{2.343} & \underline{2.222} \\
      \midrule
      \rowcolor{wanderlight!40!white}\textbf{\textsc{Wander}} & \textbf{0.378} & \textbf{0.558} & \textbf{0.460} & \textbf{0.609} & \textbf{0.353} & \textbf{0.516} & \textbf{0.399} & \textbf{0.565} & \textbf{1.750} & \textbf{1.611} \\
       \bottomrule
      \end{tabular}
  \end{table}

\paragraph{Knowledge graph link prediction.}
We evaluate whether \textsc{Wander} can match foundation models specialized
to knowledge graph reasoning by comparing against \textsc{Ultra}~\citep{ultra},
\textsc{Trix}~\citep{trix}, and \textsc{Flock}~\citep{kim2025flock} in the
zero-shot setting. We consider entity prediction, where, given a head entity
\(h\) and relation \(r\), the model ranks candidate tails for the query
\((h,r,?)\). All models are pretrained on the same three knowledge graphs included in \textsc{Wander}'s pretraining. We evaluate on 54 knowledge graphs spanning diverse domains and three settings: inductive on entities and relations (inductive \(e,r\)), inductive on entities (inductive \(e\)), and transductive. Note that in our experiment, all three settings are in principle indifferent, as all models encounter unseen relations in the inference graph and align with previous work~\citep{ultra,kim2025flock}. We use the filtered ranking
protocol~\citep{bordes2013translating} and report mean reciprocal rank (MRR)
and Hits@10; per-dataset results are given in Table~\ref{tab::full_kg}.
As shown in Table~\ref{tab:link_prediction_summary}, \textsc{Wander} achieves
the strongest overall results, with a mean MRR of \(0.399\) and Hits@10 of
\(0.565\), compared with \(0.391\) and \(0.560\) for the strongest specialized
baseline \textsc{Flock}.

\paragraph{Homogeneous link prediction.}
Next, we evaluate \textsc{Wander} on homogeneous link prediction, where, given
a source node, the model ranks all other nodes as candidate link targets. We
consider ten datasets from diverse domains (see
Appendix~\ref{app::homo_lp_eval}) and report Recall@20. We compare against three classes of baselines: the classical heuristics Common
Neighbors (CN)~\citep{liben2003link} and Resource Allocation
(RA)~\citep{zhou2009predicting}; the end-to-end link-prediction models
\textsc{Buddy}~\citep{chamberlain2022graph} and \textsc{NBFNet}~\citep{zhu2022neural}, tuned
per dataset;
and the graph foundation model \textsc{AnyGraph}~\citep{xia2024anygraph}.

Two recent foundation models, \textsc{UniLP}~\citep{dong2024universal} and
\textsc{TFMLinker}~\citep{liao2026tfmlinker}, also target homogeneous link prediction
but omit node features and score node pairs individually. In contrast,
\textsc{Wander} conditions on a source node and jointly scores all candidate
targets. Because pairwise scoring makes full-candidate Recall@20 impractical
at this scale, we compare with \textsc{UniLP} separately in
Appendix~\ref{app::homo_lp_eval}.

Table~\ref{tab:homo-lp} summarizes the results. \textsc{Wander} achieves the
highest mean Recall@20 of \(43.39\%\), compared with \(38.82\%\) for the
strongest end-to-end baseline, \textsc{NBFNet}, and \(27.85\%\) for
\textsc{AnyGraph}. Thus, \textsc{Wander} remains competitive with
dataset-specific models while substantially outperforming the existing
foundation-model baseline.

\begin{table}[t]
        \fontsize{8.5pt}{10pt}\selectfont
        \centering
        \caption{Homogeneous link prediction (Recall@20 \%).}
        \label{tab:homo-lp}
        \setlength{\tabcolsep}{3pt}
        \begin{tabular}{lcccccccccc|cc}
        \toprule
         & CiteSeer & Cora & PubMed & CS & DDI & \makecell{Products\\Home} & \makecell{P2P\\Gnutella} & \makecell{Email\\Enron} & \makecell{Proteins\\Spec1} & \makecell{SOC\\Epinions} & Mean & Mean rank \\
        \midrule
        CN & 29.83 & 33.30 & 11.02 & 49.09 & 3.45 & 56.94 & 1.46 & 52.88 & 15.65 & 15.73 & 26.94 & 5.00 \\
        RA & 30.26 & 34.99 & 11.12 & 55.63 & 4.85 & 63.21 & 1.29 & 63.07 & 20.67 & 16.37 & 30.15 & 4.10 \\
        \midrule
        \textsc{Buddy} & \underline{44.32} & 40.03 & \underline{25.72} & 48.28 & 9.42 & 58.66 & 7.10 & 50.18 & 13.67 & 10.78 & 30.82 & 4.00 \\
        \textsc{NBFNet} & 34.21 & \underline{45.98} & 22.88 & \underline{62.17} & \textbf{16.34} & \underline{68.19} & \textbf{9.29} & \underline{69.56} & \textbf{39.85} & \underline{19.70} & \underline{38.82} & \underline{1.90} \\
        \midrule
        \textsc{AnyGraph} & 32.35 & 44.98 & 15.59 & 47.87 & 7.67 & 55.04 & 5.80 & 40.75 & 15.36 & 13.06 & 27.85 & 4.70 \\
        \midrule
        \rowcolor{wanderlight!40!white}\textbf{\textsc{Wander}} & \textbf{65.88} & \textbf{56.78} & \textbf{31.73} & \textbf{64.62} & \underline{12.98} & \textbf{71.12} & \underline{9.06} & \textbf{69.78} & \underline{31.40} & \textbf{20.59} & \textbf{43.39} & \textbf{1.30} \\
        \bottomrule
        \end{tabular}
    \end{table}
\begin{table}[t]
      \scriptsize
      \centering
      \caption{Node classification (test accuracy \%, averaged over all available splits).}
      \label{tab:node-cls-multiclass-summary}
      \setlength{\tabcolsep}{3pt}
      \begin{tabular}{l*{14}{c}|cc}
      \toprule
  
      & \makecell{Web-\\KB}
      & \makecell{Air-\\ports}
      & \makecell{Wiki\\filt.}
      & Actor
      & \makecell{Plane-\\toid}
      & \makecell{Ama-\\zon}
      & \makecell{Wiki\\CS}
      & \makecell{Cit.\\Full}
      & \makecell{Roman\\Emp.}
      & \makecell{Amz.\\Rat.}
      & \makecell{Coau-\\thor}
      & \makecell{HM\\Cat.}
      & \makecell{Pokec\\100k}
      & \makecell{City\\Netw.}
      & \makecell{Mean}
      & \makecell{Mean\\rank} \\
      \midrule
      \# graphs
        & 3 & 3 & 2 & 1 & 3 & 2 & 1
        & 2 & 1 & 1 & 2 & 1 & 1 & 3
        &  &  \\
      mean \# nodes
        & 206 & 573 & 1.6k & 7.6k & 8.6k & 10.7k & 11.7k
        & 18.8k & 22.7k & 24.5k & 26.4k & 46.6k & 100k & 180k
        &  &  \\
      \midrule
      \textsc{GCN}
        & 51.3 & 46.9 & 37.4 & 33.1 & 70.8 & 87.9 & \textbf{79.4}
        & \underline{66.2} & \textbf{79.6} & \textbf{51.2} & \underline{92.0} & \textbf{66.4} & 78.4 & 38.2
        & 60.64 & 3.12 \\
        \midrule
      \textsc{TabPFNv3}
        & 73.3 & 32.0 & 44.0 & \textbf{39.0} & 67.2 & 83.1 & 73.8
        & 59.2 & 65.7 & \underline{49.1} & 90.8 & 42.1 & 43.2 & \underline{47.0}
        & 58.66 & 3.69 \\
        \midrule
      \textsc{GraphAny}
        & 63.0 & 40.2 & 29.2 & 29.1 & 74.7 & 86.9 & 75.4
        & 64.2 & 64.2 & 42.5 & 91.7 & 37.1 & 44.0 & 18.9
        & 54.85 & 3.96 \\
      $\textsc{NodePFN}^*$
          & \underline{74.5} & 59.1 & 44.3 & 32.7 & \textbf{76.2} & 85.7 & 76.0
        & --- & 48.8 & 44.2 & 91.7 & --- & 78.1 & 44.6
        & --- & --- \\
      \textsc{GraphPFN}
        & 73.6 & \textbf{65.2} & \textbf{49.5} & \underline{38.7} & 71.9 & \underline{88.2} & 79.1
        & 66.0 & \underline{67.8} & 48.6 & 91.1 & \underline{55.2} & \underline{83.8} & 45.1
        & \underline{66.55} & \underline{2.42} \\
        \midrule
      \rowcolor{wanderlight!40!white}\textbf{\textsc{Wander}}
        & \textbf{78.5} & \underline{64.0} & \underline{45.9} & 36.4 & \underline{75.4} & \textbf{88.3} & \underline{79.3}
        & \textbf{69.9} & 65.4 & 49.0 & \textbf{92.5} & 48.7 & \textbf{86.1} & \textbf{53.8}
        & \textbf{68.20} & \textbf{1.73} \\
      \bottomrule
      \end{tabular}
      {$^*$\text{NodePFN} is omitted from columns
  containing datasets that exceed its maximum number of supported classes.}
  \end{table}
\paragraph{Node classification on simple graphs.}
Finally, to evaluate whether \textsc{Wander} can match models specialized to node
classification, we evaluate it in the zero-shot setting on a test set of 26
attributed simple graphs from diverse domains (see
Appendix~\ref{app::node_classification} for details). We consider a strongly tuned
GCN~\citep{kipf2016semi} as a supervised baseline, performing a per-dataset
grid search over learning rate, layer normalization, dropout, and skip
connections. We further compare against the graph foundation models
\textsc{GraphAny}~\citep{zhao2024graphany},
\textsc{NodePFN}~\citep{choi2026learning}, and
\textsc{GraphPFN}~\citep{eremeev2025graphpfn}, as well as the tabular
foundation model \textsc{TabPFN}v3~\citep{tabpfn3}. While
\textsc{GraphAny} was pretrained on the ogbn-arxiv graph from
\citet{hu2020open}, all other foundation models were pretrained on synthetic
priors.
\Cref{tab:node-cls-multiclass-summary} summarizes the results, with
datasets from the same family grouped into a single column. Overall,
\textsc{Wander} closely matches the performance of the strongest
node-classification foundation model, \textsc{GraphPFN}, achieving a mean
accuracy of \(68.20\%\), surpassing \(66.55\%\).

\subsection{Transfer through joint training (Q2)}

\begin{wraptable}{r}{0.4\columnwidth}
\vspace{-10pt}
    \centering
    \footnotesize
    \caption{Joint vs.\ single-task training.}
    \label{tab:joint-training}
    \begin{tabular}{lccc}
        \toprule
         & \makecell[c]{KG LP\\(MRR)} & \makecell[c]{Hom. LP\\
        (R@20)} & \makecell[c]{NC\\(Acc.)} \\
        \midrule
        \textsc{Wander} & 0.399 & 43.39 & 68.20\\
        \textsc{Wander} \tiny \textit{one task} & 0.398 & 40.37 & 67.69 \\
        \bottomrule
    \end{tabular}
    \end{wraptable}
We next investigate whether joint pretraining across tasks causes interference
or enables positive transfer. For each task, we train a single-task variant for
exactly the number of steps allocated to that task during joint pretraining. This keeps the
number of task-specific updates fixed, while the jointly trained model
additionally receives updates from the other tasks.
As shown in Table~\ref{tab:joint-training}, knowledge-graph link prediction as well as node classification are
essentially unaffected by joint training. For homogeneous link prediction, however, joint training improves
Recall@20 from \(40.37\%\) to \(43.39\%\). This suggests that joint pretraining
does not necessarily trade off task specialization and can instead yield
positive transfer across graph-learning settings. The effect is particularly
pronounced for homogeneous link prediction, which shares the link-prediction
objective with knowledge graphs while operating on the homogeneous attributed
graphs also encountered in node classification.

\subsection{Composing separately learned capabilities (Q3)}
\begin{wraptable}{l}{0.42\columnwidth}
    \vspace{-8pt}
    \centering
    \footnotesize
    \caption{KG entity prediction (MRR) with vs.\ without node features and edge types.}
    \label{tab:kg-with-features}
    \begin{tabular}{ccccc}
        \toprule
          \makecell{edge\\types} & \makecell{fea-\\tures} & \makecell[c]{Concept-\\Net100k} & \makecell[c]{WN-v1-\\WN-v4} & \makecell[c]{WN-\\18RR} \\
        \midrule
          \xmark & \xmark & 0.135 & 0.162 & 0.116 \\
          \xmark & \checkmark & 0.202 & 0.273 & 0.163 \\
          \checkmark & \xmark & \underline{0.258} & \underline{0.626} & \underline{0.550} \\
          \checkmark & \checkmark & \textbf{0.334 }& \textbf{0.637} & \textbf{0.564} \\
        \bottomrule
    \end{tabular}
    \end{wraptable}
During pretraining, \textsc{Wander} encounters node features on simple graphs and edge types on knowledge graphs, but never both simultaneously. To test whether it can combine these separately learned capabilities, we augment several knowledge graphs with node features obtained by embedding entity names using MiniLM \citep{wang2020minilm}. As shown in Table~\ref{tab:kg-with-features}, both sources of information improve entity prediction individually, while using them jointly performs best on all three benchmarks. This indicates that \textsc{Wander} can compositionally generalize to attributed, edge-typed graphs, even though this type of graph was never encountered
during training.
 
\subsection{Inference-Time Adaptation (Q4)}
\label{sec:inference_adaptation}

A distinctive property of \textsc{Wander} is that its structural computation
remains adjustable after training: we can vary the number and length of
sampled walks, as well as the sampling mechanism, without changing learned
parameters. We investigate whether this flexibility enables adaptation to
different structural requirements at inference time (\textbf{Q4}).

\begin{wrapfigure}{r}{0.4\textwidth}
    \centering
    \includegraphics[width=\linewidth]{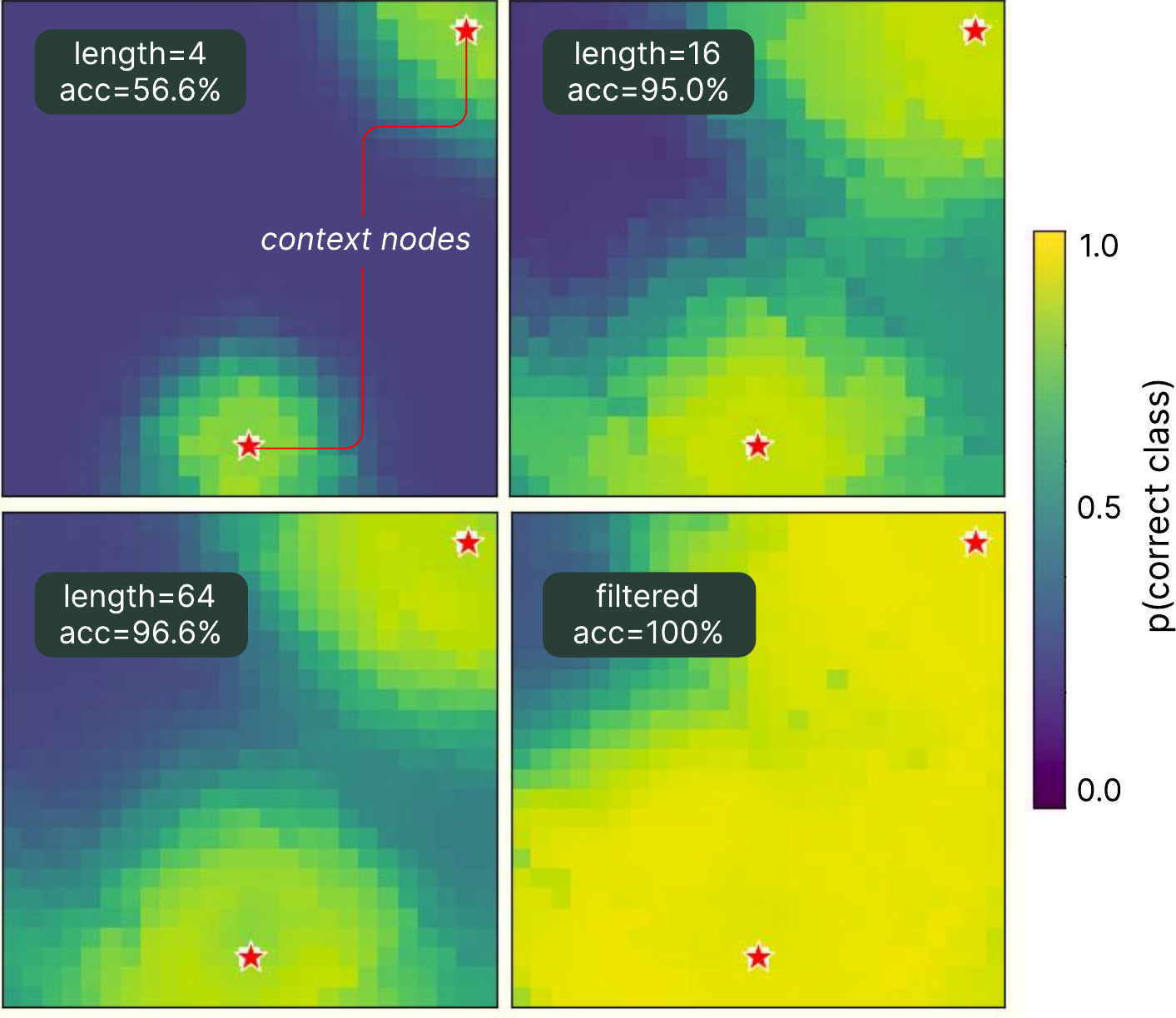}
    \caption{Predicted true-class probability on single grid taken from \textsc{Grids} for increasing walk lengths and context-aware walk filtering.}
    \label{fig:grids_experiment}
\end{wrapfigure}
\paragraph{Adapting the reasoning range.}
We first consider \textsc{Grids}, a synthetic node-classification task designed
to require long-range propagation. It consists of five disconnected
\(25\times25\) grids, one per class, with only two randomly labeled nodes per
grid; test nodes are on average 11 hops from their nearest labeled node.
Figure~\ref{fig:grids_experiment} shows the predicted true-class probability when
evaluating the same pretrained model with walk lengths \(4\), \(16\), and
\(64\). Increasing the walk length progressively propagates information from
the labeled context nodes across the graph, raising accuracy from \(56.6\%\)
at length \(4\) to \(96.6\%\) at length \(64\). At this setting,
\textsc{Wander} also substantially outperforms all GFM baselines as shown in
Table~\ref{tab:long_range_summary}.

Inference-time adaptation is not limited to changing the walk budget. As a
simple example, we additionally modify the sampling mechanism by retaining only
length-\(64\) walks that visit at least one labeled context node. This targeted
sampling further increases accuracy on \textsc{Grids} to \(100\%\), illustrating
that \textsc{Wander} can adapt not only \textit{how much} but also \textit{which} information is
collected.

\paragraph{Long-range supervision in real-world graphs.}
\begin{wraptable}{l}{0.35\columnwidth}
\vspace{-10pt}
    \centering
    \scriptsize
    \caption{Accuracy on \textsc{Grids} and four real-world graphs without node features.}
    \label{tab:long_range_summary}
    \setlength{\tabcolsep}{4pt}
    \begin{tabular}{lcc}
        \toprule
        & \textsc{Grids}
        & Real-world \\
        \midrule
        \textsc{GraphAny}
        & 20.0
        & 34.0 \\
        \textsc{NodePFN}
        & 26.6
        & 41.2 \\
        \textsc{GraphPFN}
        & 23.3
        & 60.1 \\
        \midrule
        \textbf{\textsc{Wander}}
        & 96.6
        & 76.5 \\
        \textbf{\textsc{Wander} \tiny(filt.)}
        & \textbf{100.0}
        & \textbf{78.3} \\
        \bottomrule
    \end{tabular}
\end{wraptable}

We next test whether the ability to exploit distant supervision extends beyond
the synthetic setting. We consider PubMed, Full DBLP, Coauthor CS, and
Coauthor Physics, whose training splits provide comparatively sparse label
supervision. We remove all node features,
leaving graph structure and observed training labels as the only predictive
signals. As shown in Table~\ref{tab:long_range_summary}, \textsc{Wander} achieves the
best performance across the four datasets. This suggests that the long-range
reasoning capability highlighted by \textsc{Grids} is also useful on
real-world graphs with sparse supervision.

\section{Conclusion}
We introduced \textsc{Wander}, a graph foundation model that uses random
walks as a shared interface across graph modalities and prediction tasks.
By framing graph learning as completion of a partially observed graph,
a single pretrained model supports node- and edge-level prediction on
homogeneous and multi-relational graphs. Under suitable assumptions,
\textsc{Wander} universally approximates the corresponding Bayes-optimal
predictor. A single checkpoint achieves state-of-the-art or highly
competitive performance across node classification, homogeneous link
prediction, and knowledge-graph link prediction. More broadly, our results
suggest that graph foundation models need not be organized around a fixed
modality or task: a shared interface can support
capabilities learned in one setting and reused in another, opening a path
toward composing them across settings.

The current pretraining prior remains limited: most attributed graphs
are synthetic, and multi-relational pretraining uses only three real-world
knowledge graphs. Inference can also be costly, particularly with large
random-walk budgets and repeated sampling. Reducing these budgets speeds
up prediction without retraining, but may reduce predictive performance.
Future work should broaden the prior, particularly toward graphs combining
node features and edge types, extend the framework to node regression and
graph classification, and develop more efficient, adaptive sampling strategies.

\clearpage
\section*{Acknowledgments}
AITHYRA is supported by the Austrian Academy of Sciences and the not-for-profit Boehringer Ingelheim Stiftung.
\bibliography{references}
\bibliographystyle{plainnat}
\clearpage
\appendix
\section{Methodology details}\label{app::method}
We here provide further details on the implementation of \textsc{Wander} used for our experiments. Note that the statements in Section~\ref{sec::theory} refer to the more general framework described in Section~\ref{sec::method}.
\paragraph{Initialization.}
At the beginning of each forward pass, all hidden states are initialized independently of the input graph identities. All relation types, including a \textit{start} relation used at the beginning of random walks, receive the same learnable initial embedding. Structure tokens are initialized with one of three learnable embeddings according to the node's structural relation to the link-prediction query head, when present: identical, neighboring, or neither. For node classification, where no query head is defined, all nodes receive the \emph{neither} embedding. Feature tokens are initialized by applying a two-layer MLP to their scalar feature value. Finally, label tokens are initialized with one of three learnable embeddings, corresponding to an observed positive label, an observed negative label, or an unknown label.

\paragraph{Intra-node update.}
The intra-node module exchanges information between the channels of each node independently. For a node $v\in V$, let
\[
\mathbf{V}^{(l)}_v
:=
\bigl(
    \mathbf{v}^{(l)}(v,1),\ldots,
    \mathbf{v}^{(l)}(v,d_X+d_Y+1)
\bigr)
\in\mathbb{R}^{(d_X+d_Y+1)\times d}
\]
denote the sequence of its feature, label, and structure tokens. We update this sequence using a pre-normalized residual multi-head attention layer,
\[
\mathbf{V}^{(l)\prime}_v
=
\mathbf{V}^{(l)}_v
+
\operatorname{MHA}_{\mathrm{RoPE}}
\!\left(
    \operatorname{RMSNorm}(\mathbf{V}^{(l)}_v)
\right),
\]
where the same module is applied independently to every $v\in V$. We use rotary positional embeddings (RoPE) \citep{su2024roformer} to associate the individual channels with distinct positions and enable the model to distinguish two channels with distributionally similar values. Since fixed positional encodings would break equivariance with respect to permutations of feature and label channels, we uniformly randomize the ordering of feature and label channels in each forward pass and assign RoPE positions only after this permutation. The resulting model therefore remains equivariant in distribution with respect to feature and label permutations. We use RMSNorm \citep{zhang2019root}.

\paragraph{ICL-update.}
The attention-to-context module instead exchanges information between nodes while keeping channels separate. For each channel $i\in[d_X+d_Y+1]$, let
\[
\mathbf{V}^{(l)\prime}_i
:=
\bigl(
    \mathbf{v}^{(l)\prime}(v,i)
\bigr)_{v\in V}
\in\mathbb{R}^{|V|\times d}.
\]
We independently update each channel using a pre-normalized Transformer block,
\[
\widetilde{\mathbf{V}}^{(l)}_i
=
\mathbf{V}^{(l)\prime}_i
+
\operatorname{MHA}_{\mid V_{\mathrm{context}}}
\!\left(
    \operatorname{RMSNorm}(\mathbf{V}^{(l)\prime}_i)
\right),
\]
followed by
\[
\mathbf{V}^{(l)\prime\prime}_i
=
\widetilde{\mathbf{V}}^{(l)}_i
+
\operatorname{FFN}
\!\left(
    \operatorname{RMSNorm}(\widetilde{\mathbf{V}}^{(l)}_i)
\right).
\]
Here, the attention is masked such that every node may attend only to keys and values corresponding to nodes in $V_{\mathrm{context}}$, while queries are computed for all nodes. We use separate attention modules for the three channel types (structural, feature, and label) while sharing the respective module across all channels of the same type.  The feed-forward network (FFN) is implemented as a SwiGLU~\citep{shazeer2020glu} block.

We next describe the random walk algorithm and then how we sample walks during training. Sampling during inference is explained in Appendix \ref{app::exp}. 

\paragraph{Random walk algorithm.}
Let $v_0\in V$ be a starting node given as input to the random walk algorithm. The algorithm constructs a walk $\eta=(v_0,r_1^{\alpha_1},v_1,\ldots,r_S^{\alpha_S},v_S)$
iteratively. At step $j\geq 1$, given the current node $v_{j-1}$, we define its set of  neighbors as
\[
N(v_{j-1})
:=
\{u\in V \mid \exists r\in R:\,(v_{j-1},r,u)\in E \ \lor\ (u,r,v_{j-1})\in E\}.
\]
Note that we include incoming neighbors in this definition. To avoid immediate backtracking whenever possible, we define the set of admissible next nodes as
\begin{align}\label{non-backtracking}
    N_j :=
\begin{cases}
N(v_0), & \text{if }j=1,\\
N(v_{j-1})\setminus\{v_{j-2}\},
&\text{if } j>1 \text{ and } N(v_{j-1})\setminus\{v_{j-2}\}\neq\emptyset,\\
N(v_{j-1}),
&\text{if } j>1 \text{ and } N(v_{j-1})=\{v_{j-2}\}.
\end{cases}
\end{align}
We first sample the next node uniformly, $v_j \sim \operatorname{Unif}(N_j)$.
Given $v_j$, let
\[
R(v_{j-1},v_j)
:=
\{(r,1)\mid (v_{j-1},r,v_j)\in E\}\cup \{(r,-1)\mid (v_j,r,v_{j-1})\in E\}
\]
denote the set of relation--direction pairs connecting $v_{j-1}$ to $v_j$. We then sample $(r_j,\alpha_j) \sim \operatorname{Unif}\bigl(R(v_{j-1},v_j)\bigr)$.

Thus, neighboring nodes are sampled independently of the number of relation types connecting them, and immediate backtracking occurs only if the previously visited node is the only available neighbor.

\paragraph{Choice of random walks.}
Let $K_{\text{base}}$ denote the base walk budget, which may depend on the input graph
(see Appendix~\ref{app::implementation}). For each structural update, we
sample $3K_{\text{base}}$ walks, with their starting locations chosen according to the
prediction task.

For a link-prediction query $q=(h_q,r_q,?)$, we follow
\citet{kim2025flock} and divide the walks into three equally sized groups.
The first $K_{\text{base}}$ walks start at the query head $h_q$, providing structural
context local to the query. For the next $K_{\text{base}}$ walks, we first sample a
relation $s\sim\mathrm{Unif}(R)$ and then an edge
$(v,s,u)$ uniformly among the observed edges of type $s$; this edge forms
the first step of the walk, after which the usual random-walk transition
rule is applied. Sampling the relation before the edge prevents frequent
relation types from dominating the walk budget and ensures coverage of
broad relations. The remaining $K_{\text{base}}$ walks
start from nodes sampled uniformly from $V$, providing broader,
query-independent coverage of the graph.

For node prediction query $V_q$, we likewise sample $3K_{\text{base}}$ walks. We start $\lceil 3K_{\text{base}}/2\rceil$ walks from nodes
sampled uniformly from $V_q$ and the remaining $\lfloor 3K_{\text{base}}/2\rfloor$ from
nodes sampled uniformly from $V_{\mathrm{context}}$. 

\paragraph{Recording protocol.}
Let $\eta=(v_0,r_1^{\alpha_1},v_1,\ldots,r_S^{\alpha_S},v_S)$ be a sampled walk, and let
$V(\eta)\subseteq V$ and $R(\eta)\subseteq R$ denote the sets of nodes and
relation types occurring in $\eta$, respectively.
We define bijections
\[
    \operatorname{anon}_{\eta,V}:V(\eta)\to[|V(\eta)|],
    \qquad
    \operatorname{anon}_{\eta,R}:R(\eta)\to[|R(\eta)|],
\]
which assign identifiers according to the order of first appearance in
$\eta$.
For each walk position, we first construct a shared walk embedding that is
independent of the channel. At layer $l$, the shared embedding associated
with node $v$, incoming relation $r\in R\cup\{\texttt{start}\}$, direction $\alpha$ and query
$q$ is
\begin{align}
    \operatorname{step}^{(l)}_{\eta}(v,r^{\alpha},q)
    ={}&
    e^{(l)}_{\mathrm{node}}\!\left(\operatorname{anon}_{\eta,V}(v)\right)
    + e^{(l)}_{\mathrm{rel}}\!\left(\operatorname{anon}_{\eta,R}(r)\right)
    + e^{(l)}_{\text{dir}}(\alpha)
    \nonumber\\
    &+
    e^{(l)}_{\mathrm{queryhead}}\!\left(\mathbf{1}_{h_q=v}\right)
    + e^{(l)}_{\mathrm{queryrel}}\!\left(\mathbf{1}_{r_q=r}\right)
    + W^{(l)}_{\mathrm{rel}}\,\mathbf{r}^{(l)}(r),
\end{align}
where the query
indicators are set to zero whenever the corresponding query component is
not defined (i.e., for a label prediction query $V_q\subset V$). For the relation \texttt{start} we set direction $\alpha=1$.
The embedding tables $e^{(l)}_{\ast}$ and projection
$W^{(l)}_{\mathrm{rel}}$ are layer-specific.

For each channel $i\in[d_X+d_Y+1]$, we combine this shared walk embedding
with the current hidden state of the corresponding node--channel pair:
\[
    \mathbf{z}^{(l)}_{\eta,v,r,\alpha,i}
    =
    W^{(l)}_{\mathrm{in},\tau(i)}
    \mathbf{v}^{(l)\prime\prime}(v,i)
    +
    g^{(l)}_{\tau(i)}
    \operatorname{step}^{(l)}_{\eta}(v,r^{\alpha},q),
\]
where
$\tau(i)\in\{\mathrm{feature},\mathrm{label},\mathrm{structure}\}$
denotes the channel type.
Both $W^{(l)}_{\mathrm{in},\tau(i)}$ and the learnable scalar gate
$g^{(l)}_{\tau(i)}$ are layer- and channel-type-specific, but shared across
all channels of the same type.

For each walk $\eta$ and channel $i$, the position embeddings are stacked
in walk order to obtain
\[
    Z^{(l)}_{\eta,i}
    =
    \left(
        \mathbf{z}^{(l)}_{\eta,v_0,\mathrm{start},i},
        \mathbf{z}^{(l)}_{\eta,v_1,r_1,\alpha_1,i},
        \ldots,
        \mathbf{z}^{(l)}_{\eta,v_S,r_S,\alpha_S,i}
    \right).
\]
Nodes and relation types may occur multiple times in this sequence whenever
they are revisited by the walk.

\paragraph{Sequence processor.} Each sequence is then processed by a bidirectional GRU \citep{cho2014learning} followed by a
SwiGLU feed-forward network, both with pre-normalization and residual
connections:
\begin{align*}
    \widetilde Z^{(l)}_{\eta,i}
    &=
    Z^{(l)}_{\eta,i}
    +
    W^{(l)}_{\mathrm{GRU}}
    \operatorname{BiGRU}^{(l)}
    \left(
        \operatorname{RMSNorm}(Z^{(l)}_{\eta,i})
    \right),
    \\
    Z^{(l)\prime}_{\eta,i}
    &=
    \widetilde Z^{(l)}_{\eta,i}
    +
    \operatorname{FFN}^{(l)}
    \left(
        \operatorname{RMSNorm}(\widetilde Z^{(l)}_{\eta,i})
    \right).
\end{align*}
We use separate modules for the three channel types (structural, feature, and label) while sharing the respective module across all channels of the same type. Finally, linear output projections are applied independently at every
sequence position. Let $\mathbf{z}'\in\mathbb{R}^{d}$ denote the processed
representation at a position associated with node $v$ and channel $i$.
For every channel type, we produce a multi-head node--channel update proposal
and confidence vector,
\begin{align*}
    \Delta_{v,i}
    &=
    \operatorname{reshape}_{d_H\times H}
    \left(
        W^{(l)}_{\Delta,\tau(i)}\mathbf{z}'
    \right),
    &
    a_{v,i}
    &=
    W^{(l)}_{a,\tau(i)}\mathbf{z}',
\end{align*}
where $d_H\cdot H=d$, such that
$\Delta_{v,i}\in\mathbb{R}^{d_H\times H}$ and
$a_{v,i}\in\mathbb{R}^{H}$.
The output projections are layer- and channel-type-specific and are shared
across channels of the same type.

For the structure channel
$i=d_X+d_Y+1$, the same processed position additionally
proposes an update to its incoming relation $r\in
R\cup\{\mathrm{start}\}$:
\begin{align}
    \Delta_r
    &=
    \operatorname{reshape}_{d_H\times H}
    \left(
        W^{(l)}_{\Delta,\mathrm{rel}}\mathbf{z}'
    \right),
    &
    a_r
    &=
    W^{(l)}_{a,\mathrm{rel}}\mathbf{z}'.
\end{align}
Thus, every position proposes an update to its corresponding node--channel
state, while structure-channel positions additionally propose an update to
the hidden state of their incoming relation type. Since nodes and relation
types may occur repeatedly within and across sampled walks, each hidden
state can receive multiple proposals.

\paragraph{Consensus protocol.}The sequence processor yields for all pairs $(v,i)\in V\times [d_X+d_Y+1]$ as well as relation types $r\in R\cup\{\texttt{start}\}$ a (possibly empty) set $\mathcal{P}=\{(\Delta^{(m)},a^{(m)})\}_{m=1}^M$, with
$\Delta^{(m)}=(\Delta^{(m,1)},\ldots,\Delta^{(m,H)})$ and $a^{(m)}=(a^{(m,1)},\ldots,a^{(m,H)})$, of update proposals and confidence scores.
The consensus protocol then forms a multi-head softmax-normalized weighted average over each set of update proposals to decide on the final update:
\[
c(\mathcal{P})
:=
\bigg\Vert_{h=1}^{H}
\sum_{m'=1}^{M}
\frac{\exp({a^{(m',h)})}}{\sum \exp(a^{(m,h)})}\Delta^{(m',h)},
\]
where $\Vert$ denotes concatenation. For $\mathcal P=\emptyset$, we set
$c(\mathcal P)=\mathbf 0$. Finally, hidden states are updated in a residual way, i.e., $\textbf{v}^{(l+1)}(v,i) =\textbf{v}^{(l)\prime\prime}(v,i) + c(\mathcal{P}_{v,i}) $ and $\textbf{r}^{(l+1)}(r) =\textbf{r}^{(l)}(r) + c(\mathcal{P}_{r}) $ for all $v\in V$, $i\in [d_X+d_Y+1]$ and $r\in R\cup\{\texttt{start}\}$.
\paragraph{Readout.}
For node classification, we apply a shared two-layer MLP
$\text{out}_{\mathrm{node}}:\mathbb{R}^d\rightarrow\mathbb{R}$ independently to
each final label token. Specifically, for a query node $v$ and class
$c\in[d_Y]$, we obtain the class logit
\[
\ell_{v,c}
=
\text{out}_{\mathrm{node}}\!\left(
    \mathbf v^{(L)}(v,d_X+c)
\right).
\]
A softmax over the $d_Y$ class logits yields the corresponding class
probabilities.

For link prediction with query $q=(h_q,r_q,?)$, we score every candidate
tail $t\in V$ using its final structure token together with the final state
of the query relation. A shared two-layer MLP
$\text{out}_{\mathrm{link}}:\mathbb{R}^{d}\rightarrow\mathbb{R}$ produces
\[
\ell_{h_q,r_q,t}
=
\text{out}_{\mathrm{link}}\!\left(
    \mathbf v^{(L)}(t,d_X+d_Y+1)
    +
    \mathbf r^{(L)}(r_q)
\right).
\]
 We convert this logit into a link
probability
$p_\theta(h_q,r_q,t)=\sigma(\ell_{h_q,r_q,t})$.

\paragraph{Loss.}
For node-classification training steps, we minimize the cross-entropy loss over
the predicted class logits of all query nodes $V_q$. The loss is
\[
\mathcal L_{\mathrm{NC}}
=
-\frac{1}{|V_q|}
\sum_{v\in V_q}
\log
\frac{\exp(\ell_{v,Y_v})}
     {\sum_{c=1}^{d_Y}\exp(\ell_{v,c})}.
\]

For link-prediction steps, each positive triple $(h,r,t)$ is paired
with $k$ negative tails $t^-_1,\ldots,t^-_k$. Following
\citet{kim2025flock}, we use self-adversarial negative weighting, which
assigns larger weight to negative candidates receiving larger model
scores. Writing $\tau_{\mathrm{adv}}>0$ for the adversarial temperature,
we define
\[
w_i^{\mathrm{adv}}
=
\frac{
    \exp\!\left(\ell_{h,r,t_i^-}/\tau_{\mathrm{adv}}\right)
}{
    \sum_{j=1}^{k}
    \exp\!\left(\ell_{h,r,t_j^-}/\tau_{\mathrm{adv}}\right)
}.
\]
The resulting binary cross-entropy loss is
\[
\mathcal L_{\mathrm{LP}}
=
-\log p_\theta(h,r,t)
-
\sum_{i=1}^{k}
w_i^{\mathrm{adv}}
\log\!\left(1-p_\theta(h,r,t_i^-)\right).
\]
The losses are averaged over all link-prediction queries in the batch. Note that, due to the self-adversarial weighting used for link prediction, the practical objective no longer has \(M_p^\star\) exactly as its population minimizer.

\section{Comparison to related (graph) foundation models}
\label{app:comparison}
While the framework stated in Section~\ref{sec::method} is stated more generally, in our implementation we instantiate the intra-node processor with self-attention and the ICL update with masked attention (see Appendix~\ref{app::method}). The resulting first two updates therefore follow the two-dimensional attention pattern common to recent tabular foundation models, alternating interactions within samples and across samples \citep{tabpfnv2}. 

In the node-classification setting, this also yields an architectural connection to \textsc{GraphPFN}~\citep{eremeev2025graphpfn}. Both architectures combine within-node feature processing and in-context information exchange with a graph-structural update. However, \textsc{GraphPFN} incorporates structure through adjacency-masked attention over neighbors, whereas \textsc{Wander} uses sampled random walks and a shared sequence processor as its structural interface. This distinction provides two notable advantages. First, random walks allow \textsc{Wander} to capture structural information beyond the expressive scope of standard message-passing architectures, as formalized in Theorem~\ref{thm:two-layer-universality} and empirically investigated in Section~\ref{app::cycles}. Second, the amount of structural information available to the model can be adjusted at inference time by varying the number and length of sampled walks, without modifying the learned parameters.

In the setting of knowledge graph link prediction, \textsc{Wander} is instead closely related to the knowledge graph foundation model \textsc{Flock}~\citep{kim2025flock}. While \textsc{Flock} likewise uses random walks as its primary structural inductive bias, it does not include an explicit in-context update that allows query representations to attend to context nodes. 

\section{Theoretical analysis}\label{app::proofs}
In this section, we provide proofs of the statements from Sections \ref{sec::bayes-predictor} and \ref{sec::theory}.

\subsection{Bayes-Optimal Predictor}
\label{app:bayes-optimal}

We start by proving Theorem~\ref{thm:prior-expansion}.

\begin{proof}[Proof of Theorem~\ref{thm:prior-expansion}]
For an observed graph \(G_{\mathrm{obs}}\), let
\(Z_p(G_{\mathrm{obs}})\) and \(Z_r(G_{\mathrm{obs}})\)
denote its marginal likelihood (probability mass or density) under
\(p\) and \(r\), respectively.
For \(p_\lambda=(1-\lambda)p+\lambda r\), its marginal likelihood is
\[
Z_{p_\lambda}(G_{\mathrm{obs}})
=
(1-\lambda)Z_p(G_{\mathrm{obs}})
+
\lambda Z_r(G_{\mathrm{obs}}).
\]
Hence, whenever \(Z_p(G_{\mathrm{obs}})>0\), the posterior under the mixture
prior can be written as
\[
p_\lambda(H\mid G_{\mathrm{obs}})
=
\bigl(1-\beta_{\lambda,r}(G_{\mathrm{obs}})\bigr)
p(H\mid G_{\mathrm{obs}})
+
\beta_{\lambda,r}(G_{\mathrm{obs}})
r(H\mid G_{\mathrm{obs}}),
\]
where
\[
\beta_{\lambda,r}(G_{\mathrm{obs}})
=
\frac{\lambda Z_r(G_{\mathrm{obs}})}
{(1-\lambda)Z_p(G_{\mathrm{obs}})
 + \lambda Z_r(G_{\mathrm{obs}})}.
\]
If \(Z_r(G_{\mathrm{obs}})=0\), then
\(\beta_{\lambda,r}(G_{\mathrm{obs}})=0\), and by convention
\(M_r^\star=M_p^\star\) on such observations. Taking the posterior
expectation of the query target therefore gives
\[
M_{p_\lambda}^\star(G_{\mathrm{obs}},q)
=
\bigl(1-\beta_{\lambda,r}(G_{\mathrm{obs}})\bigr)
M_p^\star(G_{\mathrm{obs}},q)
+
\beta_{\lambda,r}(G_{\mathrm{obs}})
M_r^\star(G_{\mathrm{obs}},q).
\]

For brevity, write
\[
\beta := \beta_{\lambda,r}(G_{\mathrm{obs}}),\qquad
M_p := M_p^\star(G_{\mathrm{obs}},q),\qquad
M_r := M_r^\star(G_{\mathrm{obs}},q),
\]
and \(T:=T_q(G^\star)\).
By convexity of \(\ell_q\) in its prediction argument,
\[
\ell_q\bigl((1-\beta)M_p+\beta M_r,T\bigr)
\leq
(1-\beta)\ell_q(M_p,T)
+
\beta\ell_q(M_r,T).
\]
Subtracting \(\ell_q(M_p,T)\) and taking expectation under \(\nu\) yields
\begin{align*}
R_\nu(M_{p_\lambda}^\star)-R_\nu(M_p^\star)
&\leq
\mathbb{E}_\nu\!\left[
    \beta_{\lambda,r}(G_{\mathrm{obs}})
    \left(
        \ell_q(M_r^\star(G_{\mathrm{obs}},q),T_q(G^\star))
        -
        \ell_q(M_p^\star(G_{\mathrm{obs}},q),T_q(G^\star))
    \right)
\right] \\
&=
\mathbb{E}_\nu\!\left[
    \beta_{\lambda,r}(G_{\mathrm{obs}})
    \,
    \mathbb{E}_\nu\!\left[
        \begin{aligned}
        &\ell_q(M_r^\star(G_{\mathrm{obs}},q),T_q(G^\star)) \\
        &\quad -
        \ell_q(M_p^\star(G_{\mathrm{obs}},q),T_q(G^\star))
        \end{aligned}
        \,\middle|\, G_{\mathrm{obs}},q
    \right]
\right] \\
&=
\mathbb{E}_\nu\!\left[
    \beta_{\lambda,r}(G_{\mathrm{obs}})
    \Delta_\nu(G_{\mathrm{obs}},q)
\right].
\end{align*}
where the second equality uses that
\(\beta_{\lambda,r}(G_{\mathrm{obs}})\) depends only on the observation.
The support assumption on \(p\) ensures that the above quantities are
defined \(\nu\)-almost surely, which proves the claim.
\end{proof}

The bound makes explicit that an added graph family matters only to the
extent that it receives posterior weight on the target observation. In
particular, adding \(r\) cannot increase the target risk if
\(\Delta_\nu(G_{\mathrm{obs}},q)>0\) only on observations for which
\(\beta_{\lambda,r}(G_{\mathrm{obs}})=0\). More generally, beneficial and
harmful contributions are weighted by the posterior relevance of the added
family.
\begin{proof}[Proof of Proposition \ref{prop:bayes-is-equivariant}]
Let $\pi=(\pi_V,\pi_R,\pi_X,\pi_Y)$. Since $p$ is a distribution over
isomorphism classes, $p(\pi(H))=p(H)$ for every $H$. Moreover,
\[
G \preceq H
\quad\Longleftrightarrow\quad
\pi(G) \preceq \pi(H),
\]
so $H\mapsto\pi(H)$ gives a bijection between the completions of $G$ and
those of $\pi(G)$. Hence
\[
p(\pi(H)\mid \pi(G)) = p(H\mid G).
\]
By definition of the query targets,
\[
T_{\pi(q)}(\pi(H))
=
\pi_{\mathrm{out}}\!\left(T_q(H)\right).
\]
Therefore,
\begin{align*}
M_p^\star(\pi(G),\pi(q))
&=
\mathbb{E}_{H\sim p(\cdot\mid \pi(G))}
    [T_{\pi(q)}(H)] \\
&=
\mathbb{E}_{H\sim p(\cdot\mid G)}
    [T_{\pi(q)}(\pi(H))] \\
&=
\pi_{\mathrm{out}}\!\left(
    \mathbb{E}_{H\sim p(\cdot\mid G)}[T_q(H)]
\right) \\
&=
\pi_{\mathrm{out}}\!\left(M_p^\star(G,q)\right),
\end{align*}
which proves the claim.
\end{proof}

\subsection{Universality}
Before we can prove Theorem \ref{thm:two-layer-universality}, we need to formalize the notion of \textit{admissible priors.}

\paragraph{Admissible graph prior.}
Let $p$ be a probability distribution over the isomorphism
classes of $\mathbb{G}_{n,m,b}^\star$.
For every isomorphism class $G$, we separate its discrete
attributes from its continuous attributes. Denote by
$\kappa_G$ its \emph{discrete isomorphism class}, which fixes
the graph structure, relation types, feature dimension, and
all discrete-valued attributes, and by
\[
    z_G \in \mathcal{Z}_{\kappa_G}
    \subseteq \mathbb{R}^{d_{\kappa_G}}
\]
its continuous attributes. We can then write the prior as
\[
    p(G)=p(\kappa_G)\,p_{\kappa_G}(z_G),
    \qquad
    p_{\kappa_G}(z_G):=p(z_G\mid\kappa_G).
\]
We call $p$ an \emph{admissible graph prior} if, for every
discrete isomorphism class $\kappa$ with $p(\kappa)>0$,
the domain $\mathcal{Z}_\kappa$ is a compact product domain
and $p_\kappa$ admits a continuous and strictly positive
density on $\mathcal{Z}_\kappa$. Moreover, for each fixed
discrete observed graph-query class $\bar{\kappa}$, all
compatible positive-mass completions share the same
observed-attribute domain: after partitioning $z=(x,u)$
into observed and unobserved attributes,
\[
    \mathcal{Z}_\kappa
    =
    \mathcal{X}_{\bar{\kappa}}
    \times \mathcal{U}_{\bar{\kappa},\kappa}.
\]
Throughout, we fix an observation mechanism that reveals
a subset of attributes, with selection probabilities
depending only on the discrete completion, not on the
continuous attributes.

\paragraph{Continuous Bayes-optimal predictors.}
We now show that the Bayes-optimal predictor introduced
in Section~\ref{sec::preliminary} is well-behaved under
admissible priors. This is required later to show that
it can be approximated.

\begin{lemma}\label{lem:bayes-continuity}
Let $p$ be an admissible graph prior. Then, within every
fixed discrete isomorphism class of observed graph-query
pairs, the Bayes-optimal predictor
\[
    M_p^\star(G_{\mathrm{obs}},q)
    =
    \mathbb{E}_{G^\star \sim p(\cdot\mid G_{\mathrm{obs}})}
    \left[T_q(G^\star)\right]
\]
is continuous in the continuous attributes observed in
$G_{\mathrm{obs}}$ for link prediction and node classification.
\end{lemma}

\begin{proof}
Fix a discrete observed graph-query class $\bar{\kappa}$
and write $x\in\mathcal{X}_{\bar{\kappa}}$ for its observed
continuous attributes. Let $\mathcal{K}$ be the finite,
nonempty set of positive-mass discrete completions that
can produce this observation. For each $\kappa\in\mathcal{K}$,
define
\[
    p_{\kappa,X}(x)
    :=
    \int_{\mathcal{U}_{\bar{\kappa},\kappa}}
    p_\kappa(x,u)\,du.
\]
Continuity on the compact product domain implies that
$p_{\kappa,X}$ is continuous on the common domain
$\mathcal{X}_{\bar{\kappa}}$. Strict positivity of
$p_\kappa$ implies $p_{\kappa,X}(x)>0$ there.

Let $a_\kappa>0$ denote $p(\kappa)$ multiplied by the
probability of the fixed discrete observation given
$\kappa$. For link prediction and node classification,
each target coordinate is determined by the discrete
completion; denote its value by $c_{\kappa,j}\in\{0,1\}$.
Bayes' rule gives
\[
    M_{p,j}^\star(x,q)
    =
    \frac{
        \sum_{\kappa\in\mathcal{K}}
        a_\kappa c_{\kappa,j}p_{\kappa,X}(x)
    }{
        \sum_{\kappa\in\mathcal{K}}
        a_\kappa p_{\kappa,X}(x)
    }.
\]
The numerator and denominator are finite sums of continuous
functions on $\mathcal{X}_{\bar{\kappa}}$, and the denominator
is strictly positive there. Hence every output coordinate
is continuous.
\end{proof}
\paragraph{Proof of Theorem \ref{thm:two-layer-universality}.} 
We are now able to prove Theorem \ref{thm:two-layer-universality} from Section~\ref{sec::theory}, here stated in a more detailed form. Note that we prove universality of the \textsc{Wander} \textit{framework} under the assumptions below. This does not necessarily correspond to the implementation details given in Appendix~\ref{app::method}.
\begin{proposition3*}
Fix $n,m,b\in\mathbb{N}$. Assume a \textsc{Wander} configuration with at least two layers,
a walk length $S$ of at least $n$, and which fulfills the following assumptions on the random walk sampler and model capacity:
\begin{enumerate}
    \item[\textnormal{(A1)}] each $v\in V$ and each $r\in R$ has positive probability of being chosen as the starting node or the first traversed relation type of a walk, respectively, every non-backtracking walk of length at most $S$ has positive probability of occurring as a prefix of a sampled walk and all walks are sampled independently and have probabilities bounded away from zero by constants depending only on $n$ and $m$;
    \item[\textnormal{(A2)}] after the intra-node update, every node token
    can retain an injective encoding of the complete collection of tokens of
    that node;
    \item[\textnormal{(A3)}] at every sequence position, the sequence
    processor can retain an injective encoding of the complete recorded walk
    together with the distinguished position;
    \item[\textnormal{(A4)}] the consensus protocol can retain an injective
    encoding of the multiset of update proposals associated with a hidden
    state, while preserving its previous state; and
    \item[\textnormal{(A5)}] the output head is a universal approximator of
    continuous functions on compact domains.
\end{enumerate}
Then \textsc{Wander} is a universal graph prior approximator for graph priors over $\mathbb{G}_{n,m,b}^\star$.
\end{proposition3*}

Note that the context-attention module may be chosen as the identity and
plays no role in this argument.

We start with a simple observation, which will help us in proving that with high probability, random walks break all symmetries.

\begin{lemma}[Random counts break finite symmetries]
\label{lem:random-counts}
Let $U$ be a finite set and let
$X_1,\ldots,X_K$ be i.i.d.\ samples from a distribution $p$ satisfying
$p(u)>0$ for every $u\in U$. For $u\in U$, define
\[
    C_u^{(K)}
    :=
    \sum_{j=1}^K \mathbf{1}_{\{X_j=u\}}.
\]
Then
\[
    \Pr\!\left[
        C_u^{(K)}\neq C_{u'}^{(K)}
        \text{ for all distinct }u,u'\in U
    \right]
    \longrightarrow 1
    \qquad\text{as }K\to\infty.
\]
\end{lemma}

\begin{proof}
Fix distinct $u,u'\in U$.

First suppose $p(u)\neq p(u')$. By the strong law of large numbers,
\[
    \frac{C_u^{(K)}-C_{u'}^{(K)}}{K}
    \longrightarrow
    p(u)-p(u')
    \neq 0
    \qquad\text{almost surely}.
\]
Hence
\(
    \Pr[C_u^{(K)}=C_{u'}^{(K)}]
    \rightarrow 0.
\)
Now suppose $p(u)= p(u')=\beta>0$. Let
\(
    M_K:=C_u^{(K)}+C_{u'}^{(K)}.
\)
Then $M_K\sim\operatorname{Bin}(K,2\beta)$ and, conditional on $M_K=\ell$,
\(
    C_u^{(K)}
    \sim
    \operatorname{Bin}(\ell,1/2).
\)
Consequently,
\[
    \Pr\!\left[
        C_u^{(K)}=C_{u'}^{(K)}
        \mid M_K=\ell
    \right]
    \leq
    \binom{\ell}{\lfloor \ell/2\rfloor}2^{-\ell}
    =
    O(\ell^{-1/2}).
\]
Since $M_K\to\infty$ in probability, it follows that
\(
    \Pr\![C_u^{(K)}=C_{u'}^{(K)}]
    \rightarrow 0.
\) Since $U$ is finite, a union bound over all unordered pairs
$\{u,u'\}\subseteq U$ gives the claim.
\end{proof}
Now we first show that with sufficiently many walks, all nodes and relation types are uniquely identifiable after a single layer.
\begin{lemma}[First-layer node and relation identifiers]
\label{lem:first-layer-identifiers}
Under assumptions \textnormal{(A1)}--\textnormal{(A4)}, for every
$n,m\in\mathbb{N}$ and every $\delta>0$, there exists a finite number
$K_1=K_1(n,m,\delta)$ of first-layer walks such that, with probability at
least $1-\delta$, after the first structural update:
\begin{enumerate}
    \item every node $v\in V$ contains a unique identifier
    $\iota_V(v)$,
    \item every relation type $r\in R$ contains a unique identifier
    $\iota_R(r)$, and
    \item the original observed node attributes and labels are preserved.
\end{enumerate}
\end{lemma}

\begin{proof}
For each node $v\in V$, let
\[
    C_v
    :=
    \#\{j\in[K_1] : \eta_j\text{ starts at }v\},
\]
and for every relation type $r\in R$, let
\[
    D_r
    :=
    \#\{j\in[K_1] :
        r\text{ is the first relation traversed by }\eta_j\}.
\]

By assumption (A1), the starting node of a sampled walk follows
a categorical distribution assigning strictly positive probability to every
$v\in V$. Similarly, the first traversed relation follows a categorical
distribution assigning strictly positive probability to every $r\in R$.
Lemma~\ref{lem:random-counts} therefore implies
\[
    \Pr[C_v\neq C_{v'}\text{ for all }v\neq v']
    \longrightarrow 1
\]
and
\[
    \Pr[D_r\neq D_{r'}\text{ for all }r\neq r']
    \longrightarrow 1
\]
as $K_1\to\infty$.

The beginning of a walk is explicitly distinguished by the start-relation
record, while the first traversed relation is identifiable from its position
in the recorded sequence. By assumption \textnormal{(A3)}, the corresponding
update proposals may therefore encode whether a node occurs as the starting
node and whether a relation occurs as the first traversed relation.
By assumption \textnormal{(A4)}, the consensus protocol may retain the
multiplicity of these proposals. Consequently, the updated node state may
retain $C_v$, and the updated relation state may retain $D_r$.

We may therefore take
\[
    \iota_V(v):=C_v,
    \qquad
    \iota_R(r):=D_r.
\]
With probability tending to one, these values are injective on $V$ and $R$,
respectively.

Finally, assumption \textnormal{(A2)} ensures that the complete local node
state, including all observed features and labels, can be retained before
the structural update, while assumption \textnormal{(A4)} allows the
structural update to preserve this information. Hence the random identifiers
do not overwrite the original node information.

Since $|V|\leq n$ and $|R|\leq m$, the required number of walks can be
chosen uniformly over all considered graphs in $\mathcal{G}_{n,m,b}$.
\end{proof}

Next we will show that when choosing the walk length to be at least the diameter of the graph, for any node-edge pair there always exists a walk that traverses both.
\begin{lemma}[Short non-backtracking edge witnesses]
\label{lem:edge-witness}
Let $G=(V,E,R,X,Y)$ have connected underlying undirected graph with
diameter $D$. For every target node $t\in V$ and every edge
$e=(u,r,v)\in E$, there exists a non-backtracking walk starting at $t$,
of length at most
\[
    D+1 \leq |V| \leq n,
\]
that traverses $e$ in one of its two directions.
\end{lemma}

\begin{proof}
Fix $t\in V$ and $e=(u,r,v)\in E$. Let
\[
    P=(t=v_0,v_1,\ldots,v_\ell=u)
\]
be a shortest path in the underlying undirected graph from $t$ to $u$.
Then $\ell\leq D$. Since $P$ is a shortest path, it contains no immediate
backtracking.

If either $\ell=0$ or $v_{\ell-1}\neq v$, append the traversal from $u$ to
$v$ along $e$. The resulting walk is non-backtracking and has length
\[
    \ell+1\leq D+1.
\]

If instead $v_{\ell-1}=v$, then the final step of $P$ already traverses
the edge $e$, namely from $v$ to $u$. Since traversal direction is recorded,
this already identifies the original directed edge $(u,r,v)$. In this case
$P$ itself is the required witness and has length at most $D$.

Thus in either case there exists a non-backtracking witness walk of length
at most $D+1$. Since a connected graph on at most $n$ nodes satisfies
$D\leq |V|-1$, the witness has length at most $n$.
\end{proof}

\begin{lemma}[Simultaneous edge witnessing]
\label{lem:simultaneous-witnessing}
Assume that the walk length $S$ is at least $n$. For every $n,m\in\mathbb{N}$
and every $\delta>0$, there exists a finite number
$K_2=K_2(n,m,\delta)$ of second-layer walks such that, with probability at
least $1-\delta$, for every node $t\in V$ and every edge $e\in E$, at least
one sampled walk starts at $t$ and contains a traversal of $e$.
\end{lemma}

\begin{proof}
By Lemma~\ref{lem:edge-witness}, for every pair $(t,e)\in V\times E$ there
exists an admissible non-backtracking witness walk of length at most $n$.

By assumption (A1) on the sampler, a walk that contains such a prefix has probability of at least $\beta_{n,m}>0$.

Therefore,
\[
    \Pr[
        (t,e)\text{ is not witnessed in }K_2\text{ walks}
    ]
    \leq
    (1-\beta_{n,m})^{K_2}.
\]
Since
\(
    |E|\leq |V|^2|R|\leq n^2m,
\)
there are at most $n^3m$ target-edge pairs. A union bound yields
\[
    \Pr[
        \exists (t,e)\in V\times E
        \text{ that is not witnessed}
    ]
    \leq
    n^3m(1-\beta_{n,m})^{K_2}.
\]
The right-hand side converges to zero as $K_2\to\infty$, proving the claim.
\end{proof}
\begin{lemma}[Two-layer graph reconstruction]
\label{lem:two-layer-reconstruction}
Suppose the events in Lemmas~\ref{lem:first-layer-identifiers}
and~\ref{lem:simultaneous-witnessing} both occur. Then, after the second
structural update, the hidden state associated with every node $t\in V$
can retain an encoding from which the complete observed graph $G$ and, in the case of link prediction, the query $q$ can be reconstructed up to isomorphism.
\end{lemma}

\begin{proof}
By Lemma~\ref{lem:first-layer-identifiers}, after the first layer every node
and relation type carries a persistent identifier
\[
    \iota_V:V\to\mathbb{N},
    \qquad
    \iota_R:R\to\mathbb{N},
\]
which is injective on $V$ and $R$, respectively. The first-layer node states also retain all observed node attributes and labels.

Consider now the second structural layer and fix a target node $t$.
By Lemma~\ref{lem:simultaneous-witnessing}, for every edge $e=(u,r,v)\in E$
there exists a sampled walk $\eta_e$ that starts at $t$ and traverses $e$.

Because the first-layer identifiers are part of the hidden states used by
the recording protocol, the recording of $\eta_e$ contains enough
information to identify the traversal of $e$ as
\[
    \bigl(
        \iota_V(u),
        \iota_R(r),
        \iota_V(v)
    \bigr),
\]
together with its orientation. In addition, whenever a node occurs in a
walk, its recorded hidden state contains its observed local attributes.

By assumption \textnormal{(A3)}, the update proposal produced at the
starting position $t$ may retain an injective encoding of the complete
recorded walk $\eta_e$. Thus, the proposal associated with $t$ contains
the information that the edge
\[
    \bigl(
        \iota_V(u),
        \iota_R(r),
        \iota_V(v)
    \bigr)
\]
exists.

By assumption \textnormal{(A4)}, the consensus update at $t$ can retain
injectively the multiset of all such proposals. Since every edge is
witnessed by at least one walk starting at $t$, the resulting state can
therefore recover
\[
    \left\{
        \bigl(
            \iota_V(u),
            \iota_R(r),
            \iota_V(v)
        \bigr)
        : (u,r,v)\in E
    \right\}.
\]
Because $\iota_V$ and $\iota_R$ are injective, this determines the complete
edge-typed graph up to a renaming of nodes and relation types.

Since the underlying graph is connected, every node is incident to an edge. Hence the collection of edge-witnessing walks also contains
the first-layer state of every node and therefore determines all observed
node features and labels.

Finally, the recording protocol contains the link query indicators. It follows that the second-layer state
at $t$ determines $G$, and in the case of link prediction $q$, up to isomorphism.
\end{proof}

We are now ready to prove the main result.
\begin{proof}[Proof of Theorem \ref{thm:two-layer-universality}]
Fix $\epsilon,\delta>0$. By Lemmas~\ref{lem:first-layer-identifiers} and~\ref{lem:simultaneous-witnessing}, the probabilities of
successful symmetry breaking in the first layer and complete edge witnessing
in the second layer both converge to one as the number of sampled walks grows.
Hence, there exists a common finite sampling budget $K\in\mathbb N$ such that,
when $K$ walks are sampled in each layer, the event in Lemma~\ref{lem:first-layer-identifiers} and the event
in Lemma~\ref{lem:simultaneous-witnessing} each occur with probability at least $1-\delta/2$.
By a union bound, both events therefore occur simultaneously with probability
at least $1-\delta$.

Condition on such a successful realization of the sampled walks.
By Lemma~\ref{lem:two-layer-reconstruction}, the second-layer hidden state
used by the output head determines the observed graph-query pair $(G,q)$
up to isomorphism.

Fix now a discrete isomorphism class $\bar\kappa$ of connected observed
graph-query pairs. Let $\mathcal K(\bar\kappa)$ denote the set of discrete
isomorphism classes of fully specified graphs that are compatible with
$\bar\kappa$ and have positive prior probability. Since $n$, $m$, and $b$
are bounded, $\mathcal K(\bar\kappa)$ is finite.

For every $\kappa\in\mathcal K(\bar\kappa)$, let
$\mathcal X_{\bar\kappa,\kappa}$ denote the projection of
$\mathcal Z_\kappa$ onto the continuous coordinates observed in
$\bar\kappa$, and define
\[
    \mathcal X_{\bar\kappa}
    :=
    \bigcup_{\kappa\in\mathcal K(\bar\kappa)}
    \mathcal X_{\bar\kappa,\kappa}.
\]
Each $\mathcal X_{\bar\kappa,\kappa}$ is compact as the projection of the
compact set $\mathcal Z_\kappa$. Since the union is finite,
$\mathcal X_{\bar\kappa}$ is compact. Let
$x\in\mathcal X_{\bar\kappa}$ denote the continuous observed attributes.

For a fixed successful realization $\omega$ of the discrete random choices
made by the model, let
\[
    \Phi_{\bar\kappa,\omega}:
    \mathcal X_{\bar\kappa}
    \longrightarrow \mathbb R^d
\]
denote the resulting representation supplied to the output head. The model
modules are continuous, so $\Phi_{\bar\kappa,\omega}$ is continuous.
Moreover, by Lemma~\ref{lem:two-layer-reconstruction}, the representation
determines the corresponding graph-query pair up to isomorphism. After
fixing a canonical representative of $\bar\kappa$, we may therefore regard
$\Phi_{\bar\kappa,\omega}$ as injective.

A continuous injection from a compact space into $\mathbb R^d$ is a
homeomorphism onto its image. Hence
\[
    \Phi_{\bar\kappa,\omega}^{-1}
\]
is continuous on
$\Phi_{\bar\kappa,\omega}(\mathcal X_{\bar\kappa})$.

Let $M_{p,\bar\kappa}^\star$ denote the restriction of the Bayes-optimal
predictor to the observed discrete class $\bar\kappa$, viewed as a function
of its continuous observed attributes $x$. By
Lemma~\ref{lem:bayes-continuity}, $M_{p,\bar\kappa}^\star$ is continuous.
Consequently,
\[
    h
    \longmapsto
    M_{p,\bar\kappa}^\star
    \!\left(\Phi_{\bar\kappa,\omega}^{-1}(h)\right)
\]
is continuous on the compact representation domain
$\Phi_{\bar\kappa,\omega}(\mathcal X_{\bar\kappa})$.

Because $n$, $m$, and $b$ are bounded, there are only finitely many
discrete isomorphism classes of connected observed graph-query pairs.
For a fixed number and length of sampled walks, there are also only
finitely many discrete walk realizations. Hence the successful
representation domains form a finite union of compact sets. Whenever two
such domains intersect, Lemma~\ref{lem:two-layer-reconstruction} implies
that the shared representation determines the same graph-query pair up to
isomorphism, and therefore the same Bayes-optimal output. The continuous
functions defined above thus agree on intersections and together define a
continuous function on this finite union of compact representation domains.

By assumption \textnormal{(A5)}, the output head can approximate this
function uniformly to accuracy $\epsilon$. Thus, whenever the two
successful sampling events occur,
\[
    \left\|
        M_{\theta,K}(G,q)-M_p^\star(G,q)
    \right\|
    <\epsilon.
\]
Since the numbers of walks were chosen uniformly over
$\mathcal G_{n,m,b}$ and the two successful events occur jointly with
probability at least $1-\delta$, we obtain
\[
    \Pr\!\left[
        \left\|
            M_{\theta,K}(G,q)-M_p^\star(G,q)
        \right\|<\epsilon
    \right]
    \geq 1-\delta
\]
uniformly over all connected $G\in\mathcal G_{n,m,b}$ and all valid
queries $q$. Thus, for every $\epsilon,\delta>0$, there exist network parameters
$\theta$ and a finite sampling budget $K$ satisfying the definition of a
universal graph prior approximator.
\end{proof}

\subsection{Permutation equivariance}

\begin{proof}[Proof of Proposition \ref{prop:permutation-equivariance}]
Let $\pi=(\pi_V,\pi_R,\pi_X,\pi_Y)$ and let
$(G',q')=(\pi(G),\pi(q))$. We show inductively over the layers that the joint
distribution of the hidden states is equivariant under $\pi$: node states
are permuted according to $\pi_V$, their feature and label channels
according to $\pi_X$ and $\pi_Y$, respectively, while relation states are
permuted according to $\pi_R$.

The claim holds at initialization by construction: feature and label values
are carried along with their corresponding channels, structural states
depend only on the node's role relative to the query, and relation states
are independent of relation identities.

Assume now that the claim holds at the beginning of layer $l$.
By assumption, the intra-node processor $\mu_\theta$ is equivariant in
distribution under permutations of feature and label channels. Since the
same processor is applied independently to every node, the intra-node
update also commutes with $\pi_V$.

The in-context update acts independently on each channel and aggregates
over the corresponding context nodes. Since
\[
V_{\mathrm{context}}(G',q')
=
\pi_V\!\left(V_{\mathrm{context}}(G,q)\right),
\]
permuting nodes and channels therefore only permutes the corresponding
queries, keys, and values, and the update remains equivariant.

For the structural update, permutation invariance of $A$ implies that the
walks sampled on $(G',q')$ have the same distribution as the images under
$\pi$ of the walks sampled on $(G,q)$. The recording protocol consequently
produces correspondingly permuted walk recordings. Since the same sequence
processor is applied to every recording, corresponding walk positions
produce corresponding update proposals. Finally, the consensus protocol is
permutation invariant in the collection of proposals associated with each
state, so the resulting node and relation states are again equivariant in
distribution. This completes the induction.

After the final layer, the shared node head maps permuted label channels to
the correspondingly permuted classification outputs. Likewise, the shared
link head maps permuted candidate-node and relation states to the
correspondingly permuted link-prediction outputs. Hence
\[
M_\theta\!\left(\pi(G),\pi(q)\right)
\overset{d}{=}
\pi_{\mathrm{out}}\!\left(M_\theta(G,q)\right),
\]
which proves the claim.
\end{proof}

\subsection{Computational complexity}\label{app:complexity}

Let \(N=|V|\), \(C=d_X+d_Y+1\) be the number of channels,
\(N_{\mathrm{ctx}}=|V_{\mathrm{context}}|\), \(K\) the number of sampled
walks per layer, \(S\) their length, and \(d\) the hidden dimension.
For a single \textsc{Wander} layer, intra-node self-attention costs
\[
    \mathcal{O}\!\left(N(C^2d+Cd^2)\right),
\]
since attention is applied independently across the \(C\) tokens of each
node. The ICL update uses the \(N\) node tokens as queries and the
\(N_{\mathrm{ctx}}\) context tokens as keys and values, independently for
each channel, and therefore costs
\[
    \mathcal{O}\!\left(
        C N N_{\mathrm{ctx}}d + C(N+N_{\mathrm{ctx}})d^2
    \right).
\]
Finally, the stochastic structural update processes \(KC\) sequences of
length \(S+1\). With the bidirectional GRU and feed-forward network used in
our implementation, this requires
\[
    \mathcal{O}\!\left(KCSd^2\right)
\]
operations. Sampling the walks and aggregating their proposals add only
\(\mathcal{O}(KS)\) and \(\mathcal{O}(KCSd)\), respectively, and are
therefore dominated by sequence processing.

Consequently, the overall time complexity of one layer is
\[
\mathcal{O}\!\left(
    NC^2d
    + C N N_{\mathrm{ctx}}d
    + \bigl(NC+C N_{\mathrm{ctx}}+KCS\bigr)d^2
\right)
\]
which simplifies to  
\begin{align*}
    \mathcal{O}\!\left(NC^2+CNN_{\text{ctx}}+KCS
\right)
\end{align*}
for fixed hidden dimension $d$.
In particular, for a fixed architecture and hidden dimension, the
structural update scales linearly in both the number and length of the
sampled walks. This makes its computation independent of the total number
of edges except for adjacency-list construction and permits the
accuracy--cost trade-offs obtained by varying the walk budget at inference
time.

\section{Experimental details}\label{app::exp}

\subsection{Graph prior}\label{app:pretraining-prior}
The training objective stated in Section~\ref{sec::bayes-predictor} assumes access to samples from a graph prior $p$.
Ideally, this prior would capture the distribution of real-world graphs directly.
In practice, however, available graph datasets form a comparatively small and
heterogeneous collection. We therefore follow the prior-data fitted network
paradigm introduced by \citet{muller2021transformers} and construct most of our pretraining
distribution synthetically.

A central ingredient of \textsc{TabPFN}~\citep{hollmann2022tabpfn} is a prior over \emph{structural causal models}
(SCMs), which generate variables by sampling a causal dependency graph together
with functions mapping parent variables and noise to their descendants.
Selecting some generated variables as features and another as the target yields
a broad synthetic distribution over supervised learning problems.

\paragraph{Synthetic attributed graphs.}
Recent GFMs extend this idea to attributed graphs, but make different
assumptions about how graph structure and node attributes arise.
\textsc{NodePFN} \citep{choi2026learning} follows an
\emph{attribute-first} construction: features and labels are first sampled from
an MLP-based SCM, after which the graph is generated using either a
label-dependent stochastic block model \citep{holland1983stochastic} or an Erd\H{o}s--R\'enyi \citep{erdos1959random} model.
\textsc{GraphPFN} \citep{eremeev2025graphpfn} instead follows a
\emph{structure-first} construction: the topology is sampled from multi-level
degree-corrected stochastic block models and preferential-attachment graphs,
and features and targets are subsequently generated by a graph-aware SCM.

These constructions capture complementary dependencies: attributes may
influence which relations form, while graph structure may itself influence
node attributes. We therefore include both generation directions in our prior.
We additionally broaden the structural distribution with Watts--Strogatz \citep{watts1998collective},
random geometric \citep{gilbert1961random}, and kernel-based graph models \citep{lovasz2006limits} to cover structural regimes
beyond block-community graphs, including highly clustered and geometry-driven
networks. Concretely, each synthetic graph is attribute-first with probability $0.4$ and structure-first with probability $0.6$.
Structure-first topologies are a multi-level degree-corrected stochastic block model with preferential attachment ($0.45$), a random geometric graph in latent dimension $2$--$10$ ($0.30$), or a Watts--Strogatz graph with rewiring probability in $[0.01,0.3]$ ($0.25$).
Attribute-first topologies are a label-ordered Watts--Strogatz graph ($0.45$), a label stochastic block model ($0.35$), or an RBF kernel graph ($0.20$).
The label stochastic block model and the RBF kernel use a homophily ratio that is homophilous (centered at $0.8$) with probability $0.7$ and heterophilous (centered at $0.1$) otherwise.
Graph size is mixed log-uniform, half in $[200,3000]$ nodes and half in $[3000,16000]$, and the average degree is log-uniform on $[2.5,20]$ with probability $0.9$.
Homogeneous link-prediction episodes reuse these undirected graphs with labels removed; node features are dropped with probability $0.5$, and a fraction of edges drawn from $[0.7,0.9]$ is kept as context.

The resulting attributed graphs are used for both node-classification
and homogeneous link-prediction pretraining.

\paragraph{Comparison with other synthetic priors.}
To characterize the structural regimes covered by our prior, we sample
1,000 graphs from each of the \textsc{Wander}, \textsc{NodePFN}, and
\textsc{GraphPFN} priors and compare their statistics with those of the
real-world evaluation graphs. Figure~\ref{fig:prior_scatter} shows that
\textsc{Wander} covers a wider range of clustering coefficients than the
other synthetic priors, including graphs with substantial clustering at
low average degree. It also spans both homophilous and heterophilous regimes,
although its samples are concentrated toward homophily. The distributions
in Figure~\ref{fig:prior_distributions} provide complementary comparisons of
graph size, average degree, clustering, homophily, and number of classes.
To test the effect of these differences on the transfer to real world graphs we replace the synthetic prior during
pretraining in Appendix~\ref{app::prior_experiments}.

\paragraph{Multi-relational graphs.}
Extending such synthetic mechanisms to multi-relational graphs is less
straightforward, as relation types in knowledge graphs follow different and
often domain-specific generative processes. We therefore augment the synthetic
attributed-graph prior with three real-world knowledge graphs: FB15k-237
\citep{FB15k-237}, WN18RR~\citep{WN18RR}, and
CoDEx Medium~\citep{codex}. These graphs provide multi-relational
link-prediction episodes during pretraining, exposing \textsc{Wander} jointly
to synthetic attributed graphs and real-world multi-relational structure. Statistics about these three real-world knowledge graphs are provided in Table~\ref{tab:transductive_datasets}.

\begin{figure}[htbp]
    \centering
    \includegraphics[width=\linewidth]{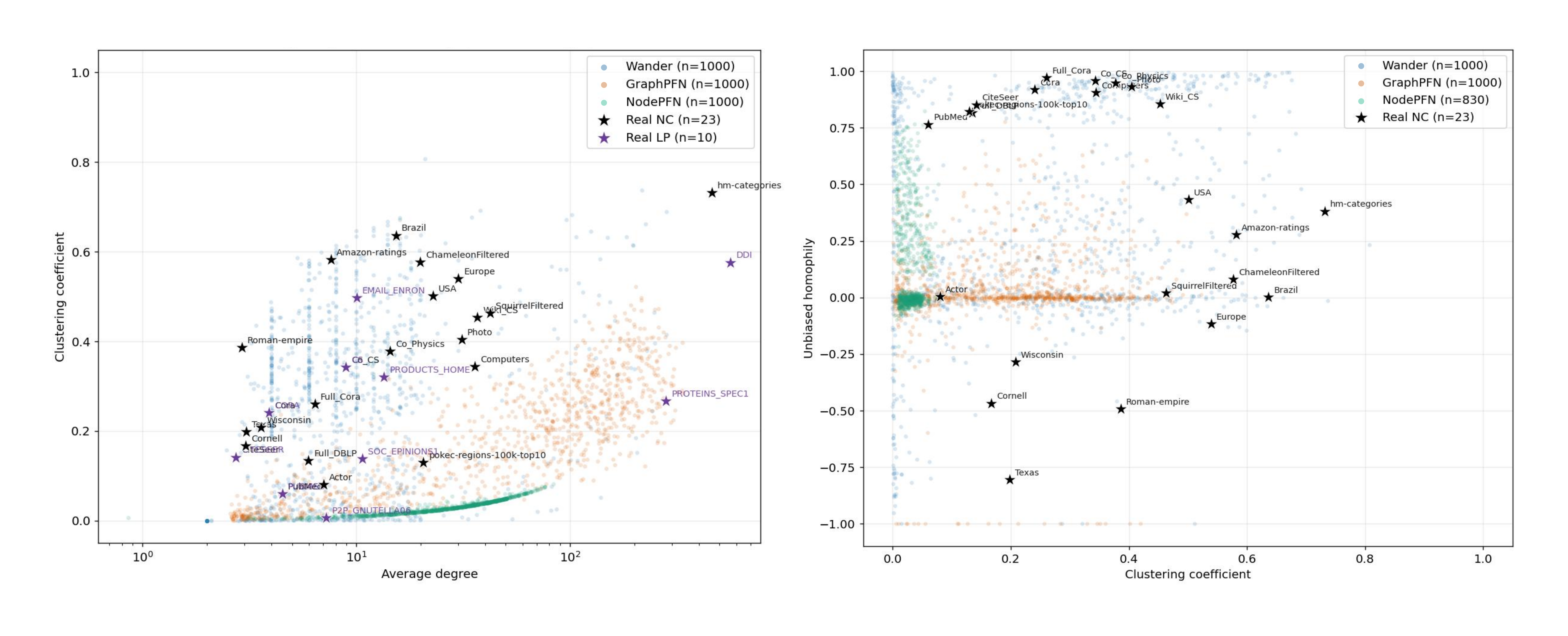}
    \caption{
    Structural coverage of the synthetic graph priors compared with the
    real-world evaluation graphs. 
    }
    \label{fig:prior_scatter}
\end{figure}

\begin{figure}[htbp]
    \centering
    \includegraphics[width=\linewidth]{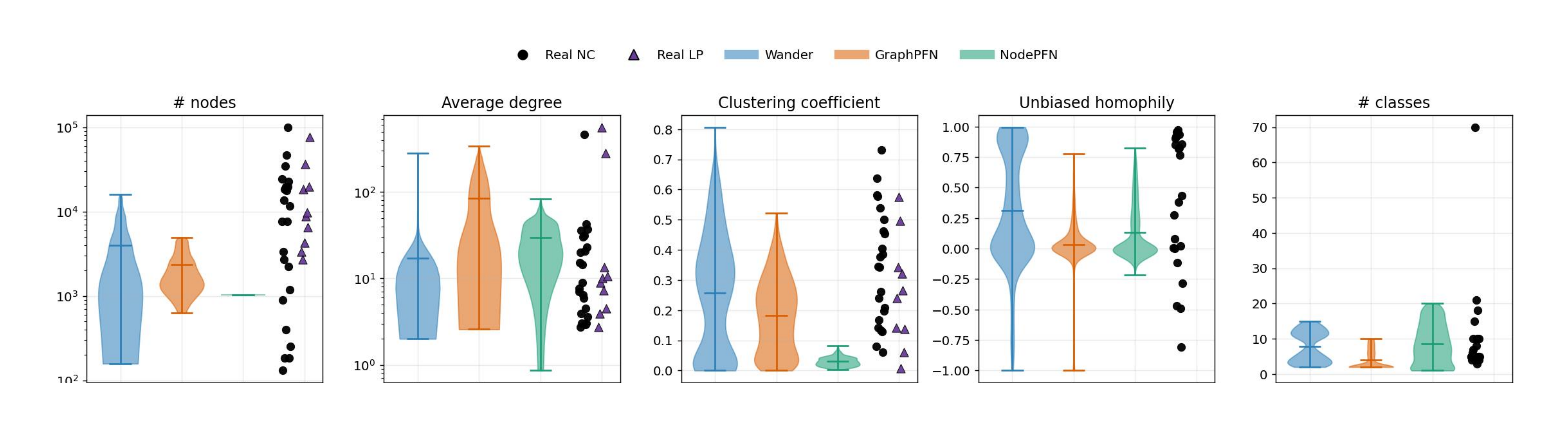}
    \caption{
    Distributions of graph statistics for 1,000 graphs sampled from each
    synthetic prior, alongside the real-world evaluation graphs. 
    }
    \label{fig:prior_distributions}
\end{figure}

\subsection{Implementation and pretraining details}\label{app::implementation}
The model and optimization hyperparameters are summarized in
Table~\ref{tab:pretrain-hparams}. We train \textsc{Wander} for 350 epochs
using a four-stage curriculum that progressively increases the complexity
of the pretraining tasks. 

During the first 200 epochs, we disable the stochastic structural update
in every layer and retain only the intra-node and ICL updates. Pretraining
during this stage uses exclusively synthetic graphs, with $80\%$ of
training steps devoted to node classification and $20\%$ to homogeneous
link prediction. This initial stage is intended to encourage the model to
learn to exploit node attributes and in-context supervision before
introducing the random-walk-based structural signal, thereby reducing the
incentive to rely primarily on graph structure or label homophily.
We additionally increase the feature dimensionality over the course of
this stage: during the first 100 epochs, the maximum feature dimension is
16, while during epochs 101--200 it is increased to 32. Whenever a sampled
synthetic task contains more features than the respective limit, we reduce
the feature dimension using PCA.

For the remaining 150 epochs, we activate the stochastic structural update
and train the complete \textsc{Wander} architecture. At this stage, $50\%$
of training steps use knowledge-graph link prediction, $40\%$ use synthetic
node classification, and $10\%$ use synthetic homogeneous link prediction.
The maximum feature dimension is increased to 64. During the final
50 epochs, we additionally increase the maximum size of synthetic graphs
from $8{,}000$ to $16{,}000$ nodes.

Each optimization step contains two graphs drawn from the same task family:
either two synthetic node-classification tasks, two synthetic homogeneous
link-prediction tasks, or two knowledge graphs. In the latter case, the two
graphs may correspond to the same underlying knowledge graph. 

The pretrained model has approximately 13M parameters.

\textsc{Wander} is stochastic because of random-walk sampling and the random
permutations of feature and label channels applied during inference. We refer
to each independently randomized forward pass as an \emph{inference sample}.
Unless stated otherwise, we average the logits from 16 inference samples
before applying the task-specific output transformation. For knowledge-graph
link prediction, we follow \citet{kim2025flock} for comparability and use the
dataset-specific numbers of inference samples reported in Table~\ref{tab:kg-walk-budgets}.

 Following prior work \citep{kim2025flock}, we add in each knowledge graph for each triple $(h,r,t)$ the corresponding inverse $(t,r^{-1},h)$, where $r^{-1}$ is a new relation symbol.

\begin{table}[t]
    \centering
    \scriptsize
    \setlength{\tabcolsep}{4pt}
    \caption{Pretraining hyperparameters across the four pretraining stages of \textsc{Wander}.  }
    \label{tab:pretrain-hparams}
    \begin{tabular}{l|c|c|c|c}
        \toprule
         & Phase 1 & Phase 2 & Phase 3 & Phase 4 \\
        \midrule
        Epoch & 1--100 & 101--200 & 201--300 & 301--350 \\
        \midrule
        \multicolumn{5}{@{}l}{\textit{Architecture}} \\
        Hidden dimension $d$ & \multicolumn{4}{c}{128 \ \textit{(all phases)}} \\
        Layers $L$ & \multicolumn{4}{c}{6 \ \textit{(all phases)}} \\
        Heads (intra-node) & \multicolumn{4}{c}{4 \ \textit{(all phases)}} \\
        Heads (ICL) & \multicolumn{4}{c}{4 \ \textit{(all phases)}} \\
        Heads (structural update) & \multicolumn{4}{c}{4 \ \textit{(all phases)}} \\
        Structural update & off & off & on & on \\
        Base walk count $K_{\text{base}}$ ($n<1$k) & --- & --- & $32$ & $32$ \\
        Base walk count $K_{\text{base}}$ ($1$k--$8$k) &  --- & --- & $64$ & $64$ \\
        Base walk count $K_{\text{base}}$ ($n\geq8$k)&  --- & --- & $128$ & $128$ \\
        Walk length $S$ ($n<1$k) & --- & --- & $32$ & $32$ \\
        Walk length $S$ ($1$k--$8$k) &  --- & --- & $64$ & $64$ \\
        Walk length $S$ ($n\geq8$k) &  --- & --- & $128$ & $128$ \\
        \midrule
        \multicolumn{5}{@{}l}{\textit{Data}} \\
        Synthetic NC & 80\% & 80\% & 40\% & 40\% \\
        Synthetic LP & 20\% & 20\% & 10\% & 10\% \\
        Real KG & 0\% & 0\% & 50\% & 50\% \\
        Max. size synthetic & 8{,}000 & 8{,}000 & 8{,}000 & 16{,}000 \\
        Max. features synthetic & 16 & 32 & 64 & 64 \\
        \midrule
        \multicolumn{5}{@{}l}{\textit{Optimizer \& training}} \\
        Optimizer & \multicolumn{4}{c}{AdamW \ \textit{(all phases)}} \\
        Learning rate & \multicolumn{4}{c}{$5{\times}10^{-5}$ \ \textit{(all phases)}} \\
        Weight decay & \multicolumn{4}{c}{$0.01$ \ \textit{(all phases)}} \\
        Batches per epoch & \multicolumn{4}{c}{2{,}000 \ \textit{(all phases)}} \\
        Graphs per batch & \multicolumn{4}{c}{2 \ \textit{(all phases)}} \\
        LP batch size (per graph) & \multicolumn{4}{c}{4 \ \textit{(all phases)}} \\
        NC batch size (per graph) & all test & all test & all visited & all visited \\
        Negatives (link prediction) & \multicolumn{4}{c}{512 \ \textit{(all phases)}} \\
        \bottomrule
    \end{tabular}
\end{table}

\subsection{Knowledge graph link prediction}\label{app::kg-evaluation}
To evaluate zero-shot transfer for knowledge graph link prediction, we consider
the entity-prediction setting, where the task is to rank candidate tail entities
given a query head entity and relation type.
\paragraph{Datasets.}
 We follow prior work on knowledge
graph foundation models~\citep{ultra,trix,kim2025flock} and evaluate on the same
benchmark suite, comprising 54 knowledge graphs grouped into three transfer
settings:

\begin{itemize}[leftmargin=*]
    \item \textbf{Inductive \(e,r\).}
    This setting evaluates transfer to knowledge graphs containing both unseen
    entities and unseen relation types. It comprises the 13 datasets introduced
    by InGram~\citep{lee2023ingram}: FB-25, FB-50, FB-75, FB-100, WK-25,
    WK-50, WK-75, WK-100, NL-0, NL-25, NL-50, NL-75, and NL-100; as well as
    10 datasets from MTDEA~\citep{zhou2023multi}: MT1 tax, MT1 health,
    MT2 org, MT2 sci, MT3 art, MT3 infra, MT4 sci, MT4 health, Metafam,
    and FBNELL. Dataset statistics are provided in
    Table~\ref{tab:inductive_er_datasets}.

    \item \textbf{Inductive \(e\).}
    This setting evaluates transfer to previously unseen entities while the
    relation vocabulary is shared with the training graph. It includes the
    12 GraIL benchmarks~\citep{teru2020inductive}, consisting of WN-v1 through
    WN-v4, FB-v1 through FB-v4, and NL-v1 through NL-v4; four INDIGO
    benchmarks~\citep{liu2021indigo}: HM 1k, HM 3k, HM 5k, and HM Indigo;
    and the two NodePiece benchmarks~\citep{galkin2021nodepiece}, ILPC Small
    and ILPC Large. Dataset statistics are provided in
    Table~\ref{tab:inductive_e_datasets}.

    \item \textbf{Transductive.}
    In the transductive setting, both entities and relation types are observed
    during training, and evaluation concerns held-out links over the same entity
    and relation vocabularies. This group comprises CoDEx Small and CoDEx
    Large~\citep{codex}, NELL-995~\citep{xiong2017deeppath},
    YAGO3-10~\citep{mahdisoltani2014yago3}, WDsinger, NELL23k,
    FB15k-237(10), FB15k-237(20), and FB15k-237(50)~\citep{lv2020dynamic},
    AristoV4~\citep{chen2021relation}, DBpedia100k~\citep{ding2018improving},
    ConceptNet100k~\citep{malaviya2020commonsense}, and
    Hetionet~\citep{himmelstein2017systematic}. Dataset statistics are
    provided in Table~\ref{tab:transductive_datasets}.
\end{itemize}

\paragraph{Baselines and evaluation protocol.}
We compare \textsc{Wander} against the knowledge graph foundation models
\textsc{Ultra}~\citep{ultra}, \textsc{Trix}~\citep{trix}, and \textsc{Flock}~\citep{kim2025flock}.
All three baselines are designed for zero-shot transfer across knowledge graphs
and are evaluated on the benchmark suite described above. 

We follow the filtered ranking protocol used in the corresponding prior
work. For each query, candidate entities corresponding to other known true
facts are removed from the ranking before computing the evaluation metrics.
For evaluating head prediction we convert the query into a tail prediction, by using the inverse relation type. We report mean reciprocal rank (MRR) and Hits@10. As the benchmark datasets
and evaluation protocol coincide with those used in the respective baseline
papers, we directly use their reported results rather than reevaluating the
models. Full per-dataset results are reported in
Table~\ref{tab::full_kg}.

For \textsc{Wander}, we use the same link-prediction random-walk sampling
protocol as during pretraining (Appendix~\ref{app::method}). Following
\textsc{Flock}~\citep{kim2025flock}, we use walks of length $S=128$ and
adopt its dataset-specific entity-prediction walk and inference sample budgets, reported in
Table~\ref{tab:kg-walk-budgets}. Here, $K_{\text{base}}$ denotes the base walk count, so
each structural update samples $3K_{\text{base}}$ walks according to the three starting
distributions defined in Appendix~\ref{app::method}.

\begin{table}[t]
\scriptsize
    \centering
    \caption{Dataset statistics for inductive-$e,r$ link prediction datasets.
    Triples are the number of edges given at training, validation, or test,
    respectively, whereas \textit{Valid} and \textit{Test} denote triples to be predicted in the
    validation and test graphs.}
    \label{tab:inductive_er_datasets}
    \scriptsize
    \setlength{\tabcolsep}{3pt}
    \resizebox{\linewidth}{!}{
    \begin{tabular}{lccccccccccc}
        \toprule
        \textbf{Dataset}
        & \multicolumn{3}{c}{\textbf{Training Graph}}
        & \multicolumn{4}{c}{\textbf{Validation Graph}}
        & \multicolumn{4}{c}{\textbf{Test Graph}} \\
        \cmidrule(lr){2-4}
        \cmidrule(lr){5-8}
        \cmidrule(lr){9-12}
        & Entities & Rels & Triples
        & Entities & Rels & Triples & Valid
        & Entities & Rels & Triples & Test \\
        \midrule
        FB-25          & 5190  & 163 & 91571 & 4097  & 216 & 17147 & 5716 & 4097  & 216 & 17147 & 5716 \\
        FB-50          & 5190  & 153 & 85375 & 4445  & 205 & 11636 & 3879 & 4445  & 205 & 11636 & 3879 \\
        FB-75          & 4659  & 134 & 62809 & 2792  & 186 & 9316  & 3106 & 2792  & 186 & 9316  & 3106 \\
        FB-100         & 4659  & 134 & 62809 & 2624  & 77  & 6987  & 2329 & 2624  & 77  & 6987  & 2329 \\
        WK-25          & 12659 & 47  & 41873 & 3228  & 74  & 3391  & 1130 & 3228  & 74  & 3391  & 1131 \\
        WK-50          & 12022 & 72  & 82481 & 9328  & 93  & 9672  & 3224 & 9328  & 93  & 9672  & 3225 \\
        WK-75          & 6853  & 52  & 28741 & 2722  & 65  & 3430  & 1143 & 2722  & 65  & 3430  & 1144 \\
        WK-100         & 9784  & 67  & 49875 & 12136 & 37  & 13487 & 4496 & 12136 & 37  & 13487 & 4496 \\
        NL-0           & 1814  & 134 & 7796  & 2026  & 112 & 2287  & 763  & 2026  & 112 & 2287  & 763 \\
        NL-25          & 4396  & 106 & 17578 & 2146  & 120 & 2230  & 743  & 2146  & 120 & 2230  & 744 \\
        NL-50          & 4396  & 106 & 17578 & 2335  & 119 & 2576  & 859  & 2335  & 119 & 2576  & 859 \\
        NL-75          & 2607  & 96  & 11058 & 1578  & 116 & 1818  & 606  & 1578  & 116 & 1818  & 607 \\
        NL-100         & 1258  & 55  & 7832  & 1709  & 53  & 2378  & 793  & 1709  & 53  & 2378  & 793 \\
        Metafam        & 1316  & 28  & 13821 & 1316  & 28  & 13821 & 590  & 656   & 28  & 7257  & 184 \\
        FBNELL         & 4636  & 100 & 10275 & 4636  & 100 & 10275 & 1055 & 4752  & 183 & 10685 & 597 \\
        Wiki MT1 tax   & 10000 & 10  & 17178 & 10000 & 10  & 17178 & 1908 & 10000 & 9   & 16526 & 1834 \\
        Wiki MT1 health& 10000 & 7   & 14371 & 10000 & 7   & 14371 & 1596 & 10000 & 7   & 14110 & 1566 \\
        Wiki MT2 org   & 10000 & 10  & 23233 & 10000 & 10  & 23233 & 2581 & 10000 & 11  & 21976 & 2441 \\
        Wiki MT2 sci   & 10000 & 16  & 16471 & 10000 & 16  & 16471 & 1830 & 10000 & 16  & 14852 & 1650 \\
        Wiki MT3 art   & 10000 & 45  & 27262 & 10000 & 45  & 27262 & 3026 & 10000 & 45  & 28023 & 3113 \\
        Wiki MT3 infra & 10000 & 24  & 21990 & 10000 & 24  & 21990 & 2443 & 10000 & 27  & 21646 & 2405 \\
        Wiki MT4 sci   & 10000 & 42  & 12576 & 10000 & 42  & 12576 & 1397 & 10000 & 42  & 12516 & 1388 \\
        Wiki MT4 health& 10000 & 21  & 15539 & 10000 & 21  & 15539 & 1725 & 10000 & 20  & 15337 & 1703 \\
        \bottomrule
    \end{tabular}
    }
\end{table}
\begin{table}[t]
\scriptsize
    \centering
    \caption{Dataset statistics for inductive-$e$ link prediction datasets.
    Triples are the number of edges given at training, validation, or test,
    respectively, whereas \textit{Valid} and \textit{Test} denote triples to be predicted in the
    validation and test graphs.}
    \label{tab:inductive_e_datasets}
    \scriptsize
    \setlength{\tabcolsep}{4pt}
    \resizebox{\linewidth}{!}{
    \begin{tabular}{lccccccccc}
        \toprule
        \textbf{Dataset}
        & \textbf{Rels}
        & \multicolumn{2}{c}{\textbf{Training Graph}}
        & \multicolumn{3}{c}{\textbf{Validation Graph}}
        & \multicolumn{3}{c}{\textbf{Test Graph}} \\
        \cmidrule(lr){3-4}
        \cmidrule(lr){5-7}
        \cmidrule(lr){8-10}
        & & Entities & Triples
        & Entities & Triples & Valid
        & Entities & Triples & Test \\
        \midrule
        FB-v1       & 180 & 1594  & 4245   & 1594  & 4245   & 489   & 1093  & 1993   & 411 \\
        FB-v2       & 200 & 2608  & 9739   & 2608  & 9739   & 1166  & 1660  & 4145   & 947 \\
        FB-v3       & 215 & 3668  & 17986  & 3668  & 17986  & 2194  & 2501  & 7406   & 1731 \\
        FB-v4       & 219 & 4707  & 27203  & 4707  & 27203  & 3352  & 3051  & 11714  & 2840 \\
        WN-v1       & 9   & 2746  & 5410   & 2746  & 5410   & 630   & 922   & 1618   & 373 \\
        WN-v2       & 10  & 6954  & 15262  & 6954  & 15262  & 1838  & 2757  & 4011   & 852 \\
        WN-v3       & 11  & 12078 & 25901  & 12078 & 25901  & 3097  & 5084  & 6327   & 1143 \\
        WN-v4       & 9   & 3861  & 7940   & 3861  & 7940   & 934   & 7084  & 12334  & 2823 \\
        NL-v1       & 14  & 3103  & 4687   & 3103  & 4687   & 414   & 225   & 833    & 201 \\
        NL-v2       & 88  & 2564  & 8219   & 2564  & 8219   & 922   & 2086  & 4586   & 935 \\
        NL-v3       & 142 & 4647  & 16393  & 4647  & 16393  & 1851  & 3566  & 8048   & 1620 \\
        NL-v4       & 76  & 2092  & 7546   & 2092  & 7546   & 876   & 2795  & 7073   & 1447 \\
        ILPC Small  & 48  & 10230 & 78616  & 6653  & 20960  & 2908  & 6653  & 20960  & 2902 \\
        ILPC Large  & 65  & 46626 & 202446 & 29246 & 77044  & 10179 & 29246 & 77044  & 10184 \\
        HM 1k       & 11  & 36237 & 93364  & 36311 & 93364  & 1771  & 9899  & 18638  & 476 \\
        HM 3k       & 11  & 32118 & 71097  & 32250 & 71097  & 1201  & 19218 & 38285  & 1349 \\
        HM 5k       & 11  & 28601 & 57601  & 28744 & 57601  & 900   & 23792 & 48425  & 2124 \\
        HM Indigo   & 229 & 12721 & 121601 & 12797 & 121601 & 14121 & 14775 & 250195 & 14904 \\
        \bottomrule
    \end{tabular}
    }
\end{table}

\begin{table}[t]
    \centering
    \caption{Dataset statistics for transductive link prediction datasets.
    Entity task denotes the entity-prediction task: h/t is predicting both heads
    and tails, and t is predicting only tails.}
    \label{tab:transductive_datasets}
    \scriptsize
    \begin{tabular}{lrrrrrl}
        \toprule
        \textbf{Dataset}
        & \textbf{Entities}
        & \textbf{Rels}
        & \textbf{Train}
        & \textbf{Valid}
        & \textbf{Test}
        & \textbf{Entity Task} \\
        \midrule
        CoDEx Small     & 2034   & 42   & 32888   & 1827   & 1828   & h/t \\
        CoDEx Large     & 77951  & 69   & 551193  & 30622  & 30622  & h/t \\
        NELL995         & 74536  & 200  & 149678  & 543    & 2818   & h/t \\
        YAGO3-10         & 123182 & 37   & 1079040 & 5000   & 5000   & h/t \\
        WDsinger        & 10282  & 135  & 16142   & 2163   & 2203   & h/t \\
        NELL23k         & 22925  & 200  & 25445   & 4961   & 4952   & h/t \\
        AristoV4        & 44949  & 1605 & 242567  & 20000  & 20000  & h/t \\
        DBpedia100k     & 99604  & 470  & 597572  & 50000  & 50000  & h/t \\
        ConceptNet100k  & 78334  & 34   & 100000  & 1200   & 1200   & h/t \\
        FB15k-237(10)   & 11512  & 237  & 27211   & 15624  & 18150  & t \\
        FB15k-237(20)   & 13166  & 237  & 54423   & 16963  & 19776  & t \\
        FB15k-237(50)   & 14149  & 237  & 136057  & 17449  & 20324  & t \\
        Hetionet        & 45158  & 24   & 2025177 & 112510 & 112510 & h/t \\
        \midrule
        \multicolumn{7}{@{}l}{\textit{Pretraining}} \\
        FB15k-237       & 14541  & 237  & 272115  & 17535  & 20466  & h/t \\
        WN18RR          & 40943  & 11   & 86835   & 3034   & 3134   & h/t \\
        CoDEx Medium    & 17050  & 51   & 185584  & 10310  & 10311  & h/t \\
        \bottomrule
    \end{tabular}
\end{table}
\begin{table}[htbp]
    \scriptsize
        \centering
        \caption{Zero-shot knowledge graph link prediction results. Best in each row is \textbf{bold}, second-best is \underline{underlined}.}
        \label{tab::full_kg}
        \setlength{\tabcolsep}{3pt}
        \begin{tabular}{clcccccc|cc}
        \toprule
        & Dataset & \multicolumn{2}{c}{\textsc{Ultra}} & \multicolumn{2}{c}{\textsc{Trix}}
        & \multicolumn{2}{c}{\textsc{Flock}} & \multicolumn{2}{c}{\textbf{\textsc{Wander}}} \\
        & & MRR & H@10 & MRR & H@10 & MRR & H@10 & MRR & H@10 \\
        \midrule

        \multirow{23}{*}{\rotatebox{90}{Inductive (entity, relation)}}
        & FB-25 & 0.388 & 0.640 & \underline{0.393} & 0.650 & \textbf{0.404} & \textbf{0.664} & \textbf{0.404} & \underline{0.663} \\
        & FB-50 & \underline{0.338} & 0.543 & 0.334 & 0.547 & \textbf{0.352} & \underline{0.566} & \textbf{0.352} & \textbf{0.568} \\
        & FB-75 & 0.403 & 0.604 & 0.401 & 0.611 & \textbf{0.418} & \textbf{0.622} & \underline{0.408} & \underline{0.617} \\
        & FB-100 & 0.449 & \underline{0.642} & 0.436 & 0.635 & \underline{0.452} & \textbf{0.663} & \textbf{0.458} & \textbf{0.663} \\
        & WK-25 & \textbf{0.316} & \textbf{0.532} & \underline{0.305} & \underline{0.496} & 0.280 & 0.491 & 0.288 & 0.477 \\
        & WK-50 & \textbf{0.166} & \textbf{0.324} & \textbf{0.166} & \underline{0.313} & 0.136 & 0.278 & \underline{0.147} & 0.289 \\
        & WK-75 & 0.365 & \underline{0.537} & 0.368 & 0.513 & \textbf{0.382} & \textbf{0.538} & \underline{0.372} & 0.517 \\
        & WK-100 & 0.164 & 0.286 & \underline{0.188} & 0.299 & 0.187 & \underline{0.304} & \textbf{0.190} & \textbf{0.309} \\
        & NL-0 & 0.342 & 0.523 & \textbf{0.385} & 0.549 & \underline{0.381} & \textbf{0.606} & 0.351 & \underline{0.553} \\
        & NL-25 & \textbf{0.395} & 0.569 & \underline{0.377} & \underline{0.589} & 0.345 & \textbf{0.590} & 0.338 & 0.546 \\
        & NL-50 & \textbf{0.407} & \underline{0.570} & \underline{0.404} & 0.548 & 0.366 & 0.565 & 0.391 & \textbf{0.579} \\
        & NL-75 & \textbf{0.368} & \underline{0.547} & 0.351 & 0.525 & 0.311 & 0.524 & \underline{0.357} & \textbf{0.555} \\
        & NL-100 & 0.471 & 0.651 & \underline{0.486} & 0.676 & 0.452 & \underline{0.692} & \textbf{0.487} & \textbf{0.707} \\
        & MT1 tax & 0.224 & 0.305 & \textbf{0.358} & \textbf{0.452} & \underline{0.282} & 0.383 & 0.281 & \underline{0.410} \\
        & MT1 health & 0.298 & 0.374 & 0.376 & 0.457 & \underline{0.385} & \underline{0.481} & \textbf{0.391} & \textbf{0.483} \\
        & MT2 org & 0.095 & 0.159 & 0.091 & 0.156 & \underline{0.100} & \underline{0.163} & \textbf{0.102} & \textbf{0.167} \\
        & MT2 sci & 0.258 & 0.354 & \underline{0.323} & \underline{0.465} & 0.318 & 0.458 & \textbf{0.359} & \textbf{0.526} \\
        & MT3 art & 0.259 & 0.402 & 0.284 & \underline{0.441} & \underline{0.301} & \textbf{0.466} & \textbf{0.304} & \textbf{0.466} \\
        & MT3 infra & 0.619 & 0.755 & 0.655 & 0.797 & \textbf{0.684} & \underline{0.821} & \underline{0.677} & \textbf{0.822} \\
        & MT4 sci & 0.274 & 0.449 & 0.290 & 0.460 & \underline{0.301} & \underline{0.463} & \textbf{0.310} & \textbf{0.480} \\
        & MT4 health & 0.624 & 0.737 & 0.677 & 0.775 & \underline{0.680} & \underline{0.780} & \textbf{0.701} & \textbf{0.787} \\
        & Metafam & 0.238 & 0.644 & 0.341 & 0.815 & \underline{0.476} & \underline{0.935} & \textbf{0.504} & \textbf{0.941} \\
        & FBNELL & 0.485 & 0.652 & 0.473 & 0.660 & \underline{0.502} & \underline{0.700} & \textbf{0.526} & \textbf{0.712} \\

        \midrule
        \multirow{18}{*}{\rotatebox{90}{ Inductive (entity)}}
        & FB-v1 & 0.498 & 0.656 & \underline{0.515} & 0.682 & 0.500 & \underline{0.697} & \textbf{0.524} & \textbf{0.701} \\
        & FB-v2 & 0.512 & 0.700 & 0.525 & 0.730 & \underline{0.535} & \underline{0.737} & \textbf{0.540} & \textbf{0.742} \\
        & FB-v3 & 0.491 & 0.654 & 0.501 & 0.669 & \underline{0.511} & \underline{0.685} & \textbf{0.514} & \textbf{0.693} \\
        & FB-v4 & 0.486 & 0.677 & 0.493 & 0.687 & \textbf{0.505} & \underline{0.702} & \underline{0.504} & \textbf{0.703} \\
        & WN-v1 & 0.648 & 0.768 & \textbf{0.699} & 0.791 & \underline{0.698} & \underline{0.803} & 0.690 & \textbf{0.807} \\
        & WN-v2 & 0.663 & 0.765 & 0.678 & 0.781 & \textbf{0.696} & \underline{0.790} & \underline{0.687} & \textbf{0.793} \\
        & WN-v3 & 0.376 & 0.476 & 0.418 & 0.541 & \underline{0.467} & \underline{0.608} & \textbf{0.478} & \textbf{0.627} \\
        & WN-v4 & 0.611 & 0.705 & 0.648 & 0.723 & \textbf{0.653} & \underline{0.729} & \underline{0.650} & \textbf{0.731} \\
        & NL-v1 & \underline{0.785} & \textbf{0.913} & \textbf{0.806} & \underline{0.898} & 0.658 & 0.863 & 0.677 & 0.878 \\
        & NL-v2 & 0.526 & 0.707 & 0.569 & 0.768 & \textbf{0.588} & \textbf{0.797} & \underline{0.587} & \underline{0.794} \\
        & NL-v3 & 0.515 & 0.702 & 0.558 & 0.743 & \underline{0.590} & \underline{0.783} & \textbf{0.610} & \textbf{0.796} \\
        & NL-v4 & 0.479 & 0.712 & 0.538 & 0.765 & \underline{0.555} & \underline{0.786} & \textbf{0.572} & \textbf{0.795} \\
        & HM 1k & 0.059 & 0.092 & \textbf{0.072} & \textbf{0.128} & \underline{0.069} & \underline{0.119} & 0.062 & \underline{0.119} \\
        & HM 3k & 0.037 & 0.077 & \textbf{0.069} & \underline{0.119} & \underline{0.067} & 0.118 & \textbf{0.069} & \textbf{0.126} \\
        & HM 5k & 0.034 & 0.071 & 0.062 & 0.110 & \underline{0.064} & \underline{0.116} & \textbf{0.067} & \textbf{0.126} \\
        & HM Indigo & \textbf{0.440} & \textbf{0.648} & \underline{0.436} & \underline{0.645} & 0.423 & 0.638 & 0.416 & 0.636 \\
        & ILPC Small & 0.302 & 0.443 & 0.303 & \underline{0.455} & \underline{0.309} & \textbf{0.459} & \textbf{0.310} & \textbf{0.459} \\
        & ILPC Large & 0.290 & 0.424 & 0.307 & 0.428 & \underline{0.318} & \underline{0.438} & \textbf{0.321} & \textbf{0.440} \\

        \midrule
        \multirow{13}{*}{\rotatebox{90}{ Transductive}}
        & NELL995 & 0.406 & 0.543 & 0.472 & 0.629 & \underline{0.494} & \underline{0.655} & \textbf{0.519} & \textbf{0.662} \\
        & NELL23k & 0.239 & 0.408 & \textbf{0.290} & \textbf{0.497} & 0.233 & 0.398 & \underline{0.249} & \underline{0.413} \\
        & WDsinger & 0.382 & 0.498 & \textbf{0.511} & \textbf{0.609} & \underline{0.410} & \underline{0.528} & 0.402 & 0.521 \\
        & ConceptNet100k & 0.082 & 0.162 & 0.193 & 0.345 & \underline{0.248} & \underline{0.453} & \textbf{0.258} & \textbf{0.477} \\
        & CoDEx Small & \textbf{0.472} & \underline{0.667} & \textbf{0.472} & \textbf{0.670} & 0.441 & 0.644 & \underline{0.457} & 0.658 \\
        & CoDEx Large & 0.338 & \underline{0.469} & 0.335 & \underline{0.469} & \underline{0.342} & 0.464 & \textbf{0.345} & \textbf{0.475} \\
        & YAGO310 & \underline{0.451} & 0.615 & 0.409 & 0.627 & 0.414 & \textbf{0.674} & \textbf{0.499} & \underline{0.666} \\
        & AristoV4 & 0.182 & 0.282 & 0.181 & 0.286 & \textbf{0.308} & \textbf{0.443} & \underline{0.295} & \underline{0.424} \\
        & DBpedia100k & 0.398 & 0.576 & 0.426 & 0.603 & \textbf{0.450} & \textbf{0.627} & \underline{0.436} & \underline{0.625} \\
        & Hetionet & \underline{0.257} & 0.379 & \textbf{0.279} & \textbf{0.420} & 0.246 & 0.371 & \textbf{0.279} & \underline{0.415} \\
        & FB15k-237(10) & \textbf{0.248} & 0.398 & \underline{0.246} & 0.393 & \underline{0.246} & \underline{0.402} & \textbf{0.248} & \textbf{0.404} \\
        & FB15k-237(20) & 0.272 & 0.436 & 0.269 & 0.430 & \underline{0.273} & \underline{0.444} & \textbf{0.277} & \textbf{0.451} \\
        & FB15k-237(50) & \textbf{0.324} & \textbf{0.526} & 0.321 & 0.521 & 0.319 & 0.518 & \underline{0.323} & \underline{0.523} \\

        \midrule
        & Mean & 0.366 & 0.518 & 0.390 & 0.548 & \underline{0.391} & \underline{0.560} & \textbf{0.399} & \textbf{0.565} \\
        & Mean rank & 3.250 & 3.417 & 2.657 & 2.750 & \underline{2.343} & \underline{2.222} & \textbf{1.750} & \textbf{1.611} \\
        \bottomrule
        \end{tabular}
        \end{table}

\begin{table}[t]
    \centering
    \scriptsize
    \setlength{\tabcolsep}{4pt}
    \caption{Random-walk configuration used for knowledge-graph link prediction.
    We follow the entity-prediction evaluation configuration of
    \textsc{Flock}~\citep{kim2025flock}. Each \textsc{Wander} structural update samples $3K_{\text{base}}$
    walks. The walk length is $S=128$ for all datasets.}
    \label{tab:kg-walk-budgets}
    \begin{tabular}{lcc}
        \toprule
        \textbf{Dataset(s)}
        & \textbf{Base walks $K_{\text{base}}$}
        & \textbf{Inference samples} \\
        \midrule

        \multicolumn{3}{l}{\textit{Inductive $(e,r)$}} \\
        FB-25, FB-50              & 16 & 16 \\
        FB-75, FB-100             & 8  & 16 \\
        WK-25, WK-75              & 4  & 16 \\
        WK-50, WK-100             & 16 & 16 \\
        NL-0, NL-25, NL-75        & 2  & 16 \\
        NL-50, NL-100             & 4  & 16 \\
        MT1-tax                    & 4  & 16 \\
        MT1-health                 & 16 & 16 \\
        MT2-org                    & 32 & 16 \\
        MT2-sci                    & 16 & 16 \\
        MT3-art                    & 32 & 16 \\
        MT3-infra                  & 32 & 16 \\
        MT4-sci                    & 16 & 16 \\
        MT4-health                 & 16 & 16 \\
        Metafam, FBNELL            & 16 & 16 \\
        \midrule

        \multicolumn{3}{l}{\textit{Inductive $(e)$}} \\
        FB-v1--v4                  & 16  & 16 \\
        WN-v1--v4                  & 16  & 16 \\
        NL-v1--v4                  & 16  & 16 \\
        ILPC Small                 & 16  & 16 \\
        ILPC Large                 & 64  & 16 \\
        HM 1k                      & 16  & 16 \\
        HM 3k                      & 32  & 16 \\
        HM 5k                      & 64  & 16 \\
        HM Indigo                  & 128 & 16 \\
        \midrule

        \multicolumn{3}{l}{\textit{Transductive}} \\
        CoDEx Small                & 16  & 16 \\
        CoDEx Large                & 512 & 2  \\
        NELL-995                   & 128 & 16 \\
        YAGO3-10                   & 512 & 8  \\
        WDsinger                   & 16  & 16 \\
        NELL23k                    & 32  & 16 \\
        FB15k-237(10)              & 32  & 16 \\
        FB15k-237(20)              & 64  & 16 \\
        FB15k-237(50)              & 64  & 16 \\
        Hetionet                   & 512 & 2  \\
        DBpedia100k                & 512 & 2  \\
        AristoV4                   & 256 & 4  \\
        ConceptNet100k             & 128 & 16 \\
        \bottomrule
    \end{tabular}
\end{table}

\subsection{Homogeneous link prediction}\label{app::homo_lp_eval}

\paragraph{Datasets.}
We follow the homogeneous link-prediction benchmark introduced by
\textsc{AnyGraph}~\citep{xia2024anygraph}, using the dataset versions and
train/validation/test splits released with its official implementation.
Our evaluation comprises ten graphs: CiteSeer, Cora, PubMed, Coauthor CS,
DDI, Products Home, P2P Gnutella06, Email-Enron, Proteins Spec1, and
Soc Epinions.
These datasets cover citation, coauthorship, e-commerce, peer-to-peer,
communication, biological, and social networks, and include both attributed
and unattributed graphs.
Dataset statistics are reported in
Table~\ref{tab:homo-lp-statistics}.

We adopt a slightly stricter evaluation protocol than the original
\textsc{AnyGraph} benchmark.
Although the graphs are treated as undirected, we found that for some test
edges $(u,v)$, the reverse edge $(v,u)$ occurs in the training split.
Such cases effectively reveal the test edge during training.
We therefore remove all test edges whose reverse is present in the training
graph.
For attributed datasets, we apply the same feature preprocessing as in the
node-classification experiments described below.

\paragraph{Baselines and hyperparameter selection.}
We compare \textsc{Wander} against the end-to-end link-prediction models
\textsc{NBFNet}~\citep{zhu2022neural} and \textsc{Buddy}~\citep{chamberlain2022graph}, as well
as the graph foundation model \textsc{AnyGraph}~\citep{xia2024anygraph}.
For \textsc{NBFNet} and \textsc{Buddy}, we perform dataset-specific hyperparameter selection
based on validation performance.
Both models are trained for at most 40 epochs with early stopping using a
patience of five epochs.

For \textsc{NBFNet}, we fix the hidden dimension to 32 following the official implementation and search over learning rates
$\{10^{-3},5\cdot10^{-3}\}$, network depths $\{3,5,6\}$, and numbers of
sampled negative examples $\{1,32,256\}$.
For \textsc{Buddy}, we search over hidden dimensions $\{128,256\}$, learning rates
$\{10^{-3},5\cdot10^{-3}\}$, and the same set of negative-sampling budgets
$\{1,32,256\}$.
The validation-selected hyperparameters for each dataset are reported in
Table~\ref{tab:lp-baseline-selected-hparams}.

Unlike \textsc{NBFNet} and \textsc{Buddy}, \textsc{AnyGraph} is used as a pretrained graph
foundation model.
Its authors provide two checkpoints pretrained on disjoint collections of
real-world graphs.
For each evaluation dataset, we use the checkpoint whose pretraining set does
not contain that dataset, thereby avoiding overlap between pretraining and
evaluation data.

\paragraph{\textsc{Wander} inference.}
We use the same link-prediction random-walk sampling protocol as during
pretraining (Appendix~\ref{app::method}). For attributed graphs, we use a base walk count
of $K_{\text{base}}=64$ and walk length $S=64$; for graphs without node features, we
use $K_{\text{base}}=128$ and $S=128$. Thus, each structural update samples $3K_{\text{base}}$ walks
according to the three starting distributions described in Appendix~\ref{app::method}.
Unless stated otherwise, we average predictions over 16 independently
sampled forward passes.

\begin{table}[htbp]
    \centering
    \caption{Val-selected hyperparameters for the \textsc{NBFNet} and \textsc{Buddy} link-prediction baselines.}
    \label{tab:lp-baseline-selected-hparams}
    \footnotesize
    \setlength{\tabcolsep}{4pt}
    \begin{tabular}{l rrr rrr}
    \toprule
     & \multicolumn{3}{c}{\textsc{NBFNet}} & \multicolumn{3}{c}{\textsc{Buddy}} \\
    \cmidrule(lr){2-4} \cmidrule(lr){5-7}
    Dataset & Depth & lr & \#Neg. & Hidden & lr & \#Neg. \\
    \midrule
    CiteSeer           & 3 & 0.005 & 1   & 256 & 0.005 & 1   \\
    Cora               & 5 & 0.001 & 32  & 128 & 0.005 & 32  \\
    PubMed             & 5 & 0.005 & 1   & 256 & 0.001 & 256 \\
    CS                 & 3 & 0.005 & 32  & 256 & 0.001 & 256 \\
    DDI                & 5 & 0.005 & 256 & 256 & 0.001 & 256 \\
    P2P Gnutella06     & 6 & 0.005 & 1   & 256 & 0.005 & 1   \\
    Email Enron        & 5 & 0.005 & 256 & 256 & 0.001 & 256 \\
    Proteins Spec1     & 6 & 0.005 & 32  & 128 & 0.001 & 256 \\
    SOC Epinions       & 6 & 0.005 & 256 & 256 & 0.001 & 32  \\
    Products Home      & 6 & 0.005 & 32  & 256 & 0.005 & 32  \\
    \bottomrule
    \end{tabular}
    \end{table}

\begin{table}[htbp]
    \centering
    \caption{Statistics of the homogeneous link-prediction datasets. Test sources are unique heads of test edges. Original is the AnyGraph split. Remaining drops a test edge when either direction already appears in the training graph, which is stricter than AnyGraph's directed protocol; a source is kept only if at least one test edge survives.}
    \label{tab:homo-lp-statistics}
    \scriptsize
    \setlength{\tabcolsep}{3pt}
    \begin{tabular}{lrrrrrrrrr}
    \toprule
    Dataset & \#Nodes & \#Edges & Avg.\ degree & Clustering Coeff. & Diameter & \#Features & \shortstack{\#Test sources\\(original)} & \shortstack{\#Test sources\\(filtered)} \\
    \midrule
    CiteSeer         & 3,327  & 4,552     & 2.7   & 0.141 & 28 & 3,703 & 423    & 175 \\
    Cora            & 2,708  & 5,278     & 3.9   & 0.241 & 18 & 1,433 & 399    & 192 \\
    PubMed          & 19,717 & 44,324    & 4.5   & 0.060 & 16 & 500   & 2,041  & 895 \\
    CS              & 18,333 & 81,894    & 8.9   & 0.343 & 24 & 6,805 & 1,973  & 1,377 \\
    DDI             & 4,267  & 1,201,400 & 563.1 & 0.576 & 5  & --    & 1,587  & 1,587 \\
    P2P Gnutella06  & 8,717  & 31,525    & 7.2   & 0.007 & 9  & --    & 2,107  & 2,107 \\
    Email Enron     & 36,692 & 183,831   & 10.0  & 0.497 & 11 & --    & 13,137 & 4,400 \\
    Proteins Spec1 & 6,568  & 922,980   & 281.1 & 0.267 & 5  & --    & 6,124  & 5,442 \\
    SOC Epinions   & 75,888 & 405,740   & 10.7  & 0.138 & 12 & --    & 17,695 & 14,511 \\
    Products Home  & 9,790  & 65,919    & 13.5  & 0.321 & 22 & 100   & 4,106  & 1,681 \\
    \bottomrule
    \end{tabular}
    \end{table}
    
\paragraph{Comparison to UniLP.}
\textsc{UniLP} \citep{dong2024universal} and \textsc{TFMLinker} \citep{liao2026tfmlinker} are two recently proposed graph foundation models for homogeneous link prediction.
We do not include them in our main comparison for three reasons.
First, both models score links conditioned on a specific pair of endpoints, whereas \textsc{Wander} conditions on a source node and jointly scores candidate targets.
Consequently, evaluating source-conditioned Recall@20 over all candidate nodes would require pairwise models to separately evaluate a large number of candidate edges, making our main evaluation protocol computationally impractical.
Second, both methods are designed for unattributed graphs and therefore cannot exploit the node features available in several of our benchmarks.
Third, for TFMLinker, we were unable to identify an official implementation at the time of evaluation.

To nevertheless provide a direct comparison, we evaluate \textsc{Wander} and \textsc{UniLP} on the datasets considered by \textsc{UniLP} using a sampled, source-conditioned ranking protocol.
\textsc{UniLP} originally reports Hits@50 by comparing positive edges against sampled negative edges whose endpoints may differ.
Such comparisons are less natural for \textsc{Wander}, whose scores are produced jointly for candidate tails conditioned on a fixed source and are therefore primarily meaningful relative to other candidates for the same query.
As a middle ground, we use sampled Recall@20: for each source node, we rank all of its held-out positive neighbors together with 100 sampled non-edges and report the fraction of positive neighbors appearing among the top 20 candidates.
We evaluate on two random $70/10/20$ train/validation/test splits.

Table~\ref{tab:unilp-vs-wander-sampled-recall} shows that, without node features, the two models achieve essentially identical average performance (85.4\% Recall@20), with \textsc{Wander} outperforming \textsc{UniLP} on Celegans, USAir, and Cora, and \textsc{UniLP} performing better on the remaining four datasets.
At the same time, \textsc{Wander} is substantially faster (under same hardware, A100 GPU) under this evaluation protocol. 
For the three datasets with node attributes, we additionally report \textsc{Wander} with features enabled.
This improves performance on Cora, CS, and Facebook, respectively, exceeding the featureless \textsc{UniLP} results on all three datasets.
\begin{table}[htbp]
\scriptsize
    \centering
    \caption{Sampled Recall@20 and test-only wall-clock for \textsc{UniLP} vs.\ \textsc{Wander} on two random $70/10/20$ splits. Per source: rank that source's eval neighbors plus $100$ sampled non-edges.}
    \label{tab:unilp-vs-wander-sampled-recall}
    \setlength{\tabcolsep}{3pt}
    \begin{tabular}{l rr rr rr rr rr rr rr}
    \toprule
     & \multicolumn{2}{c}{Celegans} & \multicolumn{2}{c}{USAir} & \multicolumn{2}{c}{NS} & \multicolumn{2}{c}{PB} & \multicolumn{2}{c}{Cora} & \multicolumn{2}{c}{CS} & \multicolumn{2}{c}{Facebook} \\
    \cmidrule(lr){2-3} \cmidrule(lr){4-5} \cmidrule(lr){6-7} \cmidrule(lr){8-9}
    \cmidrule(lr){10-11} \cmidrule(lr){12-13} \cmidrule(lr){14-15}
     & R@20 & min & R@20 & min & R@20 & min & R@20 & min & R@20 & min & R@20 & min & R@20 & min \\
    \midrule
    \textsc{UniLP} & 70.9 & 57 & 85.2 & 45 & \textbf{94.2} & 112 & \textbf{83.4} & 450 & 75.4 & 269 & \textbf{93.1} & 187 & \textbf{95.5} & 438 \\
    \textsc{Wander} \tiny (w/o feat.) & \textbf{77.2} & \textbf{3} & \textbf{90.6} & \textbf{3} & 89.4 & \textbf{6} & 80.7 & \textbf{13} & \textbf{76.4} & \textbf{15} & 91.3 & \textbf{45} & 92.2 & \textbf{52} \\
    \midrule
    \textsc{Wander} \tiny (w/ feat.) & -- & -- & -- & -- & -- & -- & -- & -- & 90.6 & 115 & 97.0 & 142 & 95.8 & 159 \\
    \bottomrule
    \end{tabular}

\vspace{4pt}
\end{table}

\subsection{Node classification on simple graphs}\label{app::node_classification}

Detailed results for the zero-shot node-classification evaluation are reported in
Table~\ref{tab:node-cls-multiclass}. We report test accuracy averaged over all
available data splits. The reported standard deviation reflects variation across
these splits and, for stochastic models (\textsc{NodePFN}, \textsc{GraphPFN}, and \textsc{Wander}),
additionally across three inference passes with different random seeds.

\paragraph{Datasets.}
We evaluate on a benchmark suite of 26 multiclass node-classification datasets,
using official data splits whenever available. The suite comprises Cornell,
Texas, Wisconsin, and Actor from \citet{pei2020geom}; the Brazil, Europe, and USA
airport networks from \citet{ribeiro2017struc2vec}; Chameleon and Squirrel from
\citet{rozemberczki2021multi}; Cora, CiteSeer, and PubMed from
\citet{yang2016revisiting}; Photo, Computers, Coauthor CS, and Coauthor Physics
from \citet{shchur2018pitfalls}; Wiki CS from \citet{mernyei2020wiki}; Full DBLP
and Full Cora from \citet{bojchevski2017deep}; Roman Empire and Amazon Ratings
from \citet{platonov2023critical}; HM Categories and Pokec Regions from
\citet{bazhenov2026graphland}; and Paris, Shanghai and LA from \citet{liang2026towards}.

As Pokec Regions contains approximately \(1.6\) million nodes, we construct a
smaller variant, \textit{Pokec 100k}, to facilitate evaluation. We randomly
select a training node as the root and retain the first \(100{,}000\) nodes
visited by a breadth-first search. Within the resulting connected subgraph, we
retain labels for the ten most frequent classes and treat nodes belonging to all
remaining classes as unlabeled. Labeled nodes retain their original GraphLand
random-low split assignments, while unlabeled nodes are excluded from all
splits.

Following the concerns about duplicate nodes in Chameleon and Squirrel raised
by \citet{platonov2023critical}, we use the filtered versions of
both datasets proposed by the authors.

Whenever official splits are provided, we evaluate on all available splits and
average the resulting test performance. For HM Categories, we use the
random-low split provided by GraphLand. The Brazil, Europe, USA, Photo,
Computers, Full DBLP, Full Cora, Coauthor CS, and Coauthor Physics datasets do
not provide official splits. For these datasets, following
\citet{zhao2024graphany,finkelshtein2026equivariance}, we sample \(20\) training
nodes per class and divide the remaining nodes equally between validation and
test sets. We repeat this procedure for five random seeds and report results aggregated across the resulting splits.

Table~\ref{tab:node-cls-statistics} summarizes the main statistics of the
evaluation datasets. In addition to standard graph and task statistics, we
report \emph{unbiased homophily} as proposed by \citet{mironov2024revisiting}.

\begin{table}[htbp]
    \scriptsize
    \centering
    \caption{Statistics of considered node classification datasets. ``\# official splits'' counts train/val/test masks shipped with the dataset. ``--'' denotes datasets without an official split; we use a custom protocol there (5 seeds, 20 labeled nodes per class).}
    \label{tab:node-cls-statistics}
    \setlength{\tabcolsep}{3pt}
    \begin{tabular}{clrrrrrrc}
    \toprule
    Group & Dataset & \# nodes & \# features & \# classes & \% train & avg.\ deg. & \shortstack{unbiased\\homophily} & \# official splits \\
    \midrule
    \multirow{3}{*}{\rotatebox{0}{\scriptsize WebKB}}
    & Cornell         & 183   & 1.7k & 5  & 47.5\% & 3.0   & -0.47 & 10 \\
    & Texas           & 183   & 1.7k & 5  & 47.5\% & 3.0   & -0.81 & 10 \\
    & Wisconsin       & 251   & 1.7k & 5  & 47.8\% & 3.6   & -0.29 & 10 \\
    \midrule
    \multirow{3}{*}{\rotatebox{0}{ Airports}}
    & Air Brazil      & 131   & 131  & 4  & 61.1\% & 15.3  & 0.00 & -- \\
    & Air Europe      & 399   & 399  & 4  & 20.1\% & 30.0  & -0.12 & -- \\
    & Air USA         & 1.2k  & 1.2k & 4  & 6.7\%  & 22.9  & 0.43 & -- \\
    \midrule
    \multirow{2}{*}{\shortstack{Wiki \\filtered}}
    & Chameleon filt. & 890   & 2.3k & 5  & 46.0\% & 19.9  & 0.08 & 10 \\
    & Squirrel filt.  & 2.2k  & 2.1k & 5  & 47.4\% & 42.3  & 0.02 & 10 \\
    \midrule
    & Actor           & 7.6k  & 932  & 5  & 48.0\% & 7.0   & 0.01 & 10 \\
    \midrule
    \multirow{3}{*}{\rotatebox{0}{ Planetoid}}
    & Cora            & 2.7k  & 1.4k & 7  & 5.2\%  & 3.9   & 0.92 & 1  \\
    & CiteSeer        & 3.3k  & 3.7k & 6  & 3.6\%  & 2.7   & 0.85 & 1  \\
    & PubMed          & 19.7k & 500  & 3  & 0.3\%  & 4.5   & 0.77 & 1  \\
    \midrule
    \multirow{2}{*}{\rotatebox{0}{ Amazon}}
    & Amz. Photos     & 7.7k  & 745  & 8  & 2.1\%  & 31.1  & 0.93 & -- \\
    & Amz. Computers  & 13.8k & 767  & 10 & 1.5\%  & 35.8  & 0.91 & -- \\
    \midrule
    & Wiki CS         & 11.7k & 300  & 10 & 5.0\%  & 36.9  & 0.86 & 20 \\
    \midrule
    \multirow{2}{*}{\rotatebox{0}{\shortstack{Citation\\full}}}
    & Full DBLP       & 17.7k & 1.6k & 4  & 0.5\%  & 6.0   & 0.82 & -- \\
    & Full Cora       & 19.8k & 8.7k & 70 & 7.1\%  & 6.4   & 0.97 & -- \\
    \midrule
    & Roman Empire    & 22.7k & 300  & 18 & 50.0\% & 2.9   & -0.49 & 10 \\
    \midrule
    & Amz. Ratings    & 24.5k & 300  & 5  & 50.0\% & 7.6   & 0.28 & 10 \\
    \midrule
    \multirow{2}{*}{\rotatebox{0}{ Coauthor}}
    & Co. CS          & 18.3k & 6.8k & 15 & 1.6\%  & 8.9   & 0.96 & -- \\
    & Co. Physics     & 34.5k & 8.4k & 5  & 0.3\%  & 14.4  & 0.95 & -- \\
    \midrule
    & HM Categories   & 46.6k & 32   & 21 & 10.0\% & 460.9 & 0.38 & 1  \\
    \midrule
    & Pokec 100k      & 100k  & 56   & 10 & 4.0\%  & 20.6  & 0.82 & 1  \\
    \midrule
    \multirow{3}{*}{\rotatebox{0}{\shortstack{City \\networks}}}
    & Paris           & 114.1k & 37   & 10 & 10.0\% & 3.2   & 0.87 & 1  \\
    & Shanghai        & 183.9k & 37   & 10 & 10.0\% & 2.9   & 0.89 & 1  \\
    & LA              & 240.6k & 37   & 10 & 10.0\% & 2.8   & 0.89 & 1  \\
    \midrule
    \multirow{2}{*}{\rotatebox{0}{\scriptsize Synth.}}
    & Cycles          & 5.0k  & 0    & 4  & 10.0\% & 3.0   & 0.75 & 1  \\
    & Grids           & 3.1k  & 0    & 5  & 0.3\%  & 3.8   & 1.00 & 1  \\
    \bottomrule
    \end{tabular}
\end{table}
\begin{table}[htbp]
    \scriptsize
    \centering
    \caption{\textsc{GCN} grid setting chosen for each dataset.
    Shared architecture: 3 layers, hidden dimension 256, weight decay $5 \times 10^{-4}$.}
    \label{tab:gcn-grid-hparams}
    \setlength{\tabcolsep}{3pt}
    \begin{tabular}{clcccc}
    \toprule
    Group & Dataset & Learning rate & Dropout rate & Layer norm & Skip connection \\
    \midrule
    \multirow{3}{*}{\rotatebox{0}{ \scriptsize WebKB}}
    & Cornell         & $10^{-3}$ & 0.1 & -- & -- \\
    & Texas           & $10^{-3}$ & 0.1 & -- & -- \\
    & Wisconsin       & $10^{-3}$ & 0.0 & \checkmark & \checkmark \\
    \midrule
    \multirow{3}{*}{\rotatebox{0}{ Airports}}
    & Brazil      & $10^{-4}$ & 0.1 & \checkmark & -- \\
    & Europe      & $3 \times 10^{-4}$ & 0.1 & -- & -- \\
    & USA         & $10^{-3}$ & 0.2 & -- & -- \\
    \midrule
    \multirow{2}{*}{\shortstack{Wiki \\filtered}}
    & Chameleon filt. & $10^{-3}$ & 0.0 & \checkmark & \checkmark \\
    & Squirrel filt.  & $3 \times 10^{-4}$ & 0.1 & \checkmark & -- \\
    \midrule
    & Actor           & $3 \times 10^{-3}$ & 0.2 & -- & \checkmark \\
    \midrule
    \multirow{3}{*}{\rotatebox{0}{ Planetoid}}
    & Cora            & $3 \times 10^{-3}$ & 0.1 & -- & -- \\
    & CiteSeer        & $3 \times 10^{-3}$ & 0.1 & \checkmark & -- \\
    & PubMed          & $3 \times 10^{-3}$ & 0.0 & -- & -- \\
    \midrule
    \multirow{2}{*}{\rotatebox{0}{ Amazon}}
    & Amz. Photos     & $10^{-3}$ & 0.2 & \checkmark & -- \\
    & Amz. Computers  & $3 \times 10^{-4}$ & 0.0 & \checkmark & -- \\
    \midrule
    & Wiki CS         & $10^{-3}$ & 0.1 & -- & \checkmark \\
    \midrule
    \multirow{2}{*}{\rotatebox{0}{\shortstack{Citation\\full}}}
    & Full DBLP       & $3 \times 10^{-3}$ & 0.1 & \checkmark & -- \\
    & Full Cora       & $3 \times 10^{-3}$ & 0.2 & \checkmark & -- \\
    \midrule
    & Roman Empire    & $3 \times 10^{-3}$ & 0.2 & -- & \checkmark \\
    \midrule
    & Amz. Ratings    & $3 \times 10^{-3}$ & 0.2 & \checkmark & \checkmark \\
    \midrule
    \multirow{2}{*}{\rotatebox{0}{ Coauthor}}
    & Co. CS          & $3 \times 10^{-3}$ & 0.1 & -- & \checkmark \\
    & Co. Physics     & $10^{-3}$ & 0.1 & \checkmark & -- \\
    \midrule
    & HM Categories   & $3 \times 10^{-3}$ & 0.2 & \checkmark & -- \\
    \midrule
    & Pokec 100k      & $3 \times 10^{-3}$ & 0.2 & -- & -- \\
    \midrule
    \multirow{3}{*}{\rotatebox{0}{\shortstack{City \\networks}}}
    & Paris           & $3 \times 10^{-3}$ & 0.2 & \checkmark & -- \\
    & Shanghai        & $10^{-3}$ & 0.2 & \checkmark & -- \\
    & LA              & $3 \times 10^{-3}$ & 0.2 & \checkmark & -- \\
    \bottomrule
    \end{tabular}
\end{table}
\paragraph{Baselines.}
We compare against a strongly tuned three-layer GCN~\citep{kipf2016semi} with
hidden dimension 256, trained for 1000 epochs. For each dataset, we perform a
grid search over the learning rate
\(\{3\!\cdot\!10^{-5}, 10^{-4}, 3\!\cdot\!10^{-4}, 10^{-3}, 3\!\cdot\!10^{-3}\}\),
the use of layer normalization and skip connections, and dropout rates in
\(\{0, 0.1, 0.2\}\). Hyperparameters are selected based on validation performance
on a single split and subsequently fixed for evaluation on all available splits
of the corresponding dataset. The selected hyperparameters are reported in
Table~\ref{tab:gcn-grid-hparams}.

We further consider the graph foundation models \textsc{GraphAny}~\citep{zhao2024graphany},
\textsc{NodePFN}~\citep{choi2026learning}, and \textsc{GraphPFN}~\citep{eremeev2025graphpfn}, as
well as the tabular foundation model \textsc{TabPFN}v3~\citep{tabpfn3}. We reevaluate all
models using their official implementations. \textsc{NodePFN} supports at most 20
classes and can therefore not be evaluated on Full Cora (70 classes) or
HM Categories (21 classes). \textsc{GraphPFN} is architecturally restricted to 10 output
classes, but circumvents this restriction through error-correcting output codes
(ECOC); following its default evaluation protocol, we use an ensemble of 10
models. \textsc{Wander} natively supports an arbitrary number of
classes.

\paragraph{Preprocessing.}
For {GraphPFN} and \textsc{Wander}, we follow the transductive
preprocessing pipeline provided with {GraphPFN}. On the GraphLand datasets
HM Categories and Pokec 100k, numerical features are transformed using a
quantile-normal transformation, fraction-valued features using a
quantile-uniform transformation, and categorical features are retained as
ordinal encodings. On all remaining datasets, we remove feature columns that
are constant over the training nodes and standardize each remaining feature
column. For datasets with more than 64 feature dimensions, we subsequently
apply PCA to reduce the representation to 64 dimensions. This dimensionality
reduction limits memory consumption and more closely matches the feature
dimensionalities encountered by {GraphPFN} and \textsc{Wander} during
pretraining. An evaluation without dimension reduction (PCA) for GraphPFN resulted in a worse mean performance and was not used for the results in Table \ref{tab:node-cls-multiclass-summary}.

For {NodePFN} and {GraphAny}, we retain the preprocessing procedures provided by
their respective official implementations. \textsc{NodePFN} additionally performs model
selection over truncated-SVD feature dimensions
\(\{10,15,20,25,30\}\) and different numbers of smoothing layers. We also preserve
this procedure for datasets that are not included in the original \textsc{NodePFN}
evaluation code.

For TabPFNv3 we also apply PCA to reduce feature representations to 64 dimensions. An evaluation without dimension reduction resulted in a worse mean performance and was not used for the results in Table \ref{tab:node-cls-multiclass-summary}.

\paragraph{Inference parameters and scalable evaluation of \textsc{Wander}.}
In a standard \textsc{Wander} forward pass, predictions can be obtained for
nodes that are visited by the sampled random walks. At the same time, the
in-context learning mechanism relies on sufficiently broad coverage of labeled
context nodes. Obtaining such coverage with a single forward pass becomes
computationally expensive on large graphs, as it would require sampling a very
large number of walks.

We therefore implement a batch-wise node-classification inference procedure that
separates the processing of context and query nodes. First, we perform a
forward pass in which random walks are initialized exclusively from training
nodes, sampled in a class-stratified manner, and cache the resulting
intermediate representations of all visited training nodes. We then partition the test nodes into batches of at most 128 by greedy BFS: starting from the highest-degree remaining test node, we collect nearby test nodes in hop order (paths may go through non-test nodes) until the batch is full. This keeps the query nodes in a batch graph-local, so their random walks share neighborhood coverage. For each batch, random walks are initialized exclusively from its query nodes, while the
cached training-node representations are supplied as context during the
in-context learning update. This allows test nodes to be processed in batches
without repeatedly recomputing representations for the labeled context.

Although this procedure is primarily motivated by scalability on large graphs,
we use it consistently for all node-classification experiments. We adjust the
random-walk budget according to graph size. For graphs with at most \(1{,}000\),
\(1{,}001\)--\(10{,}000\), \(10{,}001\)--\(40{,}000\), and more than \(40{,}000\)
nodes, we use base walk counts of \(32\), \(64\), \(128\), and \(256\),
respectively. We use a walk length of \(32\) for graphs with at most
\(1{,}000\) nodes and \(128\) otherwise. In both passes, we retain the total budget of $3K_{\text{base}}$ walks per structural update, only their starting-node distribution differs from the standard sampling scheme of Appendix~\ref{app::method}.

\subsection{Analysis of dataset-specific performance gaps}
\label{app:performance_gaps}

While \textsc{Wander} performs strongly on average, its advantage is not
uniform across datasets. We investigate two possible contributors to these
performance gaps: insufficient random-walk coverage on particularly dense
graphs and mismatch between the pretraining prior and the structural regime
of the evaluation graph.

\paragraph{Dense graphs.}
HM Categories and DDI are among the datasets on which \textsc{Wander} trails
the strongest task-specific baseline. Both graphs are unusually dense
relative to most of the evaluation suite. In such graphs, a fixed number of
sampled walks explores only a small fraction of the many possible local
transitions, potentially yielding insufficient structural coverage. Consistent
with this explanation, increasing the base walk count on HM Categories from
\(256\) to \(2048\), using walks of length \(32\), raises node-classification
accuracy from \(48.7\%\) to \(52.1\%\). Similarly, increasing the base walk
count on DDI from \(128\) to \(512\) improves Recall@20 from \(13.0\%\) to
\(14.1\%\). These results indicate that part of the performance gap on dense
graphs can be reduced by allocating more inference-time computation. They
also suggest that adapting the walk budget not only to graph size, but also
to graph density, may provide a better inference policy.

\paragraph{Prior mismatch.}
A second possible source of weaker performance is mismatch between the
synthetic pretraining prior and particular evaluation graphs. As discussed
in Appendix~\ref{app::prior_experiments}, samples from the \textsc{Wander}
prior span both homophilous and heterophilous regimes but are concentrated
toward homophily. Correspondingly, replacing the synthetic component with
the \textsc{GraphPFN} prior improves performance on several heterophilous
node-classification datasets, despite reducing aggregate accuracy. This
suggests that the comparatively weaker results on some heterophilous graphs
may partly reflect the relative frequency of these regimes during
pretraining. However, this evidence is only suggestive: the priors also
differ in their graph structures and attribute-generation mechanisms, so the
effect cannot be attributed to homophily alone.

\begin{table}[htbp]
    \scriptsize
    \centering
    \caption{Multiclass node classification results (test accuracy \%, mean $\pm$ std over all available splits). \textsc{NodePFN} is shown but excluded from mean and mean rank (missing Full Cora and HM Categories). \textsc{TabPFN}v3 uses node features only (ignores the graph).}
    \label{tab:node-cls-multiclass}
    \setlength{\tabcolsep}{3pt}
    \begin{tabular}{clcccccc}
    \toprule
    Group & Dataset & \textsc{GCN} & \textsc{TabPFN}v3 & \textsc{GraphAny} & \textsc{NodePFN} & \textsc{GraphPFN} & \textbf{\textsc{Wander}} \\
    \midrule
    \multirow{3}{*}{\rotatebox{0}{ \scriptsize WebKB}}
    & Cornell         & 38.9 $\pm$ 5.8 & \underline{71.4 $\pm$ 4.2} & 58.6 $\pm$ 6.4 & 67.3 $\pm$ 4.6 & 70.3 $\pm$ 3.0 & \textbf{76.9 $\pm$ 3.4} \\
    & Texas           & 54.9 $\pm$ 6.4 & 71.1 $\pm$ 8.4 & 67.0 $\pm$ 6.0 & \underline{76.8 $\pm$ 5.3} & 70.5 $\pm$ 12.1 & \textbf{77.8 $\pm$ 7.1} \\
    & Wisconsin       & 60.2 $\pm$ 6.5 & 77.5 $\pm$ 2.5 & 63.5 $\pm$ 4.7 & 79.4 $\pm$ 5.2 & \underline{80.0 $\pm$ 5.2} & \textbf{80.8 $\pm$ 5.7} \\
    \midrule
    \multirow{3}{*}{\rotatebox{0}{ Airports}}
    & Air Brazil      & 46.2 $\pm$ 18.0 & 33.1 $\pm$ 9.3 & 35.4 $\pm$ 15.1 & 63.9 $\pm$ 7.9 & \textbf{74.6 $\pm$ 6.7} & \underline{74.1 $\pm$ 4.6} \\
    & Air Europe      & 47.0 $\pm$ 0.7 & 39.1 $\pm$ 3.0 & 41.6 $\pm$ 6.4 & 52.6 $\pm$ 4.3 & \textbf{58.0 $\pm$ 4.3} & \underline{55.5 $\pm$ 4.1} \\
    & Air USA         & 47.4 $\pm$ 1.6 & 23.8 $\pm$ 2.3 & 43.5 $\pm$ 1.9 & 60.9 $\pm$ 1.4 & \textbf{63.0 $\pm$ 1.7} & \underline{62.4 $\pm$ 1.8} \\
    \midrule
    \multirow{2}{*}{\shortstack{Wiki \\filtered}}
    & Chameleon filt. & 41.6 $\pm$ 4.9 & 46.9 $\pm$ 3.4 & 32.6 $\pm$ 3.7 & 46.0 $\pm$ 2.6 & \textbf{50.9 $\pm$ 4.1} & \underline{47.9 $\pm$ 3.8} \\
    & Squirrel filt.  & 33.2 $\pm$ 2.1 & 41.1 $\pm$ 2.4 & 25.7 $\pm$ 1.3 & 42.5 $\pm$ 0.9 & \textbf{48.0 $\pm$ 2.8} & \underline{43.9 $\pm$ 1.7} \\
    \midrule
    & Actor           & 33.1 $\pm$ 0.5 & \textbf{39.0 $\pm$ 1.0} & 29.1 $\pm$ 0.8 & 32.7 $\pm$ 1.4 & \underline{38.7 $\pm$ 1.1} & 36.4 $\pm$ 0.9 \\
    \midrule
    \multirow{3}{*}{\rotatebox{0}{ Planetoid}}
    & Cora            & 77.6 & 66.3 & 79.4 & \textbf{82.1 $\pm$ 0.3} & 74.6 $\pm$ 0.5 & \underline{80.7 $\pm$ 1.1} \\
    & CiteSeer        & 61.0 & 64.9 & \underline{68.4} & 67.9 $\pm$ 0.2 & \textbf{69.2 $\pm$ 0.4} & 66.7 $\pm$ 1.0 \\
    & PubMed          & 73.9 & 70.3 & 76.3 & \underline{78.6 $\pm$ 0.3} & 72.0 $\pm$ 0.7 & \textbf{78.9 $\pm$ 0.1} \\
    \midrule
    \multirow{2}{*}{\rotatebox{0}{ Amazon}}
    & Amz. Photos     & 90.2 $\pm$ 0.6 & 87.2 $\pm$ 0.9 & 90.7 $\pm$ 0.8 & 89.6 $\pm$ 0.8 & \underline{92.5 $\pm$ 0.4} & \textbf{93.1 $\pm$ 0.2} \\
    & Amz. Computers  & \textbf{85.5 $\pm$ 1.0} & 78.9 $\pm$ 1.2 & 83.0 $\pm$ 1.1 & 81.8 $\pm$ 0.7 & \underline{83.9 $\pm$ 1.1} & 83.5 $\pm$ 1.2 \\
    \midrule
    & Wiki CS         & \textbf{79.4 $\pm$ 0.6} & 73.8 $\pm$ 0.6 & 75.4 $\pm$ 0.7 & 76.0 $\pm$ 0.7 & 79.1 $\pm$ 0.6 & \underline{79.3 $\pm$ 0.4} \\
    \midrule
    \multirow{2}{*}{\rotatebox{0}{\shortstack{Citation\\full}}}
    & Full DBLP       & 73.7 $\pm$ 0.3 & 69.3 $\pm$ 1.2 & 70.9 $\pm$ 0.8 & \underline{75.3 $\pm$ 1.3} & 74.8 $\pm$ 0.6 & \textbf{79.9 $\pm$ 1.2} \\
    & Full Cora       & \underline{58.7 $\pm$ 0.7} & 49.0 $\pm$ 0.7 & 57.4 $\pm$ 0.4 & --- & 57.2 $\pm$ 0.6 & \textbf{59.9 $\pm$ 0.4} \\
    \midrule
    & Roman Empire    & \textbf{79.6 $\pm$ 0.4} & 65.7 $\pm$ 0.5 & 64.2 $\pm$ 0.6 & 48.8 $\pm$ 0.5 & \underline{67.8 $\pm$ 0.4} & 65.4 $\pm$ 0.5 \\
    \midrule
    & Amz. Ratings    & \textbf{51.2 $\pm$ 0.7} & \underline{49.1 $\pm$ 0.5} & 42.5 $\pm$ 0.5 & 44.2 $\pm$ 0.5 & 48.6 $\pm$ 0.5 & 49.0 $\pm$ 0.5 \\
    \midrule
    \multirow{2}{*}{\rotatebox{0}{ Coauthor}}
    & Co. CS          & \underline{91.4 $\pm$ 0.5} & 90.6 $\pm$ 0.3 & 90.6 $\pm$ 0.5 & 90.8 $\pm$ 0.2 & 91.3 $\pm$ 0.4 & \textbf{92.0 $\pm$ 0.2} \\
    & Co. Physics     & 92.6 $\pm$ 0.6 & 90.9 $\pm$ 0.9 & \underline{92.7 $\pm$ 0.5} & 92.5 $\pm$ 0.6 & 90.9 $\pm$ 0.9 & \textbf{92.9 $\pm$ 0.9} \\
    \midrule
    & HM Categories   & \textbf{66.4} & 42.1 & 37.1 & --- & \underline{55.2 $\pm$ 0.7} & 48.7 $\pm$ 0.1 \\
    \midrule
    & Pokec 100k      & 78.4 & 43.2 & 44.0 & 78.1 $\pm$ 0.5 & \underline{83.8 $\pm$ 0.2} & \textbf{86.1 $\pm$ 0.5} \\
    \midrule
    \multirow{3}{*}{\rotatebox{0}{ City}}
    & City Paris      & 40.6 & \underline{48.5} & 20.2 & 44.8 & 46.9 $\pm$ 0.6 & \textbf{52.8 $\pm$ 0.2} \\
    & City LA         & 35.9 & \underline{44.8} & 17.1 & 44.3 & 44.3 $\pm$ 0.8 & \textbf{55.2 $\pm$ 0.2} \\
    & City Shanghai   & 38.1 & \underline{47.6} & 19.3 & 44.7 & 44.2 $\pm$ 0.5 & \textbf{53.5 $\pm$ 0.1} \\
    \midrule
    & Mean accuracy (\%) & 60.64 & 58.66 & 54.85 & --- & \underline{66.55} & \textbf{68.20} \\
    & Mean rank          & 3.12 & 3.69 & 3.96 & --- & \underline{2.42} & \textbf{1.73} \\
    \bottomrule
    \end{tabular}
\end{table}

\subsection{Pretraining on single task}
\label{app:single-task-pretraining}

For the single-task variants in Table~\ref{tab:joint-training}, we follow the same
pretraining setup and curriculum as described in Appendix~\ref{app::implementation},
but restrict training to the respective task. In each curriculum stage, we train for
exactly the number of optimization steps that this task receives in the joint setting.
For example, a task assigned \(40\%\) of the steps in a given stage is trained for
\(40\%\) of that stage's total number of steps. All other hyperparameters and
task-specific data distributions are kept unchanged.

\subsection{Compositional generalization to attributed knowledge graphs}

During pretraining, \textsc{Wander} encounters node features on homogeneous
graphs and relation types on knowledge graphs, but never both jointly. We test
whether these capabilities compose at inference time by evaluating the same
pretrained checkpoint zero-shot on ConceptNet100k, WN18RR, and the four
inductive WN-v1--v4 graphs in a \(2\times2\) design with MiniLM node features
and relation types enabled or disabled. Without relation types, all relations,
including inverses, are mapped to a single dummy type; without features,
entities are attribute-free. We use filtered head and tail ranking and the
default knowledge-graph inference settings from Appendix~\ref{app::kg-evaluation} and Table~\ref{tab:kg-walk-budgets}.
Table~\ref{tab:kg-with-features} reports the unweighted mean MRR across WN-v1--v4.

We encode ConceptNet entity strings and WordNet synsets using
\texttt{all-MiniLM-L6-v2} \cite{wang2020minilm}, representing each synset as
``\textit{lemma}: \textit{gloss}.'' The resulting 384-dimensional embeddings
are reduced to 32 dimensions using PCA. 

As shown in Table~\ref{tab:kg-with-features}, node features and relation types each improve MRR
individually, while their combination performs best on every benchmark.
Relation types contribute most strongly, particularly on WordNet, whereas
MiniLM features provide the largest gains on ConceptNet100k and the inductive
WN-v1--v4 graphs.
\subsection{Inference-Time Adaptation}
\label{app:inference-time-adaptation}

\paragraph{Synthetic \textsc{Grids} experiment.}
To study how the structural range of \textsc{Wander} can be adapted at inference
time, we construct the featureless synthetic benchmark \textsc{Grids} as the
disjoint union of five \(25\times25\) four-neighbor grids. The resulting graph
contains \(3{,}125\) nodes and \(6{,}000\) undirected edges, with each connected
component corresponding to one class and having diameter \(48\). We sample two
training nodes per grid, for \(10\) labeled nodes in total, and randomly split the
remaining nodes into \(400\) validation and \(2{,}715\) test nodes. The nearest
labeled node is on average \(11\) hops from a test node and at most \(28\) hops
away. Since no node features are provided and all five components are
structurally identical, prediction requires propagating the sparse supervision
available elsewhere in the same connected component.

We evaluate the pretrained \textsc{Wander} checkpoint zero-shot, without any
dataset-specific finetuning, using the node-classification inference
procedure from Appendix~\ref{app::node_classification}. To vary the available structural context
without changing the model parameters, we evaluate the same checkpoint with
walk lengths \(4\), \(16\), and \(64\), while keeping the remaining inference
settings fixed. For the context-aware sampling variant, we use length-\(64\)
walks, but retain only walks that visit at least one labeled context node. The
visualizations in Figure~\ref{fig:grids_experiment} show the predicted probability of the true
class for one of the five grid components.

\begin{wraptable}{r}{0.6\columnwidth}
    \vspace{-10pt}
    \centering
    \scriptsize
    \caption{Test accuracy on the synthetic \textsc{Grids} dataset and selected real-world graphs without node features.}
    \label{tab:long_range}
    \setlength{\tabcolsep}{2pt}
    \begin{tabular}{lccccc}
        \toprule
        
        & \textsc{Grids}
        & PubMed
        & \makecell[c]{Full \\ DBLP}
        & Co. CS
        & \makecell[c]{Co.\\Physics} \\
        \midrule
        \textsc{GraphAny}
        & 20.0
        & 18.0
        & 44.7 $\pm$ 0.0
        & 22.6 $\pm$ 0.0
        & 50.5 $\pm$ 0.0 \\
        \textsc{NodePFN}
        & 26.6 $\pm$ 0.0
        & 39.5 $\pm$ 0.0
        & 45.9 $\pm$ 1.0
        & 28.3 $\pm$ 0.5
        & 51.2 $\pm$ 0.3 \\
        \textsc{GraphPFN}
        & 23.3 $\pm$ 0.7
        & 49.9 $\pm$ 2.6
        & 61.8 $\pm$ 1.6
        & 59.2 $\pm$ 3.5
        & 69.4 $\pm$ 1.2 \\
        \midrule
        \textbf{\textsc{Wander}}
        & 96.6 $\pm$ 0.5
        & \textbf{73.7 $\pm$ 0.7}
        & 65.9 $\pm$ 3.9
        & 84.1 $\pm$ 0.6
        & 82.2 $\pm$ 5.0 \\
        \textbf{\textsc{Wander} \tiny(filtered)}
        & \textbf{100.0 $\pm$ 0.0}
        & 72.8 $\pm$ 0.5
        & \textbf{70.0 $\pm$ 4.6}
        & \textbf{84.6 $\pm$ 0.6}
        & \textbf{85.8 $\pm$ 2.7} \\
        \bottomrule
    \end{tabular}
    \vspace{-8pt}
\end{wraptable}

\paragraph{Featureless real-world graphs.}
We additionally evaluate on PubMed, Full DBLP, Coauthor CS, and Coauthor
Physics, whose training splits provide comparatively sparse supervision
(see Table~\ref{tab:node-cls-statistics}). We remove all node features
and retain only graph structure and the observed training labels. We compare
\textsc{GraphAny}, \textsc{NodePFN},
\textsc{GraphPFN}, and \textsc{Wander} in the zero-shot setting,
without dataset-specific finetuning. The baseline GFMs are evaluated using their
official implementations and preprocessing procedures, while \textsc{Wander}
uses the graph-size-dependent random-walk
budget described in Appendix~\ref{app::node_classification}. We report test accuracy averaged over
all available splits.

\subsection{Inference-time efficiency and compute--performance trade-offs}
\label{app::runtime}

\paragraph{Runtime.}
Table~\ref{tab:inference-time-cora-dblp} compares the wall-clock cost of zero-shot
\textsc{Wander} inference with that of training and evaluating
dataset-specific end-to-end models. We consider Cora link prediction with
\textsc{NBFNet} and Full DBLP node classification with a GCN. All measurements
were obtained on a single NVIDIA RTX PRO 6000 GPU. For \textsc{Wander}, we use the default inference configuration: batches of
eight source nodes with \(64\) walks of length \(64\) for link prediction, and
batches of \(128\) nodes with \(128\) walks of length \(128\) for node
classification.
The reported end-to-end runtimes include hyperparameter selection.

Scoring an individual batch with \textsc{Wander} takes only a few seconds.
Nevertheless, evaluating an entire graph remains computationally substantial,
particularly for node classification. With the default \(16\) inference samples, a full
evaluation takes approximately \(11\) minutes on Cora and \(21\) minutes on
Full DBLP. The number of inference samples can, however, be reduced without
retraining the model. With a single inference sample, the corresponding
runtimes decrease to \(37\) seconds and \(78\) seconds, respectively, making
zero-shot \textsc{Wander} inference faster than training and selecting the
dataset-specific baselines on both datasets.

\begin{table}[htbp]
    \centering
    \scriptsize
    \caption{Wall-clock runtimes on a single NVIDIA RTX PRO 6000 GPU. Times are
reported for one seed and, for Full DBLP, one data split.}
    \label{tab:inference-time-cora-dblp}
    \setlength{\tabcolsep}{4pt}
    \begin{tabular}{l cc cc}
        \toprule
         & \multicolumn{2}{c}{Cora (LP)}
         & \multicolumn{2}{c}{Full DBLP (NC)} \\
        \cmidrule(lr){2-3}\cmidrule(lr){4-5}
         & \textsc{Wander} & \textsc{NBFNet}
         & \textsc{Wander} & \textsc{GCN} \\
        \midrule
        Grid search
            & -- & 13 min 36 s
            & -- & 1 min 28 s \\
        Single batch (1 inference sample)
            & 1.70 s & 1.1 ms
            & 1.13 s & 2.0 ms \\
        Single batch (16 inference sample)
            & 27.2 s & --
            & 18.1 s & -- \\
        Full eval (1 inference sample)
            & 37 s & 13 min 36 s
            & 1 min 18 s & 1 min 28 s \\
        Full eval (16 inference samples)
            & 10 min 58 s & --
            & 21 min 16 s & -- \\
        \bottomrule
    \end{tabular}
\end{table}

\paragraph{Compute--performance trade-offs.}
A useful property of \textsc{Wander} is that its inference-time computation can
be adjusted without retraining: we may vary the number of sampled walks, their
length, and the number of inference samples whose logits are averaged.
Figure~\ref{fig:compute_tradeoff} quantifies the resulting trade-off between
predictive performance and inference speed. Starting from the default
configuration, we independently reduce the number of walks and the walk length
to one half and one quarter of their dataset-specific defaults, or reduce the
number of inference samples from \(16\) to \(8\) or \(1\). For node
classification, changes to the random-walk configuration are applied to both
the context and query passes of the scalable inference procedure. All other
settings remain fixed. We report performance and wall-clock speedup relative
to the default configuration, averaged across the homogeneous link-prediction
and node-classification benchmarks, respectively.

Reducing the number of inference samples offers the most favorable trade-off:
eight samples roughly halve runtime with little performance loss, while one
sample provides a larger speedup at a higher cost in performance. Reducing the
number of walks also degrades performance gradually, whereas shorter walks
cause the largest drop, particularly for link prediction.

\begin{figure}[htbp]
    \centering
    \includegraphics[width=\linewidth]{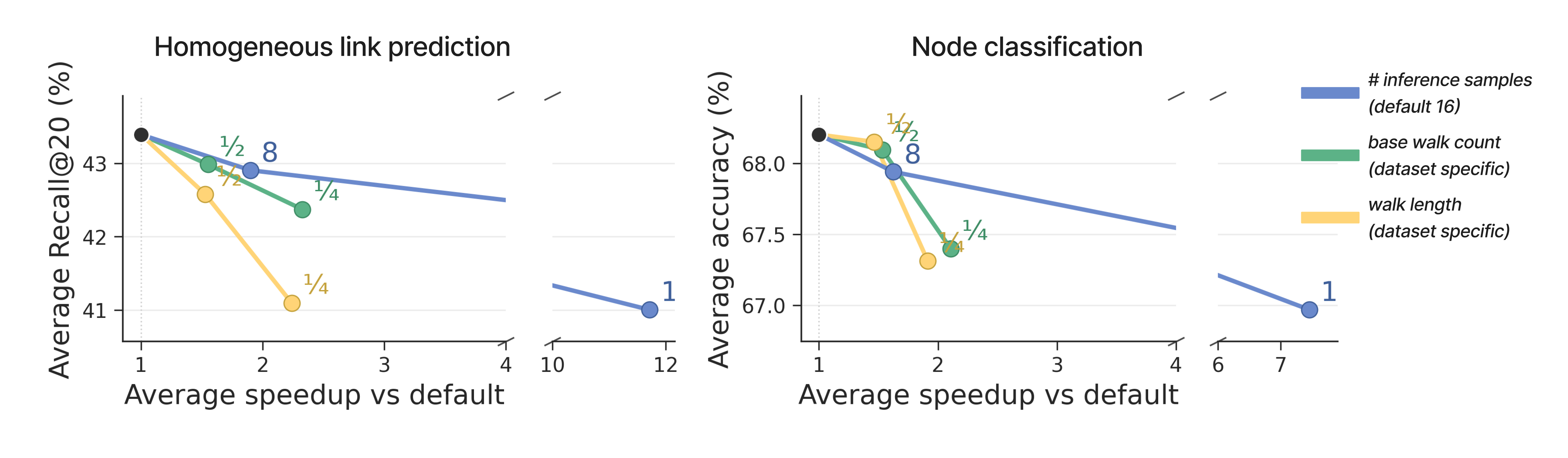}
    \caption{Compute--performance trade-offs obtained by reducing the number of
walks, walk length, or inference samples relative to the default configuration.
Results are averaged across the datasets of each benchmark.}
    \label{fig:compute_tradeoff}
\end{figure}

\subsection{Higher-order structural patterns}\label{app::cycles}
To evaluate \textsc{Wander}'s ability to identify higher-order structural
motifs that are challenging for GNNs bounded by the expressive power of
1-WL, we construct a synthetic four-class node-classification benchmark,
\textsc{Cycles}.

\paragraph{Dataset.}
The graph is featureless, exactly \(3\)-regular, and contains \(n=5000\)
nodes. We plant \(200\) triangles, \(150\) four-cycles, and \(120\)
five-cycles as pairwise vertex-disjoint motifs, yielding \(600\) nodes in
each motif class. Each motif node is connected by one edge to a distinct
node in the remaining \(3200\)-node backbone.

The backbone is completed by matching the remaining degree stubs while
rejecting self-loops, duplicate edges, and edges between nodes at graph
distance of most three. We further reject graphs containing any
\(C_3\), \(C_4\), or \(C_5\) outside the planted motifs, ensuring that all
such cycles are known by construction. Nodes are labeled by membership in
the backbone, \(C_3\), \(C_4\), or \(C_5\). We use stratified
\(10\%/20\%/70\%\) train/validation/test splits.

\begin{table}[t]
\centering
\scriptsize
\caption{Accuracy on \textsc{Cycles} (finetuned).}
\label{tab:cycles}
\begin{tabular}{lc}
    \toprule
    & \textsc{Cycles} \\
    \midrule
    Majority & 64.0 \\
    \textsc{GCN} & 63.3 \\
    \textsc{GraphPFN} & 67.4 \\
    \midrule
    \textsc{Wander} & \textbf{97.5} \\
    \bottomrule
\end{tabular}
\end{table}

\paragraph{Evaluation protocol.}
We compare against GCN~\citep{kipf2016semi} and
\textsc{GraphPFN}~\citep{eremeev2025graphpfn}. GCN uses the same
hyperparameter search as in Appendix~\ref{app::node_classification}, while
\textsc{GraphPFN} is finetuned using the authors' official procedure.
For \textsc{Wander}, we use walk length \(16\) and base walk count \(32\)
during finetuning and evaluation, and finetune for \(200\) epochs with
\(32\) batches per epoch and batch size \(16\). All other hyperparameters
match pretraining, only training nodes contribute to the supervised loss,
and checkpoint selection is based on validation performance.

As shown in Table~\ref{tab:cycles}, \textsc{GraphPFN} reaches \(67.4\%\),
only modestly above the \(64.0\%\) majority baseline, whereas
\textsc{Wander} reaches \(97.5\%\), almost perfectly separating nodes by
their cycle participation.

\subsection{Finetuning}
\label{app:finetuning}

We additionally evaluate how much \textsc{Wander} benefits from
dataset-specific finetuning on the link-prediction benchmarks.

\paragraph{Knowledge graph link prediction.}
We finetune using the pretraining learning rate of \(5\times10^{-5}\), no
weight decay, and \(256\) negative samples per query. Random-walk sampling
follows the default knowledge-graph evaluation protocol described in
Appendix~\ref{app::kg-evaluation}. We train on the full training split for up to five epochs
and select the checkpoint with the highest validation MRR. The resulting
performance is reported in Table~\ref{tab::full_kg_finetuning} and compared to the finetuned performance of \textsc{Ultra}, \textsc{Trix} and \textsc{Flock} as reported in the respective papers.

\begin{table}[htbp]
    \scriptsize
        \centering
        \caption{Finetuned knowledge graph link prediction results. Best in each row is \textbf{bold}, second-best is \underline{underlined}.}
        \label{tab::full_kg_finetuning}
        \setlength{\tabcolsep}{3pt}
        \begin{tabular}{clcccccc|cccc}
        \toprule
        & Dataset & \multicolumn{2}{c}{\textsc{Ultra} (ft.)} & \multicolumn{2}{c}{\textsc{Trix} (ft.)}
        & \multicolumn{2}{c}{\textsc{Flock} (ft.)} & \multicolumn{2}{c}{{\textsc{Wander}}} &\multicolumn{2}{c}{{\textsc{Wander}} (ft.)}\\
        & & MRR & H@10 & MRR & H@10 & MRR & H@10 & MRR & H@10 & MRR & H@10\\
        \midrule
        \multirow{23}{*}{\rotatebox{90}{Inductive (entity, relation)}}
    &FB-25       & 0.383 & 0.635 & 0.393 & 0.650 & \underline{0.405} & \textbf{0.666} & 0.404 & 0.663 & \textbf{0.406} & \underline{0.665} \\
    &FB-50       & 0.334 & 0.538 & 0.334 & 0.547 & \textbf{0.357} & \textbf{0.570} & \underline{0.352} & \underline{0.568} & 0.345 & 0.562 \\
    &FB-75       & 0.400 & 0.598 & 0.401 & 0.611 & \textbf{0.425} & \underline{0.630} & 0.408 & 0.617 & \underline{0.424} & \textbf{0.631} \\
    &FB-100      & 0.444 & 0.643 & 0.436 & 0.633 & \textbf{0.460} & \textbf{0.668} & \underline{0.458} & 0.663 & \underline{0.458} & \underline{0.666} \\
    &WK-25       & \textbf{0.321} & \textbf{0.535} & \underline{0.300} & 0.493 & 0.298 & \underline{0.506} & 0.288 & 0.477 & 0.293 & 0.493 \\
    &WK-50       & 0.140 & 0.280 & \textbf{0.166} & \textbf{0.313} & 0.127 & 0.260 & \underline{0.147} & \underline{0.289} & 0.126 & 0.272 \\
    &WK-75       & 0.380 & \underline{0.530} & 0.368 & 0.513 & \textbf{0.405} & \textbf{0.556} & 0.372 & 0.517 & \underline{0.401} & \textbf{0.556} \\
    &WK-100      & 0.168 & 0.286 & \underline{0.188} & 0.299 & 0.187 & \underline{0.306} & \textbf{0.190} & \textbf{0.309} & 0.182 & 0.297 \\
    &NL-0        & 0.329 & 0.551 & 0.385 & 0.549 & \textbf{0.418} & \underline{0.619} & 0.351 & 0.553 & \underline{0.403} & \textbf{0.628} \\
    &NL-25       & \textbf{0.407} & \underline{0.596} & 0.377 & 0.589 & \underline{0.405} & \textbf{0.626} & 0.338 & 0.546 & 0.390 & \underline{0.596} \\
    &NL-50       & \textbf{0.418} & \underline{0.595} & \underline{0.405} & 0.555 & 0.391 & 0.562 & 0.391 & 0.579 & 0.403 & \textbf{0.605} \\
    &NL-75       & \textbf{0.374} & \underline{0.570} & 0.351 & 0.525 & 0.344 & 0.544 & 0.357 & 0.555 & \underline{0.360} & \textbf{0.582} \\
    &NL-100      & 0.458 & 0.684 & 0.482 & 0.691 & 0.486 & \textbf{0.714} & \underline{0.487} & \underline{0.707} & \textbf{0.495} & 0.706 \\
    &MT1 tax     & 0.330 & 0.459 & 0.397 & \underline{0.508} & \textbf{0.413} & 0.497 & 0.281 & 0.410 & \underline{0.405} & \textbf{0.519} \\
    &MT1 health  & 0.380 & 0.467 & 0.376 & 0.457 & \textbf{0.394} & \textbf{0.493} & \underline{0.391} & 0.483 & \textbf{0.394} & \underline{0.492} \\
    &MT2 org     & 0.104 & 0.170 & 0.098 & 0.162 & \underline{0.107} & \underline{0.174} & 0.102 & 0.167 & \textbf{0.108} & \textbf{0.176} \\
    &MT2 sci     & 0.311 & 0.451 & 0.331 & \underline{0.526} & \underline{0.366} & 0.525 & 0.359 & \underline{0.526} & \textbf{0.378} & \textbf{0.542} \\
    &MT3 art     & 0.306 & 0.473 & 0.289 & 0.441 & \textbf{0.330} & \underline{0.483} & 0.304 & 0.466 & \underline{0.329} & \textbf{0.491} \\
    &MT3 infra   & 0.657 & 0.807 & 0.672 & 0.810 & \underline{0.709} & \underline{0.838} & 0.677 & 0.822 & \textbf{0.713} & \textbf{0.846} \\
    &MT4 sci     & 0.303 & 0.478 & 0.305 & 0.482 & \underline{0.324} & \underline{0.509} & 0.310 & 0.480 & \textbf{0.327} & \textbf{0.512} \\
    &MT4 health  & 0.704 & 0.785 & 0.702 & 0.785 & \underline{0.711} & \textbf{0.790} & 0.701 & \underline{0.787} & \textbf{0.714} & \textbf{0.790} \\
    &Metafam     & \underline{0.997} & \textbf{1.000} & \underline{0.997} & \textbf{1.000} & 0.992 & \textbf{1.000} & 0.504 & \underline{0.941} & \textbf{1.000} & \textbf{1.000} \\
    &FBNELL      & 0.481 & 0.661 & 0.478 & 0.655 & \textbf{0.531} & \textbf{0.714} & \underline{0.526} & \underline{0.712} & 0.525 & \underline{0.712} \\

        \midrule
        \multirow{18}{*}{\rotatebox{90}{ Inductive (entity)}}
    &FB-v1       & 0.509 & 0.670 & 0.515 & 0.682 & \textbf{0.549} & \textbf{0.721} & \underline{0.524} & \underline{0.701} & \textbf{0.549} & \textbf{0.721} \\
    &FB-v2       & 0.524 & 0.710 & 0.525 & 0.730 & \underline{0.553} & \underline{0.754} & 0.540 & 0.742 & \textbf{0.560} & \textbf{0.755} \\
    &FB-v3       & 0.504 & 0.663 & 0.501 & 0.669 & \textbf{0.528} & \textbf{0.696} & 0.514 & 0.693 & \underline{0.525} & \underline{0.695} \\
    &FB-v4       & 0.496 & 0.684 & 0.493 & 0.687 & \underline{0.510} & 0.702 & 0.504 & \underline{0.703} & \textbf{0.515} & \textbf{0.705} \\
    &WN-v1       & 0.685 & 0.793 & 0.705 & 0.798 & \textbf{0.715} & \underline{0.811} & 0.690 & 0.807 & \underline{0.713} & \textbf{0.814} \\
    &WN-v2       & 0.679 & 0.779 & 0.682 & 0.780 & \underline{0.702} & \underline{0.795} & 0.687 & 0.793 & \textbf{0.706} & \textbf{0.801} \\
    &WN-v3       & 0.411 & 0.546 & 0.425 & 0.543 & \underline{0.494} & \textbf{0.627} & 0.478 & \textbf{0.627} & \textbf{0.496} & \underline{0.625} \\
    &WN-v4       & 0.614 & 0.720 & 0.650 & 0.722 & \textbf{0.665} & \underline{0.741} & 0.650 & 0.731 & \underline{0.663} & \textbf{0.744} \\
    &NL-v1       & 0.757 & 0.878 & \textbf{0.804} & \underline{0.899} & 0.762 & \textbf{0.928} & 0.677 & 0.878 & \underline{0.789} & 0.888 \\
    &NL-v2       & 0.575 & 0.761 & 0.571 & 0.764 & \underline{0.612} & \underline{0.806} & 0.587 & 0.794 & \textbf{0.616} & \textbf{0.812} \\
    &NL-v3       & 0.563 & 0.755 & 0.571 & 0.759 & 0.606 & \underline{0.803} & \underline{0.610} & 0.796 & \textbf{0.620} & \textbf{0.809} \\
    &NL-v4       & 0.469 & 0.733 & 0.551 & 0.772 & \underline{0.572} & \textbf{0.801} & \underline{0.572} & 0.795 & \textbf{0.587} & \underline{0.798} \\
    &HM 1k       & 0.042 & 0.100 & \textbf{0.072} & 0.128 & \underline{0.071} & \textbf{0.153} & 0.062 & 0.119 & 0.066 & \underline{0.132} \\
    &HM 3k       & 0.030 & 0.090 & \underline{0.069} & 0.119 & 0.067 & \textbf{0.153} & \underline{0.069} & 0.126 & \textbf{0.078} & \underline{0.150} \\
    &HM 5k       & 0.025 & 0.068 & \textbf{0.074} & 0.118 & 0.061 & \underline{0.130} & \underline{0.067} & 0.126 & 0.060 & \textbf{0.136} \\
    &HM Indigo   & \underline{0.432} & \underline{0.639} & \textbf{0.436} & \textbf{0.645} & 0.418 & 0.633 & 0.416 & 0.636 & 0.411 & 0.636 \\
    &ILPC Small  & 0.303 & 0.453 & 0.303 & 0.455 & \underline{0.305} & 0.454 & \textbf{0.310} & \underline{0.459} & \textbf{0.310} & \textbf{0.462} \\
    &ILPC Large  & 0.308 & 0.431 & 0.310 & 0.431 & \underline{0.320} & \underline{0.441} & \textbf{0.321} & 0.440 & \textbf{0.321} & \textbf{0.442} \\

        \midrule
        \multirow{13}{*}{\rotatebox{90}{ Transductive}}
    &NELL995         & 0.509 & 0.660 & 0.506 & 0.648 & \underline{0.531} & \underline{0.665} & 0.519 & 0.662 & \textbf{0.565} & \textbf{0.681} \\
    &NELL23k         & 0.268 & 0.450 & \underline{0.306} & \underline{0.536} & 0.280 & 0.465 & 0.249 & 0.413 & \textbf{0.327} & \textbf{0.552} \\
    &WDsinger        & 0.417 & 0.526 & \textbf{0.502} & \underline{0.620} & \underline{0.435} & 0.543 & 0.402 & 0.521 & \textbf{0.502} & \textbf{0.633} \\
    &ConceptNet100k  & 0.310 & 0.529 & 0.340 & 0.564 & \underline{0.352} & \underline{0.580} & 0.258 & 0.477 & \textbf{0.381} & \textbf{0.614} \\
    &CoDEx Small     & \textbf{0.490} & \textbf{0.686} & \underline{0.484} & \underline{0.676} & 0.463 & 0.648 & 0.457 & 0.658 & 0.473 & 0.663 \\
    &CoDEx Large     & 0.343 & \underline{0.478} & \textbf{0.348} & \textbf{0.481} & 0.342 & 0.467 & 0.345 & 0.475 & \underline{0.347} & 0.477 \\
    &YAGO310         & \underline{0.557} & \underline{0.710} & 0.541 & 0.702 & 0.552 & 0.700 & 0.499 & 0.666 & \textbf{0.575} & \textbf{0.718} \\
    &AristoV4        & 0.343 & 0.496 & 0.345 & 0.499 & \underline{0.383} & \underline{0.523} & 0.295 & 0.424 & \textbf{0.387} & \textbf{0.530} \\
    &DBpedia100k     & 0.436 & 0.603 & 0.457 & 0.619 & \underline{0.470} & 0.623 & 0.436 & \underline{0.625} & \textbf{0.472} & \textbf{0.630} \\
    &Hetionet        & \textbf{0.399} & \textbf{0.538} & \underline{0.394} & \underline{0.534} & 0.314 & 0.465 & 0.279 & 0.415 & 0.338 & 0.478 \\
    &FB15k-237(10)   & \underline{0.254} & 0.411 & 0.253 & 0.408 & \textbf{0.260} & \underline{0.420} & 0.248 & 0.404 & \textbf{0.260} & \textbf{0.422} \\
    &FB15k-237(20)   & 0.274 & 0.445 & 0.273 & 0.441 & \textbf{0.284} & \textbf{0.459} & 0.277 & 0.451 & \underline{0.282} & \underline{0.457} \\
    &FB15k-237(50)   & \textbf{0.325} & \textbf{0.528} & 0.322 & 0.522 & 0.317 & 0.517 & \underline{0.323} & \underline{0.523} & 0.322 & \underline{0.523} \\

        \midrule
        & Mean & 0.407 & 0.562 & 0.418 & 0.569 & \underline{0.427} & \underline{0.583} & 0.399 & 0.565 & \textbf{0.433} & \textbf{0.590} \\
        & Mean rank & 3.796 & 3.898 & 3.435 & 3.657 & \underline{2.380} & \underline{2.324} & 3.463 & 3.370 & \textbf{1.926} & \textbf{1.750} \\
        \bottomrule
        \end{tabular}
        \end{table}

Finetuning increases mean MRR from \(0.399\) to \(0.433\) and Hits@10 from
\(0.565\) to \(0.590\). The finetuned model thereby surpasses the strongest
specialized baseline on average, showing that \textsc{Wander} retains its
advantage in knowledge-graph reasoning when all methods are adapted to each
target graph.

\paragraph{Homogeneous link prediction.}
We use the same finetuning protocol for homogeneous link prediction, with
random-walk sampling following the corresponding default evaluation setup.
Since not all datasets provide a validation split, we do not perform
dataset-specific checkpoint selection and instead evaluate the checkpoint
after five finetuning epochs for every dataset. Results are reported in
Table~\ref{tab:homo-lp-finetune}.

\begin{table}[t]
    \vspace{-10pt}
        \scriptsize
        \centering
        \caption{Finetuned homogeneous link prediction (Recall@20 \%).}
        \label{tab:homo-lp-finetune}
        \setlength{\tabcolsep}{3pt}
        \begin{tabular}{lcccccccccc|cc}
        \toprule
         & CiteSeer & Cora & PubMed & CS & DDI & \makecell{Products\\Home} & \makecell{P2P\\Gnutella} & \makecell{Email\\Enron} & \makecell{Proteins\\Spec1} & \makecell{SOC\\Epinions} & Mean & Mean rank \\
        \midrule
        \textsc{Wander} & 65.88 & 56.78 & 31.73 & 64.62 & 12.98 & 71.12 & 9.06 & 69.78 & 31.40 & 20.59 & 43.39 & 1.90 \\
        \midrule
        \textsc{Wander} (ft.) & 69.20 & 60.90 & 39.70 & 70.00 & 11.90 & 72.60 & 9.30 & 71.90 & 36.60 & 21.30 & 46.34 & 1.10 \\
        \bottomrule
        \end{tabular}
    \end{table}

Finetuning improves mean Recall@20 from \(43.39\%\) to \(46.34\%\), with
gains on nine of the ten datasets. The largest improvements occur on PubMed,
Coauthor CS, and Proteins, whereas performance decreases slightly on DDI.
Thus, the pretrained model generally adapts effectively to individual graphs
in the homogeneous link-prediction setting, although the benefit depends on
the target graph.

\paragraph{Node classification.}
We finetune the pretrained \textsc{Wander} checkpoint separately on each
dataset using only the official training nodes of split~0. Each epoch comprises
five optimizer steps with \(16\) query nodes each. To provide additional
supervision, at every step we also hide the labels of a randomly sampled
\(25\%\) of the training nodes, capped at \(1{,}024\), and compute the loss on
those hidden nodes encountered in at least one forward pass. During training,
we use \(128\) walks and the size-dependent walk length specified in
Appendix~\ref{app::node_classification}. We train for at most \(100\) epochs,
with early stopping after \(15\) epochs without improvement in validation
accuracy. For each dataset, we select the learning rate from \(10\)
logarithmically spaced values between \(5\times10^{-6}\) and
\(5\times10^{-4}\) based on validation accuracy. The selected checkpoint is
evaluated on the split-0 test nodes using the same size-tier walk budget as in
the zero-shot evaluation. We report the mean and standard deviation over three
seeds in Table~\ref{tab:nc-finetune}.

\begin{table}[htbp]
    \scriptsize
    \centering
    \caption{Node classification finetuning versus zero-shot \textsc{Wander} on split 0 only (test accuracy \%, mean $\pm$ std over 3 seeds). }
    \label{tab:nc-finetune}
    \setlength{\tabcolsep}{3pt}
    \begin{tabular}{clcc}
    \toprule
    Group & Dataset & Zero-shot & Finetuned \\
    \midrule
    \multirow{3}{*}{\rotatebox{0}{ \scriptsize WebKB}}
    & Cornell         & \textbf{74.77 $\pm$ 1.27} & 72.07 $\pm$ 1.27 \\
    & Texas           & 75.68 $\pm$ 2.21 & \textbf{80.18 $\pm$ 1.27} \\
    & Wisconsin       & \textbf{76.47 $\pm$ 0.00} & 71.24 $\pm$ 0.92 \\
    \midrule
    \multirow{3}{*}{\rotatebox{0}{ Airports}}
    & Air Brazil      & 78.21 $\pm$ 1.81 & \textbf{80.77 $\pm$ 0.00} \\
    & Air Europe      & 57.29 $\pm$ 2.06 & 57.29 $\pm$ 0.29 \\
    & Air USA         & 61.20 $\pm$ 0.08 & \textbf{61.62 $\pm$ 0.67} \\
    \midrule
    \multirow{2}{*}{\shortstack{Wiki \\filtered}}
    & Chameleon filt. & \textbf{50.52 $\pm$ 1.11} & 49.31 $\pm$ 0.64 \\
    & Squirrel filt.  & 43.58 $\pm$ 0.79 & \textbf{45.80 $\pm$ 0.65} \\
    \midrule
    & Actor           & 37.61 $\pm$ 0.20 & \textbf{39.65 $\pm$ 0.38} \\
    \midrule
    \multirow{3}{*}{\rotatebox{0}{ Planetoid}}
    & Cora            & 80.80 $\pm$ 1.02 & \textbf{81.90 $\pm$ 0.80} \\
    & CiteSeer        & 66.50 $\pm$ 1.28 & \textbf{67.00 $\pm$ 1.04} \\
    & PubMed          & 78.87 $\pm$ 0.24 & \textbf{82.10 $\pm$ 0.90} \\
    \midrule
    \multirow{2}{*}{\rotatebox{0}{ Amazon}}
    & Amz. Photos     & 92.89 $\pm$ 0.09 & \textbf{93.20 $\pm$ 0.09} \\
    & Amz. Computers  & 84.72 $\pm$ 0.07 & \textbf{86.87 $\pm$ 0.14} \\
    \midrule
    & Wiki CS         & 79.32 $\pm$ 0.08 & \textbf{80.59 $\pm$ 0.06} \\
    \midrule
    \multirow{2}{*}{\rotatebox{0}{\shortstack{Citation\\full}}}
    & Full DBLP       & 81.11 $\pm$ 0.17 & \textbf{81.58 $\pm$ 0.14} \\
    & Full Cora       & \textbf{59.94 $\pm$ 0.12} & 59.27 $\pm$ 0.22 \\
    \midrule
    & Roman Empire    & 64.41 $\pm$ 0.11 & \textbf{83.18 $\pm$ 0.19} \\
    \midrule
    & Amz. Ratings    & 49.35 $\pm$ 0.12 & \textbf{55.45 $\pm$ 0.11} \\
    \midrule
    \multirow{2}{*}{\rotatebox{0}{ Coauthor}}
    & Co. CS          & 92.21 $\pm$ 0.06 & \textbf{92.84 $\pm$ 0.08} \\
    & Co. Physics     & 93.01 $\pm$ 0.24 & \textbf{93.62 $\pm$ 0.24} \\
    \midrule
    & HM Categories   & 48.72 $\pm$ 0.11 & \textbf{54.13 $\pm$ 0.12} \\
    \midrule
    & Pokec 100k      & \textbf{86.05 $\pm$ 0.49} & 85.95 $\pm$ 0.07 \\
    \midrule
    \multirow{3}{*}{\rotatebox{0}{ City}}
    & City Paris      & 52.83 $\pm$ 0.20 & \textbf{55.44 $\pm$ 0.12} \\
    & City LA         & 55.16 $\pm$ 0.20 & \textbf{58.30 $\pm$ 0.18} \\
    & City Shanghai   & 53.47 $\pm$ 0.10 & \textbf{60.35 $\pm$ 0.05} \\
    \midrule
    & Mean accuracy (\%) & 68.26 & \textbf{70.37} \\
    \bottomrule
    \end{tabular}
\end{table}

Finetuning raises mean accuracy from \(68.26\%\) to \(70.37\%\), improving
performance on \(20\) of the \(26\) graphs. Note that zero-shot mean accuracy differs from earlier tables, since we only considered split~0 here. The largest gains occur on Roman
Empire and several Amazon, HM Categories, and city-network datasets, indicating
that dataset-specific supervision can substantially adapt the pretrained
representations when the target graph differs from the pretraining distribution.
However, gains are not uniform, with the largest decreases occurring on the two
small datasets Cornell and Wisconsin. Overall, finetuning provides useful
adaptation from an already strong zero-shot initialization. Future work could
investigate protocols that adapt the optimization and sampling budgets to graph
size and the amount of available supervision.

\subsection{Prior ablations}\label{app::prior_experiments}
To assess the effect of the synthetic graph distribution, we repeat
\textsc{Wander}'s pretraining with either the \textsc{NodePFN} or
\textsc{GraphPFN} prior in place of our synthetic attributed-graph prior.
In both ablations, we retain the original mixture of synthetic and
real-world knowledge-graph training episodes, including the same three
knowledge graphs; only the generator of synthetic graphs changes.
The architecture and remaining training setup are held fixed.

Table~\ref{tab:prior_ablations} reports the
aggregate results. The \textsc{Wander} prior yields the highest mean node
classification accuracy and homogeneous link-prediction Recall@20; knowledge
graph link-prediction MRR changes only marginally.

Figure~\ref{fig:prior_ablations} shows where the differences arise. In
homogeneous link prediction, both alternative priors underperform
particularly on graphs with high clustering relative to their average
degree, consistent with the broader coverage of this structural regime in
Figure~\ref{fig:prior_scatter}. For node classification, the pattern is less
uniform: \textsc{GraphPFN}'s prior improves accuracy on several
heterophilous graphs, even though its aggregate accuracy is lower. One
possible explanation is that homophilous graphs make up a larger share of
the \textsc{Wander} prior, leaving its pretrained model less adapted to
heterophily. These interpretations remain tentative, since the priors
differ in other aspects of graph and attribute generation as well.

\begin{figure}[t]
    \centering
    \begin{subfigure}[t]{0.49\linewidth}
        \centering
        \includegraphics[width=\linewidth]{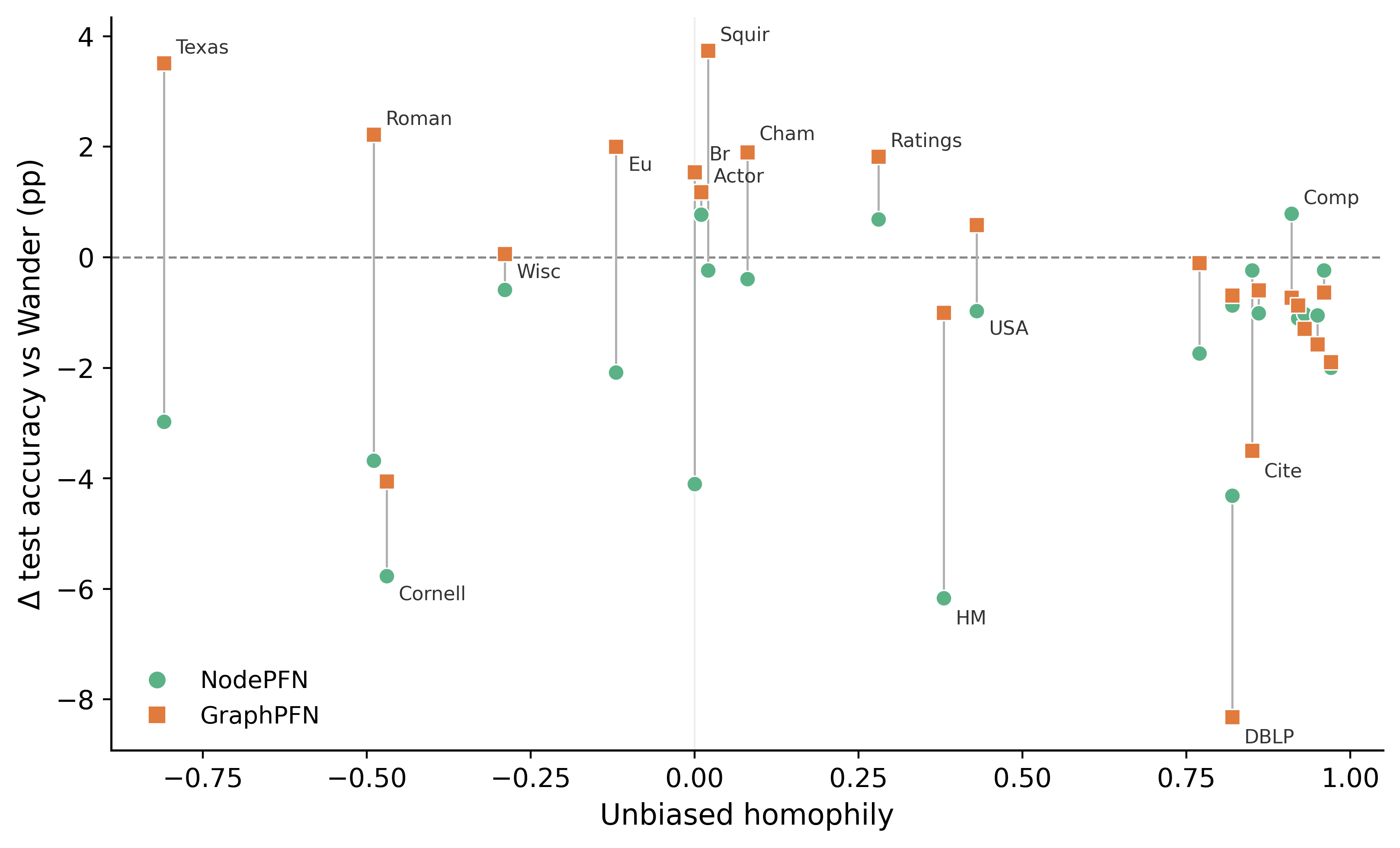}
        \caption{Change in node-classification accuracy versus unbiased homophily.}
        \label{fig:prior_ablations_nc}
    \end{subfigure}
    \hfill
    \begin{subfigure}[t]{0.49\linewidth}
        \centering
        \includegraphics[width=\linewidth]{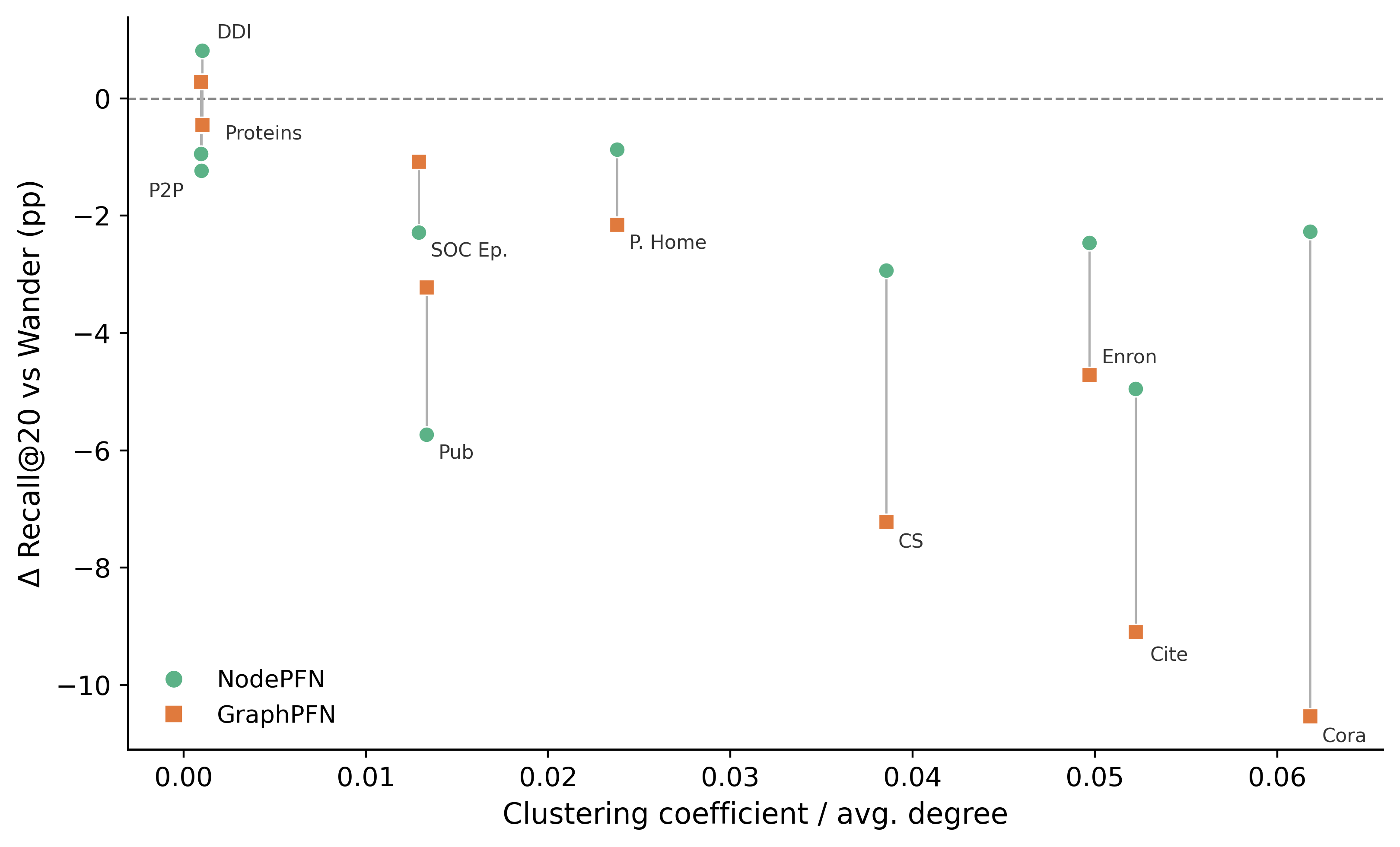}
        \caption{Change in link-prediction Recall@20 versus clustering coefficient divided by average degree.}
        \label{fig:prior_ablations_lp}
    \end{subfigure}
    \caption{Effect of replacing \textsc{Wander}'s synthetic attributed-graph prior with the \textsc{NodePFN} or \textsc{GraphPFN} prior. Points show results on individual real-world datasets; vertical differences are measured relative to the model trained with the \textsc{Wander} prior, in percentage points. }
    \label{fig:prior_ablations}
\end{figure}
\subsection{Architecture ablations}\label{app::ablations}
We assess the contribution of the three update operations that constitute each
\textsc{Wander} layer: the intra-node update, the in-context-learning (ICL)
update, and the stochastic structural update. We train three variants in which
one of these operations is removed from every layer. Without the intra-node
update, information cannot be exchanged between the feature, label,
and structure tokens associated with the same node. Without the ICL update,
tokens cannot directly attend to the labeled context nodes within their
respective channels. Finally, without the stochastic structural update, no graph information can be used, and the model falls back to a tabular model.

The results confirm that all three updates contribute, although their
importance depends on the task. Removing the intra-node update has almost no
effect on KG link prediction, which is expected
because the knowledge graphs contain no node features and prediction is driven
primarily by relation-conditioned structural information. In contrast, the
same ablation substantially degrades node classification and homogeneous link
prediction, where intra-node mixing is needed to combine attributes with label
and structural tokens. Removing the ICL update similarly hurts both attributed
tasks and causes a smaller reduction on knowledge graphs. The stochastic
structural update is essential across all settings: without it, performance
drops sharply, and KG link prediction becomes effectively impossible in the
absence of node attributes.

\begin{table*}[t]
    \centering
    \begin{minipage}[t]{0.49\textwidth}
        \centering
        \caption{Ablation of the synthetic attributed-graph pretraining prior.}
        \label{tab:prior_ablations}
        \scriptsize
        \setlength{\tabcolsep}{3pt}
        \begin{tabular}{lccc}
            \toprule
            Synthetic prior
                & \makecell[c]{KG LP\\MRR}
                & \makecell[c]{Node cls.\\Acc.\ (\%)}
                & \makecell[c]{Hom.\ LP\\R@20 (\%)} \\
            \midrule
            \textsc{Wander} prior
                & \textbf{0.399} & \textbf{68.20} & \textbf{43.39} \\
            \textsc{NodePFN} prior
                & 0.397 & 66.44 & \underline{41.16} \\
            \textsc{GraphPFN} prior
                & \underline{0.398} & \underline{67.33} & 39.53 \\
            \bottomrule
        \end{tabular}
    \end{minipage}
    \hfill
    \begin{minipage}[t]{0.49\textwidth}
        \centering
        \caption{Ablations of the three layer components of \textsc{Wander}.}
        \label{tab:architecture_ablations}
        \scriptsize
        \setlength{\tabcolsep}{3pt}
        \begin{tabular}{lccc}
            \toprule
            Model
                & \makecell[c]{KG LP\\MRR}
                & \makecell[c]{Node cls.\\Acc.\ (\%)}
                & \makecell[c]{Hom.\ LP\\R@20 (\%)} \\
            \midrule
            \textsc{Wander}
                & \textbf{0.399} & \textbf{68.20} & \textbf{43.39} \\
            \quad w/o intra-node update
                & 0.398 & 60.23 & 38.1 \\
            \quad w/o ICL update
                & 0.393 & 61.02 & 37.1 \\
            \quad w/o stochastic structural update
                & 0.001 & 51.11 & 15.6 \\
            \bottomrule
        \end{tabular}
    \end{minipage}
\end{table*}

\end{document}